\documentclass{siamonline1116}

\usepackage{amssymb}
\usepackage{amsmath}
\usepackage{graphicx}
\usepackage{epstopdf}
\usepackage{algorithmic}
\ifpdf
  \DeclareGraphicsExtensions{.png,.pdf,.jpg,.eps}
\else
  \DeclareGraphicsExtensions{.eps}
\fi
\graphicspath{{figs/}}

\newcommand{\RR}{\mathbb{R}}

\providecommand{\abs}[1]{\lvert#1\rvert}
\let\Re\relax
\let\Im\relax
\DeclareMathOperator{\Re}{Re}
\DeclareMathOperator{\Im}{Im}

\theoremstyle{definition}
\newtheorem{remark}{Remark}

\numberwithin{theorem}{section}

\newcommand{\TheTitle}{Practical Distance Functions for Path-Planning in Planar Domains}
\newcommand{\TheAuthors}{R. Chen, C. Gotsman, and K. Hormann}

\headers{\TheTitle}{\TheAuthors}

\title{{\TheTitle}}
\author{
  Renjie Chen\thanks{Max-Planck Institute, Saarbr\"ucken, Germany
    (\email{renjie.chen@mpi-inf.mpg.de}).}
  \and
  Craig Gotsman\thanks{New Jersey Institute of Technology, Newark, NJ
    (\email{gotsman@njit.edu}).}
  \and
  Kai Hormann\thanks{Universit\`a della Svizzera italiana, Lugano, Switzerland
    (\email{kai.hormann@usi.ch}).}
}

\begin{document}

\maketitle

\begin{abstract}

Path planning is an important problem in robotics. One way to plan a path between two points $x,y$ within a (not necessarily simply-connected) planar domain $\Omega$, is to define a non-negative distance function $d(x,y)$ on $\Omega\times\Omega$ such that following the (descending) gradient of this distance function traces such a path. This presents two equally important challenges: A mathematical challenge -- to define $d$ such that $d(x,y)$ has a single minimum for any fixed $y$ (and this is when $x=y$), since a local minimum is in effect a ``dead end''; A computational challenge -- to define $d$ such that it may be computed efficiently. In this paper, given a description of $\Omega$, we show how to assign coordinates to each point of $\Omega$ and define a family of distance functions between points using these coordinates, such that both the mathematical and the computational challenges are met. This is done using the concepts of \emph{harmonic measure} and \emph{$f$-divergences}.

In practice, path planning is done on a discrete network defined on a finite set of \emph{sites} sampled from $\Omega$, so any method that works well on the continuous domain must be adapted so that it still works well on the discrete domain. Given a set of sites sampled from $\Omega$, we show how to define a network connecting these sites such that a \emph{greedy routing} algorithm (which is the discrete equivalent of continuous gradient descent) based on the distance function mentioned above is guaranteed to generate a path in the network between any two such sites. In many cases, this network is close to a (desirable) planar graph, especially if the set of sites is dense.
\end{abstract}


% REQUIRED

%%%%%%%%%%%%%%%%%%%%%%%%%%%%%%%%%%%%%%%%%%%%%%%%%%%%%%%%%%%%%%%%%%%%%%%%%%%%

\section{Introduction}

Path planning in a planar domain containing obstacles is an important problem in robotic navigation. The objective is for an autonomous agent to move from one point (the source) in the domain to another (the target) along a realistic path which avoids the obstacles, where the path is determined automatically and efficiently based only on knowledge of the domain and local information related to the current position of the agent. This important problem has attracted much attention in the robotics community and is the topic of ongoing research. A well-known family of path planning algorithms, inspired by the physics of electrical force fields, is based on \emph{potential} functions. These were first proposed by Khatib~\cite{Khatib:1986:RTO} and developed by Kim and Khosla~\cite{Kim:1992:RTO}, Rimon and Koditschek~\cite{Rimon:1992:ERN}, and Connolly and Grupen~\cite{Conolly:1993:TAO} soon after. The main idea is, given the target point, to construct a scalar function on the domain, such that a path to the target point from any other source point may be obtained by following the negative gradient of the function. While elegant, Koren and Borenstein~\cite{Koren:1991:PFM} have identified a number of significant pitfalls that these methods may encounter, the most important being the so-called ``trap'' situations -- the presence of local minima in the potential function. To avoid this, the scalar function must have a global minimum (typically zero-valued) at the target, and be void of local minima elsewhere in the domain. The presence of ``spurious'' local minima could be fatal, since the gradient vanishes and the robot will be ``stuck'' there. Other critical points, such as saddles, are undesirable but not fatal, since a negative gradient can still be detected by ``probing'' around the point.

Designing and computing potential functions for planar domains containing obstacles has been a topic of intense activity for decades. Perhaps the most elegant type of potential function is the harmonic function~\cite{Axler:2001:HFT}, which has very useful mathematical properties, most notably the guaranteed absence of spurious local minima. Alas, the main problems preventing widespread use of these types of potential functions are the high complexity of computing the function, essentially the solution of a very large system of linear equations on a discretization of the domain every time the target point is changed, and the fact that very high precision numerical methods are required, as the functions are almost constant, especially in regions distant from the target. A recent paper of Chen et al.~\cite{Chen:2016:PPW} addresses the first of these issues. They describe a new family of distance functions, which, while quite distinct from the harmonic potential function, generate exactly the same gradient-descent paths. However, they do this at a tiny fraction of the computational cost.

Chen et al.~\cite{Chen:2016:PPW} use the concept of harmonic measure~\cite{Garnett:2005:HM} and its conformal invariance to define a family of ``shape-aware'' distance functions $d_f\colon\Omega\times\Omega\to[0,\infty)$ on a bounded simply-connected planar domain $\Omega$. A function in this family is based on any real strictly convex function $f$ and has the key property that for any $z\in\Omega$,
\begin{equation}\label{eq:grad_df=0}
  \nabla_z d_f(z,y) = 0
  \qquad\text{iff}\qquad
  z = y,
\end{equation}
which means that a continuous path from $z$ to $y$ may be planned by simply following (the negative of) this gradient vector. The distance function is defined as the boundary integral
\begin{equation}\label{eq:df}
  d_f(z,y) = \oint_{\partial\Omega} P(z,t) f\biggl(\frac{P(y,t)}{P(z,t)}\biggr) dt
\end{equation}
where $P(z,t)$ is the \emph{Poisson kernel} of $\Omega$ at $z$, namely, the normal derivative at the boundary of $\Omega$'s \emph{Green's function} for the \emph{Laplace equation}~\cite{Axler:2001:HFT}. The distance $d_f$ is called the \emph{$f$-divergence distance}, as it is based on the $f$-divergence~\cite{Csiszar:1967:ITM,Kullback:1951:OIA} of the two functions $P(z,t)$ and $P(y,t),$. $f$-divergence is commonly used in statistics to measure the distance between two probability distributions. Although there are many choices for $f$ having different desirable properties, we mention the two special cases $f(x)=-\log x$ and $f(x)=2(1-\sqrt{x})$, which are called the Kullback--Leibler (KL) and the Hellinger (H) divergences, respectively. Although $f$-divergence is in general not a symmetric function, it can be shown that $f$-divergence distances are symmetric, namely $d_f(z,y) = d_f(y,z)$. However, $f$-divergence distance is in general not a metric, because it does not satisfy the triangle inequality. Chen et al.~\cite{Chen:2016:PPW} prove that distance functions of the type~\eqref{eq:df} indeed have property~\eqref{eq:grad_df=0} if $\Omega$ is a simply-connected domain, and, most interestingly, the path generated is \emph{invariant} to $f$. This is because only the \emph{magnitude} of $\nabla d_f$ depends on $f$, \emph{but not its direction.} Furthermore, the path generated is identical to that generated using the classical potential function method (a close relative of the Green's function of the domain). Figure~\ref{fig:grad-descent} shows the paths generated by the gradient-descent path planner in two different simply-connected domains. In the simple disk domain, it may be shown that the paths generated are always circular arcs (so-called \emph{hyperbolic geodesics}). The right image in Figure~\ref{fig:grad-descent} corresponds to using the na\"{i}ve ``Euclidean'' distance function $d_2(z,y)=\oint_{\partial\Omega} \bigl(P(z,t)-P(y,t)\bigr)^2 dt$, which results in the distance function having local minima, which, in turn, results in the path-planner being attracted to these points and getting ``stuck'' at them.

\begin{figure}
  \includegraphics[height=2in]{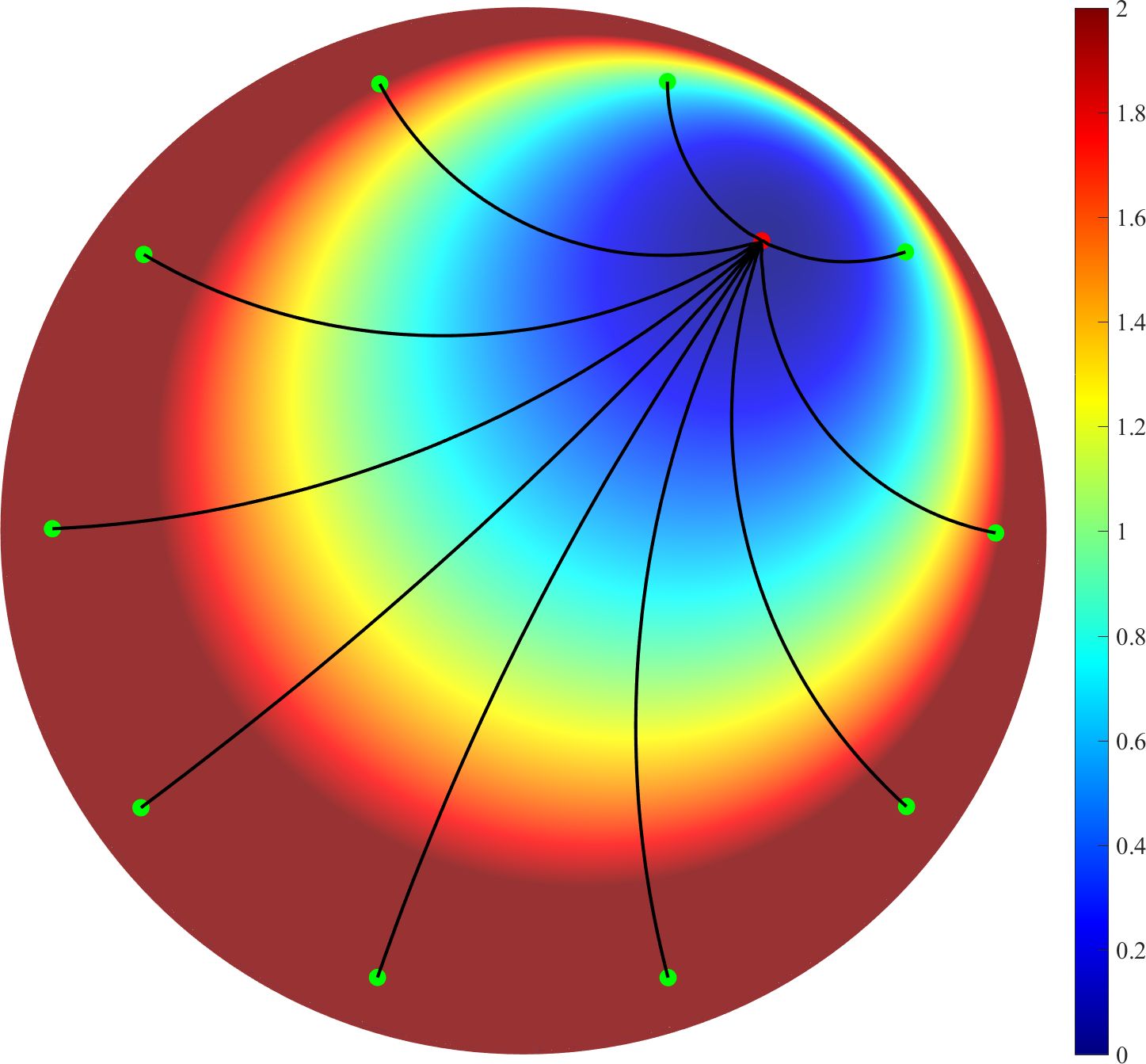}\hfill
  \includegraphics[height=2in]{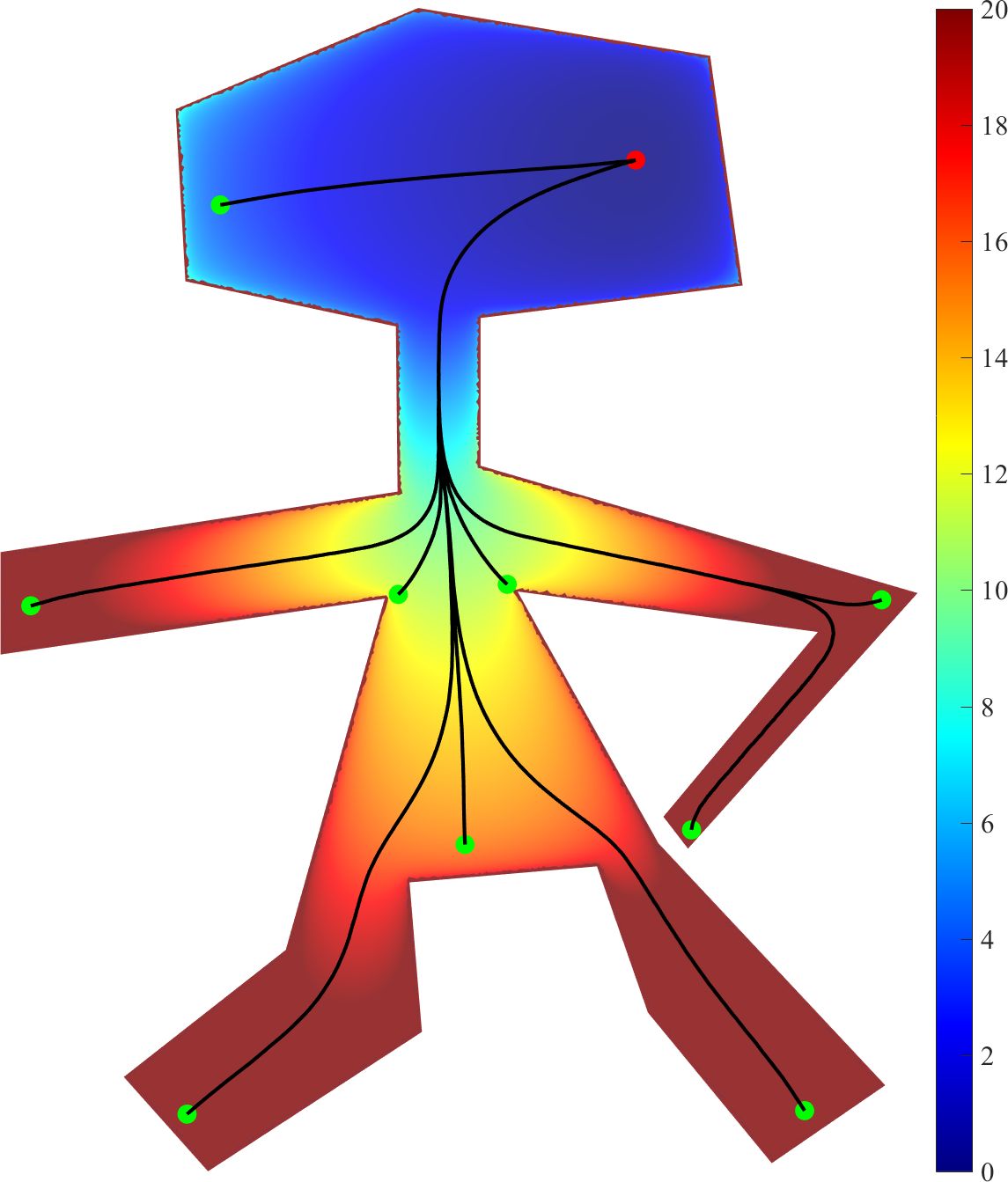}\hfill
  \includegraphics[height=2in]{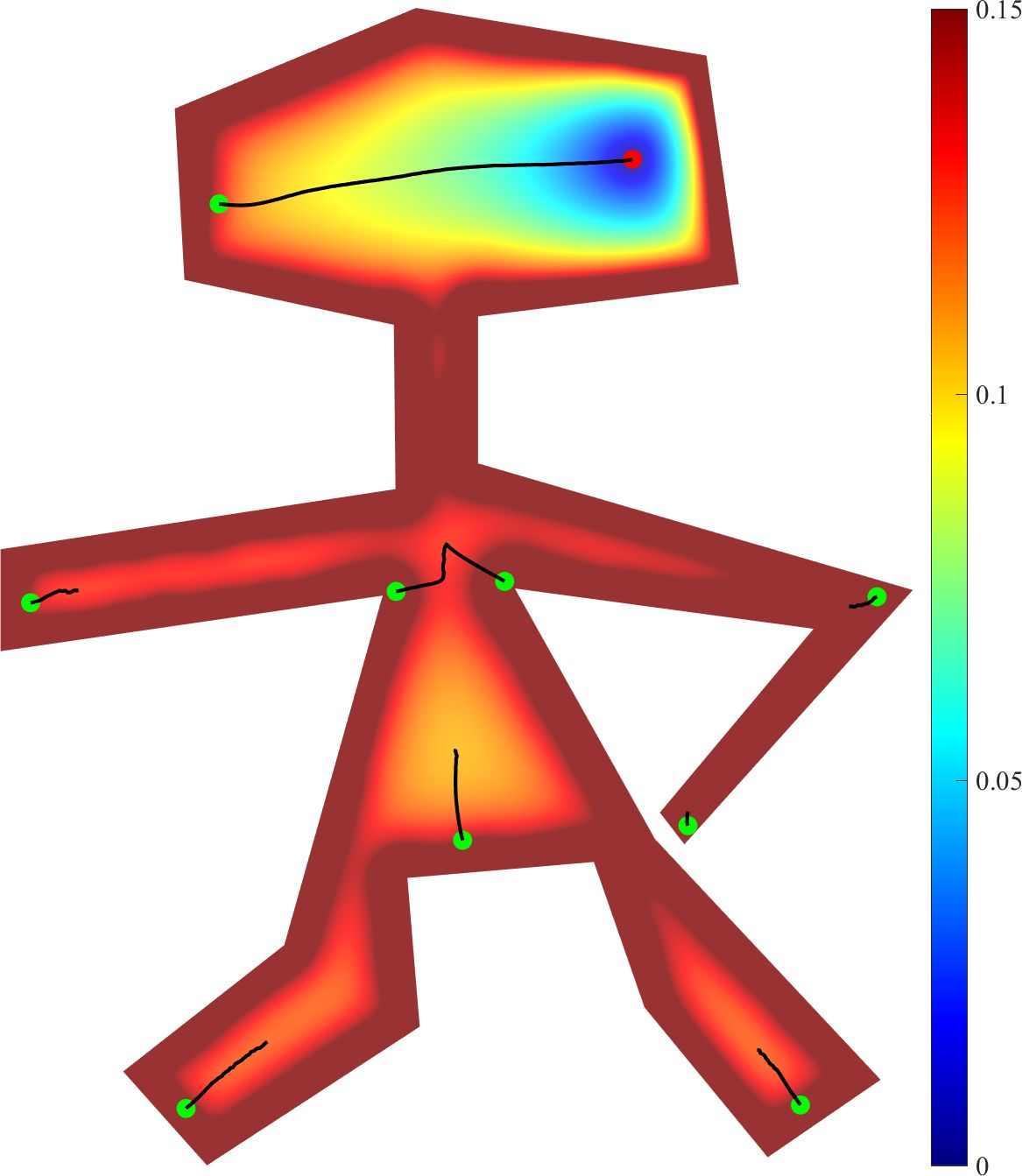}\par
  \caption{Gradient-descent paths generated by an $f$-divergence distance function on two domains. The red point is the target and the green points are different sources. The domains are color-coded by the $f$-divergence distance from the target. \emph{Left and center:} Using the strictly convex function $f(x)=-\log x$. \emph{Right:} Using a ``Euclidean'' distance function $d_2$. Note how, in the latter case, the gradient descent gets ``stuck'' at local minima of the distance function, so a gradient-descent path cannot be followed. Both domains were discretized by a triangulation with $k=2\times10^5$ points.}
  \label{fig:grad-descent}
\end{figure}

In practice, the path-planner discretizes the domain $\Omega$ -- typically into a triangulation of $k$ points -- where $k$ can be on the order of hundreds of thousands, and all computations are done on the points of this triangulation. Given a source point $x\in\Omega$ and a target point $y\in\Omega$, path planning using a potential function requires solving a large $k\times k$ system of (sparse) linear equations (dependent on $y$) -- essentially a Finite Element Method (FEM) applied to the continuous Laplace equation -- and then following the negative gradient of this scalar function along the edges of the triangulation. The advantage of using the distance function~\eqref{eq:df} instead is that it requires only preprocessing the domain once in advance -- to compute the $P(z,t)$ functions -- and then, given $x$ and $y$, following the negative gradient (by $z$) of $d_f(z,y)$. According to~\eqref{eq:df}, each computation of $d_f$ involves computing a boundary integral using $P(z,t)$. In practice, the integral is replaced with a sum, namely $P(z,t)$ is discretized to a ``coordinate vector'' -- one value for each boundary point of the triangulation. These vectors typically have length $O(\sqrt{k})$. Changing $y$ does not incur any extra computation.

While a cost of $O(\sqrt{k})$ per distance computation~\eqref{eq:df} does not seem too expensive, in practice it may still be too much for real-time performance. Furthermore, the coordinate vectors generated by the preprocessing procedure must be stored for each of the $k$ discrete points of the FEM mesh, implying an $O(k^{1.5})$ storage requirement, which could be prohibitive. This paper addresses these two issues.

%%%%%%%%%%%%%%%%%%%%%%%%%%%%%%%%%%%%%%%%%%%%%%%%%%%%%%%%%%%%%%%%%%%%%%%%%%%%

\section{Contribution}

This paper makes two main contributions: The first contribution shows how to make gradient-descent path-planning more practical. Instead of assigning a ``continuous'' (namely, a very long $O(\sqrt{k})$ discrete) coordinate vector to each point of $\Omega$, we show in Section \ref{section:reduced} how to ``reduce'' this to a very small number of \emph{reduced coordinates} $n\ll O(\sqrt{k})$, without losing the key property~\eqref{eq:grad_df=0}. This small number of coordinates reduces the computation complexity of computing $d_f$ from $O(\sqrt{k})$ to $O(n)$. The small price paid in reducing the size of the coordinate vector is that the gradient-descent paths generated by the planner may not be as natural as before, and we lose the properties of symmetry and gradient direction independence on $f$.

The second contribution shows how to use reduced $f$-divergence distances in a purely discrete setting. While reducing the size of the coordinate vector from $O(\sqrt{k})$ to $n$ makes for an efficient computation of the distance function $d_f$, generation of a ``continuous'' path requires the use of the dense ``underlying mesh'' (essentially that used for the FEM computation), implying storage requirement of $O(kn)$. In practice it would be much more efficient to plan a path on a sparse network of $m$ points sampled in $\Omega$. This requires building a suitable network of edges between the points, one that supports \emph{greedy routing}: if $V$ is the set of points and $N(z)$ is the set of points of $V$  connected to $z$ by a network edge, then
\begin{equation}\label{eq:neighbours}
  \forall z \neq y \in V,\quad \exists u \in N(z) : d_f(u,y) < d_f(z,y).
\end{equation}
In other words, there is always a neighbor of $z$ which is closer to the target $y$ than $z$ is. This is the discrete analog to~\eqref{eq:grad_df=0}. In Section \ref{section:discrete} we describe an algorithm to build this graph. To illustrate, a typical domain, such as those used in the figures of this paper, requires an underlying mesh containing $k=2\times10^5$ points. A typical boundary size would be $700$ points. Thus, using the algorithm of Chen et al.~\cite{Chen:2016:PPW} would require storing $1.4\times10^8$ real values, and every computation of $d_f$ at each path point would require a loop of $700$ iterations. In contrast, using 30 reduced coordinates on a domain sampled to 300 sites would require storing only 9,000 real values and each computation of $d_f$ would require a loop of just 30 iterations.

%%%%%%%%%%%%%%%%%%%%%%%%%%%%%%%%%%%%%%%%%%%%%%%%%%%%%%%%%%%%%%%%%%%%%%%%%%%%

\section{The $f$-divergence Distance}

A fundamental concept used in our solution is the $f$-divergence function, first introduced by Kullback and Leibler \cite{Kullback:1951:OIA} and later generalized by Csisz{\'a}r~\cite{Csiszar:1967:ITM}, for measuring the difference between two probability distributions:

\begin{definition}[$f$-divergence]
Let $f$ be a strictly convex function such that $f(1)=0$ and $p,q\colon E\to[0,1]$ be two real functions on some domain $E$ such that $\int_E p(t)dt=\int_E q(t)dt = 1$. The $f$-divergence between $p$ and $q$ is
\[
  d_f(p,q) = \int_E p(t) f \biggl( \frac{q(t)}{p(t)} \biggr) dt.
\]
\end{definition}

It is well-known that
\[
  d_f(p,q) \ge 0
\]
and
\[
  d_f(p,q)=0 
  \qquad\text{iff}\qquad 
  p=q,
\]
but $d_f$ is not necessarily a metric.
Many instances of $f$ have been proposed over the years, each suitable for some specific application, mostly in probability theory, statistics and information theory. The interested reader is referred to~\cite{Liese:2006:ODA} for a survey of the possibilities.

The concept of the dual function
\[
  f^*(x) = x f \Bigl( \frac{1}{x} \Bigr)
\]
is also noteworthy. For example, if $f(x)=-\log x$, then $f^*(x)=x\log x$, and if $f(x)=\abs{1-x}$, then $f^*(x)=f(x)$. It is well known that
\begin{enumerate}
  \item  $f$ is (strictly) convex iff $f^*$ is (strictly) convex;
  \item  $d_f(p,q)=d_{f^*}(q,p)$;
  \item  $f(1)\le d_f(p,q)\le f(0)+f^*(0)$.
\end{enumerate}

With slight abuse of notation, the \emph{$f$-divergence distance} between two points in a planar domain $\Omega$ is defined using the Poisson kernel of $\Omega$:

\begin{definition}[$f$-divergence distance]
Let $f$ be a strictly convex function such that $f(1)=0$ and $y,z\in \Omega$. The $f$-divergence distance between $y$ and $z$ is
\begin{equation}\label{eq:divdist}
  d_f(z,y) = \oint_{\partial\Omega} P(z,t) f\biggl(\frac{P(y,t)}{P(z,t)}\biggr) dt
\end{equation}
where $P(z,t)$ is the \emph{Poisson kernel} of $\Omega$ at $z$, 
\end{definition}

Although the $f$-divergence of two probability functions is not neccesarily symmetric, the special nature of the $f$-divergence distance implies that it is symmetric. In fact, $f$ and $f^*$ generate \emph{identical} divergence distances:
\[
  d_f(y,z)=d_{f^*}(y,z)=d_f(z,y)
\]
Although symmetric, the $f$-divergence distance will, in general, not be a metric, since it may fail to satisfy the triangle inequality.

For path-planning purposes, the gradient of the $f$-divergence distance plays a key role. Chen et al.~\cite{Chen:2016:PPW} show that for simply-connected domains, the gradient field never vanishes, and its direction is independent of $f$:
\[
\nabla d_f(y,z)=0   ~~~\text{iff} ~~~y=z
\]
\[
    \forall \text{strictly convex} ~ f ~\text{and}~ g, \quad \arg(\nabla d_f(y,z))= \arg(\nabla d_g(y,z))
\]

These important properties of the gradient allow the $f$-divergence distance to be used for gradient-descent path-planning.
%%%%%%%%%%%%%%%%%%%%%%%%%%%%%%%%%%%%%%%%%%%%%%%%%%%%%%%%%%%%%%%%%%%%%%%%%%%%

\section{Reduced Coordinates} \label{section:reduced}

As mentioned above, Chen et al.~\cite{Chen:2016:PPW} define the $f$-divergence distance between points $z, y$ in a planar domain $\Omega$ as the $f$-divergence of their Poisson kernels, which can be viewed as probability functions on the domain boundary $\partial\Omega$. They subsequently prove that (for simply-connected domains) the gradient of this distance function never vanishes except when $z=y$, implying a single minimum at that point. So the Poisson kernels can be viewed as a continuous ``coordinate vector'' for a planar point. We now describe the construction of the discrete \emph{reduced coordinates}, which is very simple: we aggregate the continuous Poisson kernel vector to a small number of (positive) values. These are the new coordinates. More formally, we partition the domain boundary $\partial\Omega$ into $n$ continuous segments,  $(E_1, \dots, E_n)$ defined by $(t_1,\dots,t_n)$, and $E_j=[t_j,t_{j+1}]$, where we identify $t_{n+1}$ cyclically with $t_1$. The $n$ \emph{reduced coordinates} $(\phi_1(z),\dots,\phi_n(z))$ of a point $z\in\Omega$ are defined as
\[
  \phi_j(z) = \int_{E_j} P(z,t) dt,
\]
The quantity $\phi_j(z)$ is called the \emph{harmonic measure} of $z$ relative to the $j$-th segment~\cite{Garnett:2005:HM}. The \emph{reduced} $f$-divergence distance is then (compare to~\eqref{eq:df})
\begin{equation}\label{eq:df-discrete}
  d_f(z,y) = \sum^n_{j=1} \phi_j(z) f \biggl( \frac{\phi_j(y)}{\phi_j(z)} \biggr).
\end{equation}

%%%%%%%%%%%%%%%%%%%%%%%%%%%%%%%%%%%%%%%%%%%%%%%%%%%%%%%%%%%%%%%%%%%%%%%%%%%%

\section{The Divergence Gradient Theorem}

In this section we prove the central result of this paper:

\begin{theorem}[Divergence Gradient Theorem] \label{theorem:DGT}
For a simply-connected domain $\Omega$, strictly convex $f$ and $n>2$, the reduced $f$-divergence distance~\eqref{eq:df-discrete} has property~\eqref{eq:grad_df=0}.
\end{theorem}

The Divergence Gradient Theorem implies that we may use any reduced $f$-divergence to generate gradient-descent paths. Below, we prove the theorem by conformal reduction to the case of a domain which is the unit disk, and the target point the origin. But beforehand, a number of preliminaries are required.

\subsection{Some circle geometry}\label{sec:CircleGeometry}

\begin{figure}
  \hspace*{\fill}%
  \includegraphics[height=2.5in]{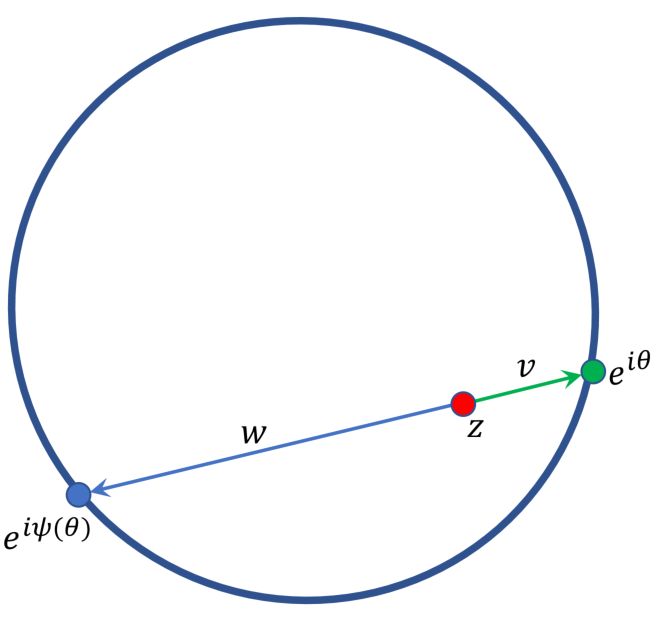}\hfill\hfill%
  \includegraphics[height=2.5in]{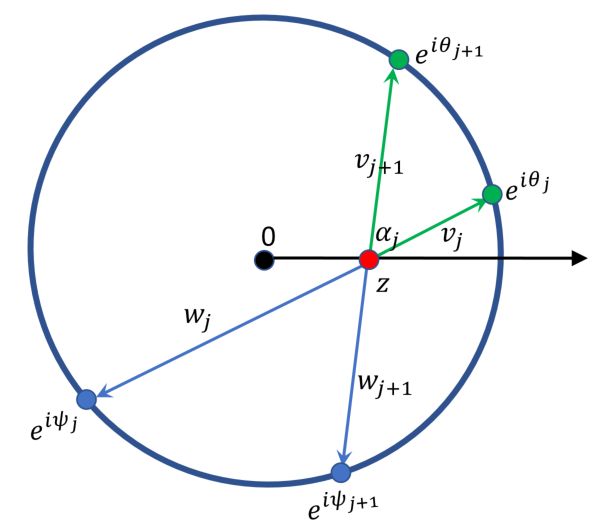}\hspace*{\fill}\par
  \caption{\emph{Left:} Notation used in Section~\ref{sec:CircleGeometry}. \emph{Right:} Notation used in Section~\ref{sec:ReducedCoordinates}.}
  \label{fig:notation}
\end{figure}

In all that follows, we use complex number algebra in the plane. As shown in Figure~\ref{fig:notation} (left), for a given $z$ in the unit disk, we denote by $\psi(\theta)\in(-\pi,\pi]$ the \emph{antipode} of $\theta\in(-\pi,\pi]$ relative to $z$, that is, $e^{i\psi(\theta)}$ is the intersection of the chord through $e^{i\theta}$ and $z$ with the unit circle. Also denote $v=e^{i\theta}-z$ and $w=e^{i\psi(\theta)}-z$.

\begin{lemma}\label{lemma:1/vbar}
For any $\theta\in(-\pi,\pi]$,
\[
  \frac{1}{\overline{v}} = -\frac{w}{1-\abs{z}^2}.
\]
\end{lemma}
\begin{proof}
By the intersecting chords theorem, $\abs{v}\abs{w}=1-\abs{z}^2$. Since $v$ and $w$ are collinear, this implies $-\overline{v}w=\abs{v}\abs{w}=1-\abs{z}^2$, from which the statement follows.
\end{proof}

\begin{lemma}\label{lemma:psi}
For any $\theta\in(-\pi,\pi]$,
\[
  \psi(\theta) = -i \log \Bigl( \frac{-v}{\overline{v}} e^{-i\theta} \Bigr).
\]
\end{lemma}
\begin{proof}
Using Lemma~\ref{lemma:1/vbar}, we know that $w+z=\frac{-v e^{-i\theta}}{\overline{v}}$, and by the definition of $\psi(\theta)$ as antipode of $\theta$, we then conclude $\psi(\theta)=-i\log(w+z)=-i\log\bigl(\frac{-ve^{-i\theta}}{\overline{v}}\bigr)$.
\end{proof}

Note that when $z=0$, we have $\psi(\theta)=\pi+\theta$, as expected.

\subsection{Reduced coordinates}\label{sec:ReducedCoordinates}

As shown in Figure~\ref{fig:notation} (right), given a partition of the unit circle $-\pi<\theta_1<\dots<\theta_n\le\pi$, we define the reduced coordinates of $z$ as $(\phi_1(z),\dots,\phi_n(z))$, where $\phi_j(z)$ is the harmonic measure of $z$ relative to the arc $(\theta_j,\theta_{j+1})$,
\begin{equation}\label{eq:phi_j}
  \phi_j(z) = \int_{\theta_j}^{\theta_{j+1}} P^D(z,\theta)d\theta = \frac{1}{2\pi} \bigl( 2\alpha_j(z)-(\theta_{j+1}-\theta_j) \bigr),
\end{equation}
with $P^D(z,\theta )$ the Poisson kernel of the unit disk,
\[
  P^D(z,\theta) = \frac{1}{2\pi} \frac{1-\abs{z}^2}{\abs{z-e^{i\theta}}^2},
\]
and $\alpha_j(z)$ denoting the angle the arc forms with $z$. Note that since $e^{i\theta}$ is a $2\pi$-periodic function, we may use $\theta_{n+1}=\theta_1+2\pi$, so that $\theta_{j+1}-\theta_j$ always gives the (positive) length of the arc, thus $\phi_j(z)\ge0$ and $\sum_j \phi_j(z)=1$.

\begin{lemma}\label{lemma:harmonicMeasure}
The harmonic measure in~\eqref{eq:phi_j} can also be expressed as
\[
  \phi_j(z) = \frac{1}{2\pi}(\psi_{j+1}-\psi_j),
\]
where $\psi_j$ is the antipode of $\theta_j$ relative to $z$.
\end{lemma}
\begin{proof}
By the intersecting chords theorem, $\theta_{j+1}-\theta_j+\psi_{j+1}-\psi_j=2\alpha_j(z)$, and the statement then follows directly from~\eqref{eq:phi_j}.
\end{proof}

\begin{lemma}\label{lemma:grad_phi_j}
The gradient (by $z$) of $\phi_j(z)$ is 
\[
  \nabla\phi_j(z) = \frac{i}{\pi(1-\abs{z}^2)}(e^{i\psi_{j+1}}-e^{i\psi_j}).
\]
\end{lemma}
\begin{proof}
By Lemma~\ref{lemma:harmonicMeasure} and viewing $\psi_j$ as a function of $z$, we get 
\[
  \nabla \phi_j(z) = \frac{1}{2\pi} \nabla (\psi_{j+1}(z) - \psi_j(z)).
\]
Then using Lemma~\ref{lemma:psi} and the complex form of the gradient gives 
\[
  \nabla \psi_j(z) 
  = - i \nabla \log \Bigl( \frac{-v_j}{\overline{v}_j} \Bigr)
  = -2i \biggl( \overline{\frac{\partial}{\partial z} \log \Bigl(\frac{-v_j}{\overline{v}_j} \Bigr)} \biggr)
  = -\frac{2i}{\overline{v}_j},
\]
so
\[
  \nabla\phi_j(z) = -\frac{i}{\pi} \biggl( \frac{1}{\overline{v}_{j+1}} - \frac{1}{\overline{v}_j} \biggr). 
\]
Applying Lemma~\ref{lemma:1/vbar}, we finally get 
\[
  \nabla\phi_j(z)
  = \frac{i}{\pi(1-\abs{z}^2)} (w_{j+1}-w_j)
  = \frac{i}{\pi(1-\abs{z}^2)} (e^{i\psi_{j+1}}-e^{i\psi_j}).
\]
\end{proof}

\subsection{Divergence distances}

Given the reduced coordinates of $z$, that is, $(\phi_1(z),\dots,\phi_n(z))$ based on a partition of the unit circle $-\pi<\theta_1<\dots<\theta_n\le\pi$, and a strictly convex function $f$, the $f$-divergence distance between $z$ and $0$ is defined as in~\eqref{eq:df-discrete},
\[
  d_f(z) = \sum_{j=1}^n \phi_j(0) f \biggl( \frac{\phi_j(z)}{\phi_j(0)} \biggr).
\]
By the chain rule, Lemma~\ref{lemma:harmonicMeasure} and Lemma~\ref{lemma:grad_phi_j},
\begin{align*}
  \nabla d_f(z)
  &= \sum_j \phi_j(0) \nabla f \biggl( \frac{\phi_j(z)}{\phi_j(0)} \biggr)\\
  &= \sum_j f' \biggl( \frac{\phi_j(z)}{\phi_j(0)} \biggr) \nabla \phi_j(z)\\
  &= \frac{i}{\pi (1-\abs{z}^2)} \sum_j f'\biggl( \frac{\psi_j-\psi_{j+1}}{\theta_j-\theta_{j+1}} \biggr)
    (e^{i\psi_{j+1}} - e^{i\psi_j}).
\end{align*}
Without loss of generality, in the following we assume that $z$ is on the positive $x$-axis, so that $z=\overline{z}=\abs{z}$. Any other case can be reduced to this by a simple rotation of the plane.

\begin{lemma}\label{lemma:monotonic}
For $z\ne0$, the function $g(\psi)=f'\bigl(\frac{1}{\theta'(\psi)}\bigr)$ is \emph{strictly} increasing for $\psi\in(0,\pi)$ and \emph{strictly} decreasing for $\psi\in(-\pi,0)$. For $z=0$, $g$ is a constant function.
\end{lemma}
\begin{proof}
For $z\neq 0$, it follows from Lemma~\ref{lemma:psi} that $\psi(\theta)=-i\log\Bigl(\frac{-(e^{i\theta}-z)}{e^{-i\theta}-z}e^{-i\theta}\Bigr)$, hence
\[
  \psi'(\theta)
  = \frac{e^{i\theta}}{e^{i\theta}-z} + \frac{e^{-i\theta}}{e^{-i\theta}-z} - 1
  = \frac{z}{e^{i\theta}-z} + \frac{z}{e^{-i\theta}-z} + 1
  = 2z \Re\biggl( \frac{1}{e^{i\theta}-z} \biggr) + 1
\]
and
\begin{align*}
  \psi''(\theta)
  &= - \frac{i z e^{i\theta}}{(e^{i\theta}-z)^2}
     + \frac{i z e^{-i\theta}}{(e^{-i\theta}-z)^2}\\
  &= 2z \Im \biggl( \frac{e^{i\theta}}{(e^{i\theta}-z)^2} \biggr)\\
  &= \frac{-2z(1-z^2)}{\abs{e^{i\theta}-z}^4}{\sin\theta}
     \begin{cases}
       < 0, & \text{if $\theta\in(0,\pi)$},\\
       > 0, & \text{if $\theta\in(-\pi,0)$}.
     \end{cases}
\end{align*}
Note that when $z=0$, $\psi''(\theta)=0$ for all $\theta$. For $z\ne0$, we conclude that $\psi'$ is \emph{strictly} decreasing for $\theta\in(0,\pi)$ and \emph{strictly} increasing for $\theta\in(-\pi,0)$. Applying the chain rule, we then get
\[
  g'(\psi)
  = -\theta''(\psi) {(\theta'(\psi))}^{-2} f''\biggl( \frac{1}{\theta'(\psi)} \biggr).
\]
Now note that $\theta(\psi)$ and $\psi(\theta)$ are actually the same function (because $\theta$ and $\psi$ are antipodes of each other), so they have the same behavior. Moreover, $f''>0$, since $f$ is strictly convex. We conclude that the sign of $g'$ is opposite to the sign of $\theta''$, thus $g$ is \emph{strictly} increasing for $\psi\in(0,\pi)$ and \emph{strictly} decreasing for $\psi\in(-\pi ,0)$.
\end{proof}

Note that $g(\psi)$ is an even function, as is $\psi'(\theta)$.

\begin{figure}\centering
  \includegraphics[height=2.5in]{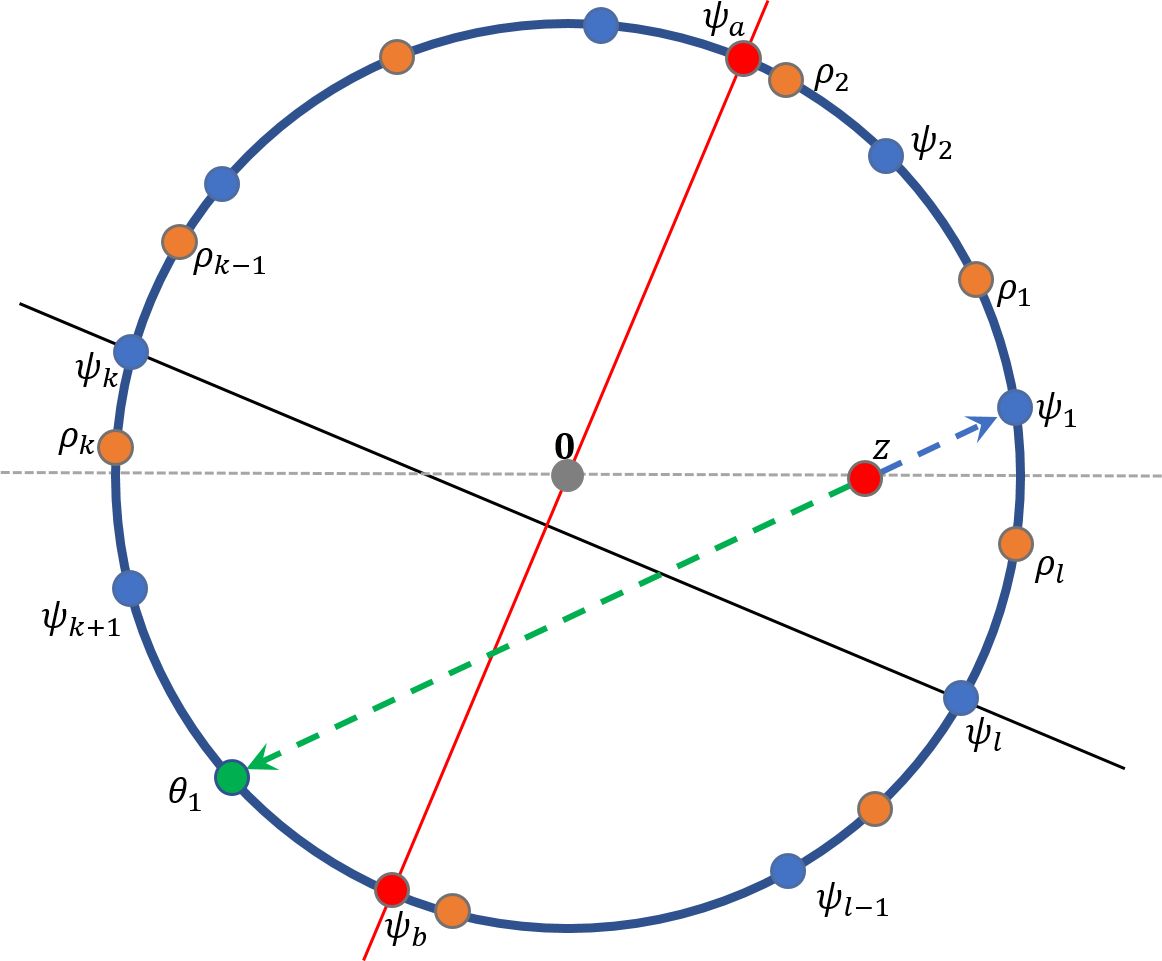}
  \caption{Notation used in the proof of Theorem~\ref{theorem:DGTD}.}
  \label{fig:notation2}
\end{figure}

\begin{theorem}[Divergence Gradient Theorem for the disk]\label{theorem:DGTD}
Given the reduced coordinates of $z$ based on a partition of the unit circle $-\pi<\theta_1<\dots<\theta_n\le\pi$, with $n\ge3$, the $f$-divergence distance satisfies
\[
  \nabla d_f(z) = 0 \qquad\text{iff}\qquad z=0.
\]
\end{theorem}
\begin{proof}
For each arc $\theta_j \theta_{j+1}$, the mean-value theorem states that there exists some $\tau_j\in(\theta_j,\theta_{j+1})$, such that $\psi'(\tau_j)=\frac{\psi_j-\psi_{j+1}}{\theta_j-\theta_{j+1}}$. Now let $\rho_j=\psi(\tau_j)\in(\psi_j,\psi_{j+1})\subset(-\pi,\pi]$ and define the piecewise constant (periodic) function $w(\psi)=g(\rho_j)$, $\psi\in(\psi_j,\psi_{j+1}]$. Obviously there exist $k$ and $l$ such that $\abs{\rho_k}\le\abs{\rho_j}\le\abs{\rho_l}$ for all $j$, that is, the indices $k$ and $l$ correspond to the leftmost and rightmost $\rho$'s, respectively. In particular, as shown in Figure~\ref{fig:notation2}, $l$ is either 1, $n-1$, or $n$. Since, by Lemma~\ref{lemma:monotonic}, $g$ is strictly increasing in $(0,\pi)$ and strictly decreasing in $(-\pi,0)$, by construction, $w(\psi)$ is monotonically (but not strictly) increasing in $(\psi_l,\psi_k]$ (if $\psi_k>0$, otherwise we use $(\psi_l,\psi_k+2\pi]$ and the same principle applies below), and decreasing in $(\psi_k,\psi_l]$.

Let $[e^{i\psi_a},e^{i\psi_b}]$ be the diameter orthogonal to $[e^{i\psi_k},e^{i\psi_l}]$. The gradient $\nabla d_f$ of \eqref{eq:df} can then be rewritten as
\[
  \nabla d_f
  = \frac{i}{\pi(1-\abs{z}^2)} \oint_C  w(\psi) de^{i\psi}
  = \frac{i}{\pi(1-\abs{z}^2)} \Biggl( \int_{\psi_l}^{\psi_k} w(\psi) d e^{i\psi}
                                     + \int_{\psi_k}^{\psi_l} w(\psi) d e^{i\psi} \Biggr)
  = \nabla d^1_f + \nabla d^2_f.
\]
The projection $\nabla d^1_f$ onto $[e^{i\psi_l},e^{i\psi_k}]$ can be computed as
\begin{align*}
  \bigl\langle i e^{i \psi_a}, \nabla d^1_f \bigr\rangle
  &= \biggl\langle i e^{i \psi_a},
                   \frac{i}{\pi(1-\abs{z}^2)} \int_{\psi_l}^{\psi_k} w(\psi) d e^{i\psi} \biggr\rangle\\
  &= \frac{-1}{\pi (1-\abs{z}^2)}
       \biggl\langle e^{i\psi_a},
                     \int_{\psi_l}^{\psi_k} i w(\psi) e^{i\psi} d\psi \biggr\rangle\\
  &= \frac{-1}{\pi (1-\abs{z}^2)}
       \int_{\psi_l}^{\psi_k} \Re\bigl(-i w(\psi) e^{i(\psi_a-\psi)}\bigr) d\psi\\
  &= \frac{-1}{\pi (1-\abs{z}^2)}
       \int_{\psi_l}^{\psi_a} w(\psi) \sin(\psi_a-\psi) d\psi
   + \frac{-1}{\pi (1-\abs{z}^2)}
       \int_{\psi_a}^{\psi_k} w(\psi) \sin(\psi_a-\psi) d\psi\\
  &= \frac{-1}{\pi (1-\abs{z}^2)}
     \int_{\psi_l}^{\psi_a} \bigl( w(\psi) - w(2\psi_a-\psi) \bigr) \sin(\psi_a-\psi) d\psi.
\end{align*}
Now observe that for $\psi\in(\psi_l,\psi_a)$ we have $2\psi_a-\psi\in(\psi_a,2\psi_a-\psi_l)=(\psi_a,\psi_k)$, therefore
\[
  w(\psi) \le w(2\psi_a-\psi),
\]
since $w$ is monotonically increasing in $(\psi_l,\psi_k]$, while $\sin(\psi_a-\psi)>0$ as $\psi_a-\psi\in(0,\psi_a-\psi_l)\subset(0,\pi)$. Overall, we conclude
\[
  \int_{\psi_l}^{\psi_a} \bigl( w(\psi)-w(2\psi_a-\psi) \bigr) \sin(\psi_a-\psi) d\psi \le 0,
\]
hence $\langle ie^{i\psi_a},\nabla d^1_f \rangle \ge 0$. The equality holds iff $w(\psi)$ is constant on $(\psi_l,\psi_k)$, which happens only in the case $k=l+1$, when the harmonic measure on the arc $(\theta_l,\theta_k)$ is reduced to a single coordinate.

Similarly, the projection $\nabla d^2_f$ of the gradient $\nabla d_f$ of~\eqref{eq:df} onto $[e^{i\psi_l},e^{i\psi_k}]$ is
\[
  \bigl\langle i e^{i\psi_a}, \nabla d^2_f \bigr\rangle
  = \biggl\langle i e^{i\psi_a},
                  \frac{i}{\pi(1-\abs{z}^2)} \int_{\psi_k}^{\psi_l} w(\psi) de^{i\psi} \biggr\rangle
  \ge 0.
\]
Again, equality holds iff $w(\psi)$ is constant on $(\psi_k,\psi_l)$, or equivalently $l=k+1$. Therefore, in total $\langle ie^{i\psi_a},\nabla d_f\rangle>0$, so $\nabla d_f\ne0$, as long as $n>2$.

If $n=2$, then $w$ is constant over the two integral intervals, so $\nabla d^1_f=\nabla d^2_f=0$.
\end{proof}

\begin{remark}
The proof shows that when $n=2$, the gradient has zero projection onto $[e^{i\psi_l},e^{i\psi_k}]$, in other words, the gradient is always in the orthogonal direction, and it can vanish at places. In fact, Equation~\eqref{eq:phi_j} implies that each point on the circle through $\theta_1,\theta_2$ and the origin has the same harmonic measure as the origin and therefore has 0 distance to the origin and vanishing gradient.
\end{remark}

We now generalize Theorem \ref{theorem:DGTD} to arbitrary simply-connected domains. The key is the conformal invariance of the $f$-divergence distances, which is implied by the well-known conformal invariance of the harmonic measure of $E\subset\Omega$,
which is defined (in the usual way) in terms of the Poisson kernel as
\[
  \phi^{\Omega}(E,z) = \int_E P^{\Omega}(t,z) dt.
\]

\begin{theorem}[Conformal Invariance of Harmonic Measure~\cite{Garnett:2005:HM}] \label{theorem:conformal_invariance_harmonic}
Let $\Omega_1$ and $\Omega_2$ be two simply-connected domains and $C\colon\Omega_1\to\Omega_2$ be a conformal map between them such that $C(\partial\Omega_1)=\partial\Omega_2$. 
Then,
\[
  \phi^{\Omega_2} (C(E),C(z)) = \phi^{\Omega_1}(E,z).
\]
\end{theorem}

Let us now turn to the $f$-divergence distances.

\begin{theorem}[Conformal Invariance of $f$-divergence] \label{theorem:conformal_invariance_divergence}
Let $C\colon\Omega_1\to\Omega_2$ be a conformal map and $d_f^{\Omega_1}\colon\Omega_1\times\Omega_1\to\RR$ and $d_f^{\Omega_2}\colon\Omega_2\times\Omega_2\to\RR$ the $f$-divergence distance functions of $\Omega_1$ and $\Omega_2$, respectively. Then,
\[
  d_f^{\Omega_2}(C(x),C(y)) = d_f^{\Omega_1}(x,y).
\]
\end{theorem}

\begin{proof}
Using the definition in \eqref{eq:df-discrete} and Theorem~\ref{theorem:conformal_invariance_harmonic}, we get
\begin{align*}
  d_f^{\Omega_2}(C(x),C(y)) 
  &= \sum_{j=1}^n \phi^{\Omega_2} (C(E_j),C(x)) 
                  f\biggl( \frac{\phi^{\Omega_2}(C(E_j),C(y))}{\phi^{\Omega_2}(C(E_j),C(x))} \biggr)\\
  &= \sum_{j=1}^n \phi^{\Omega_1}(E_j,x)
                  f\biggl( \frac{\phi^{\Omega_1}(E_j,y)}{\phi^{\Omega_1}(E_j,x)} \biggr)\\
  &= d_f^{\Omega_1}(x,y)
\end{align*}
\end{proof}

Now we are in the position to prove Theorem~\ref{theorem:DGT} for any simply-connected domain by conformally mapping it to the unit disk.

\begin{proof} (of Theorem \ref{theorem:DGT})
For a given target point $y\in\Omega$, the Riemann Mapping Theorem~\cite{Garnett:2005:HM} implies that there exists a conformal map $C\colon\Omega\to D$, where $D$ is the unit disk, such that $\partial\Omega$ is mapped to $\partial D$ and $C(y)=0$. By Theorem~\ref{theorem:conformal_invariance_divergence},
\[
  d_f^{\Omega}(x,y) = d_f^{D}(C(x),0).
\]
The gradients of the two distance functions with respect to their first argument are related by the chain rule for holomorphic functions. Dropping the second (fixed) argument, we get
\[
  \nabla d_f^\Omega (x) = \nabla d_f^D(C(x)) \overline{\frac{\partial C}{\partial z}} (x).
\]
Since the derivative of a conformal map never vanishes, we have
\[
  \nabla d_f^\Omega (x) = 0 
    \qquad\text{iff}\qquad
  \nabla d_f^D(C(x)) = 0
    \qquad\text{iff}\qquad
  C(x)=0
    \qquad\text{iff}\qquad
  x=y.
\]
\end{proof}

\begin{figure}
  \parbox{.3\textwidth}{\centering\includegraphics[height=1.8in]{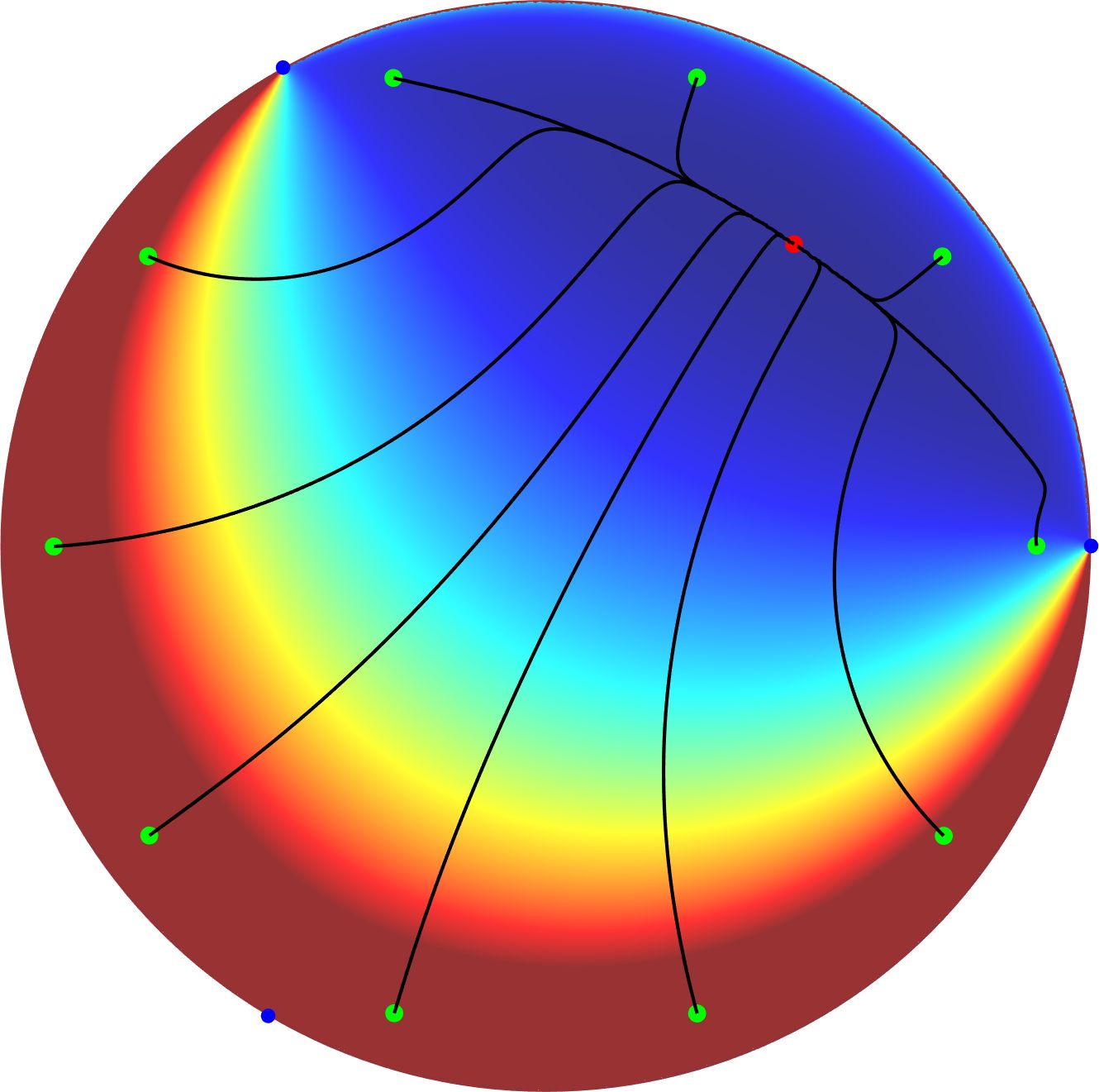}\\ KL, $n=3$}\hfill
  \parbox{.3\textwidth}{\centering\includegraphics[height=1.8in]{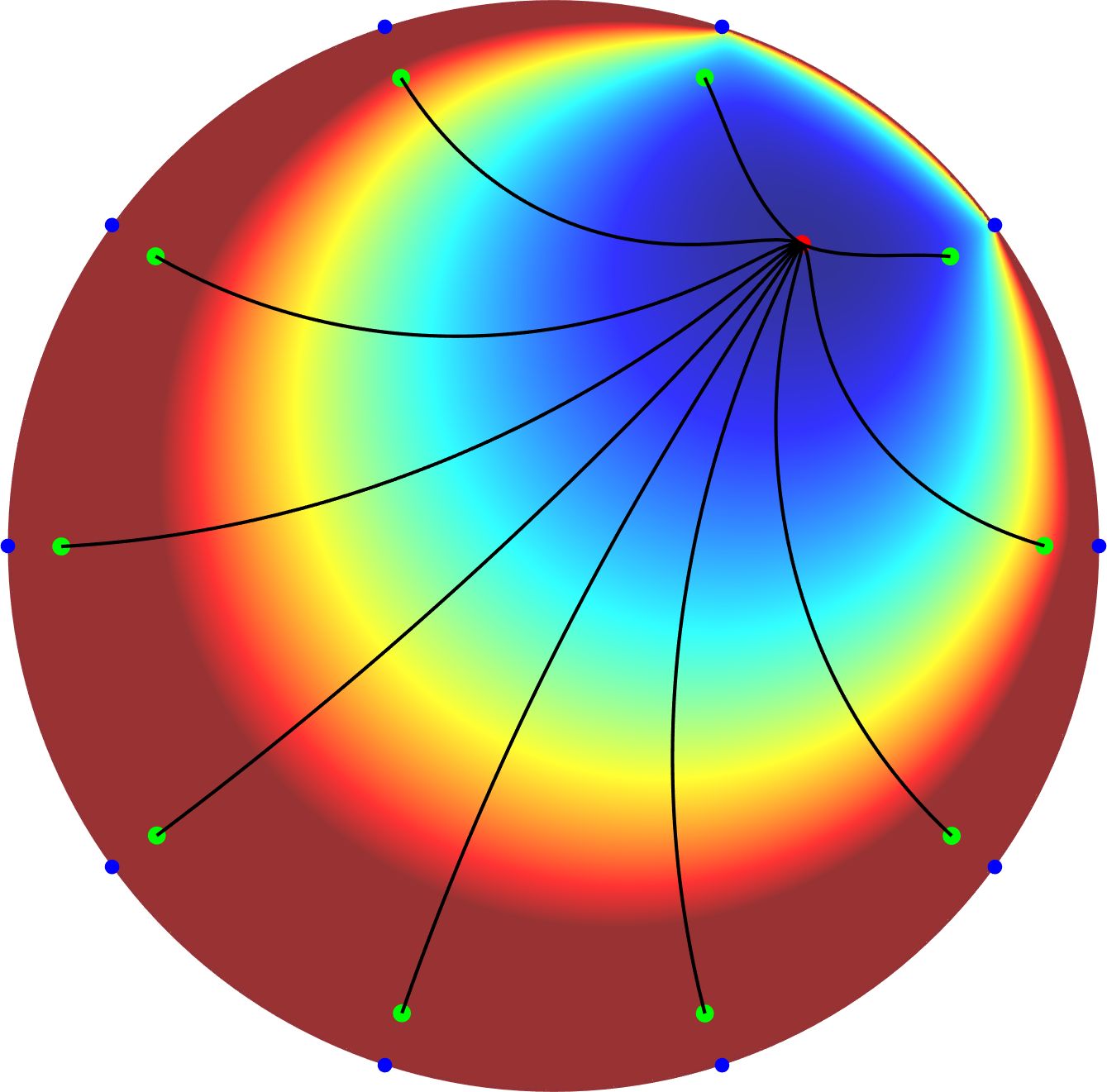}\\ KL, $n=10$}\hfill
  \parbox{.3\textwidth}{\centering\includegraphics[height=1.8in]{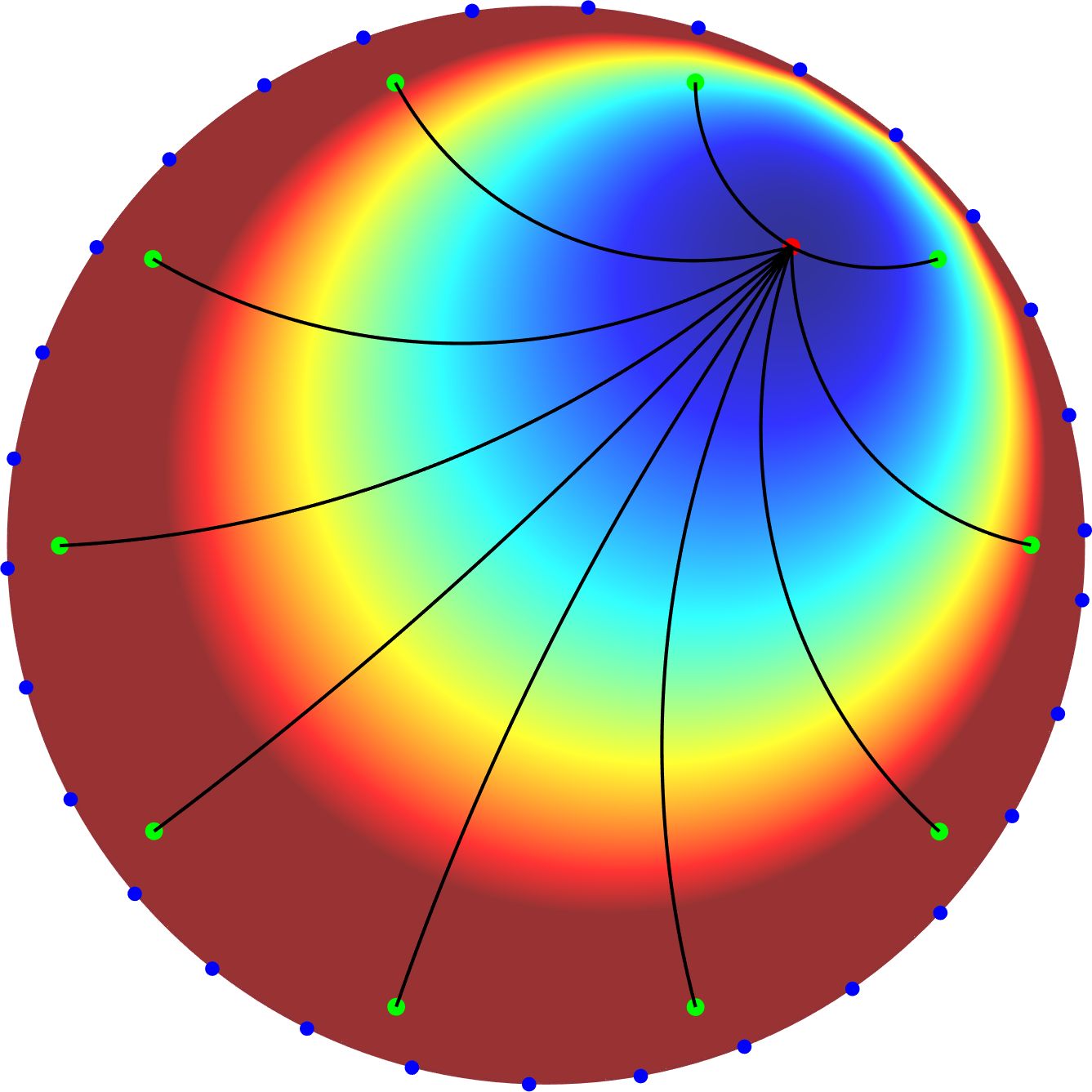}\\ KL, $n=30$}\hfill
  \parbox{.02\textwidth}{\centering\includegraphics[height=1.8in]{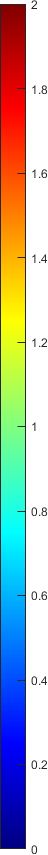}\\ ~}\\[1ex]
  \parbox{.3\textwidth}{\centering\includegraphics[height=1.8in]{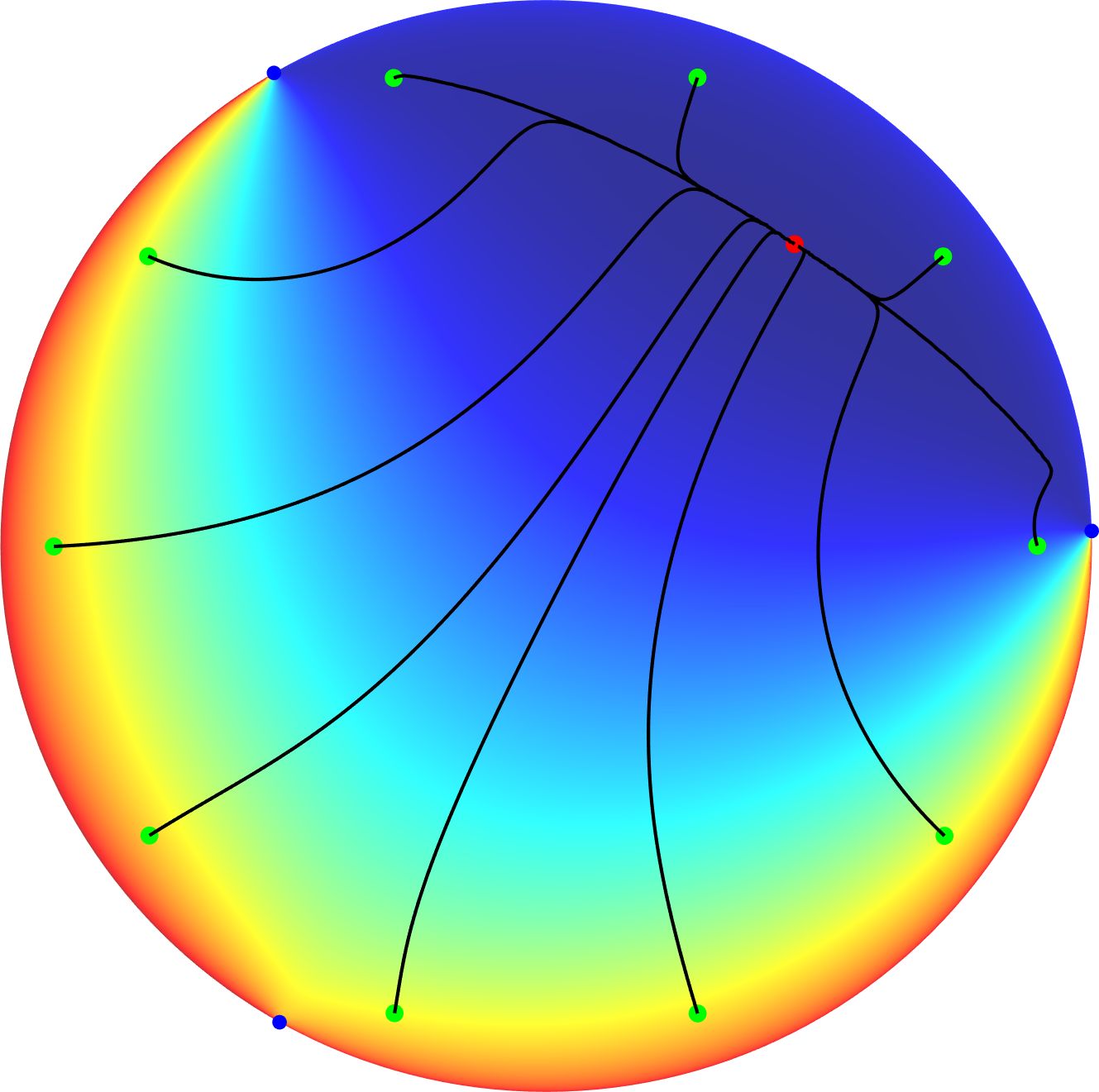}\\ H, $n=3$}\hfill
  \parbox{.3\textwidth}{\centering\includegraphics[height=1.8in]{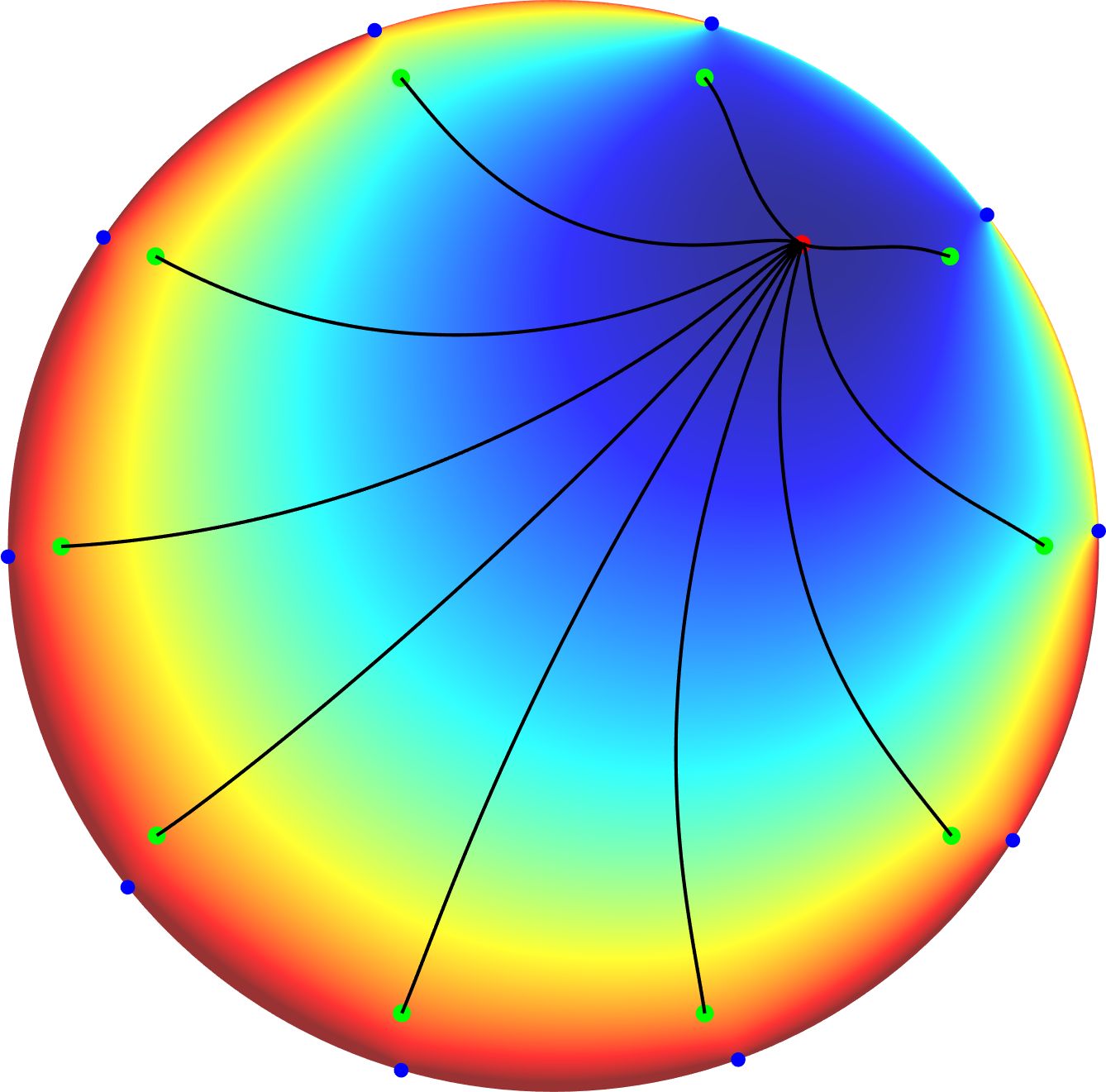}\\ H, $n=10$}\hfill
  \parbox{.3\textwidth}{\centering\includegraphics[height=1.8in]{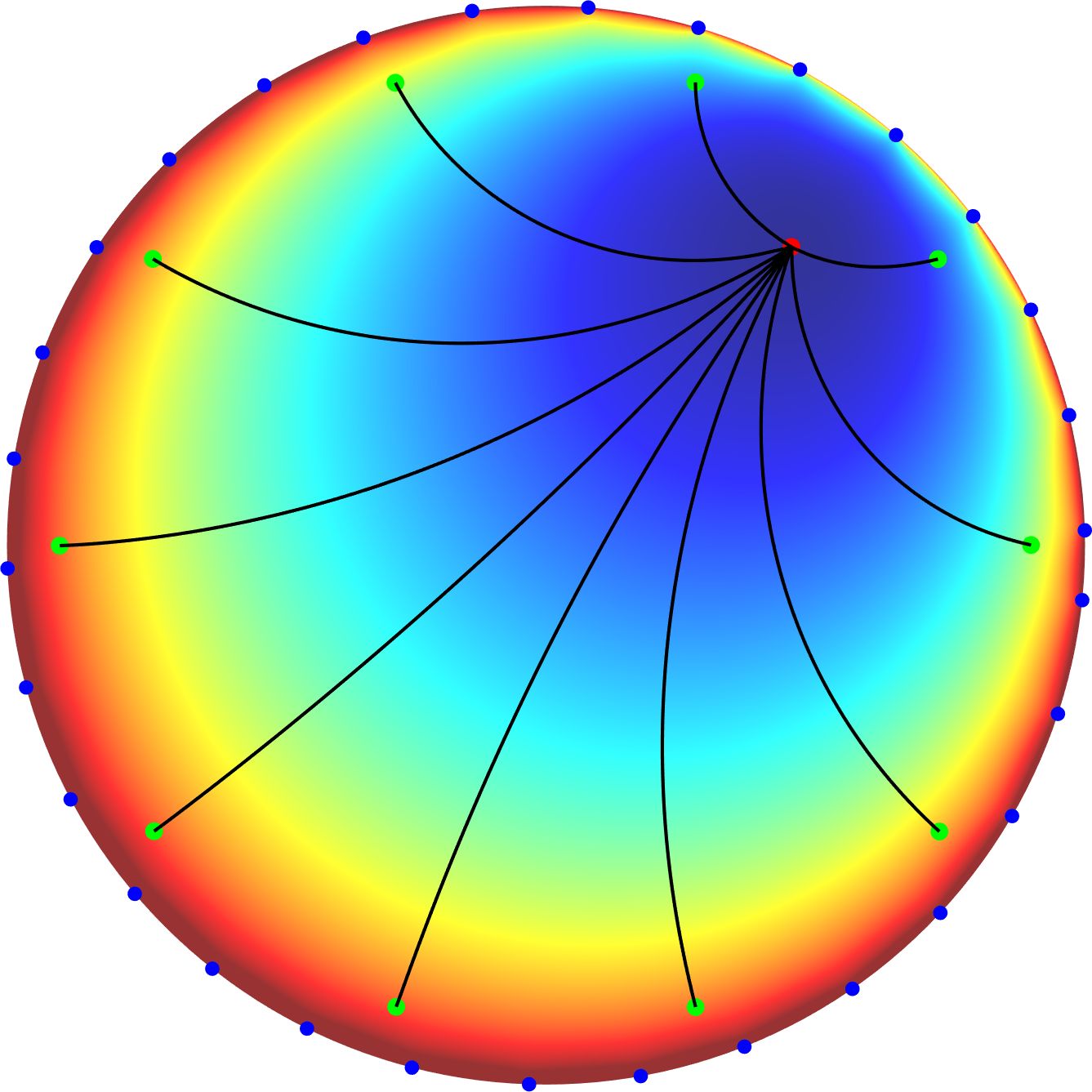}\\ H, $n=30$}\hfill
  \parbox{.02\textwidth}{\centering\includegraphics[height=1.8in]{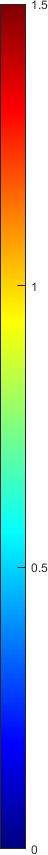}\\ ~}\par
  \caption{Path planning (from the green source points to the red target point) using reduced $f$-divergence distances for different numbers $n$ of reduced coordinates. The reduced coordinates are generated by a uniform partition of the domain boundary, marked by the blue points. \emph{Top:} Using the KL distance function. \emph{Bottom:} Using the H distance function. \emph{Left to right:} Increasing the number of reduced coordinates. Note that the KL paths are in general not identical to the H paths, but both converge to the ``continuous'' paths of Figure~\ref{fig:grad-descent} as $n$ increases.}
  \label{fig:convex-KL-H}
\end{figure}

\begin{figure}
  \parbox{.3\textwidth}{\centering\includegraphics[height=2.2in]{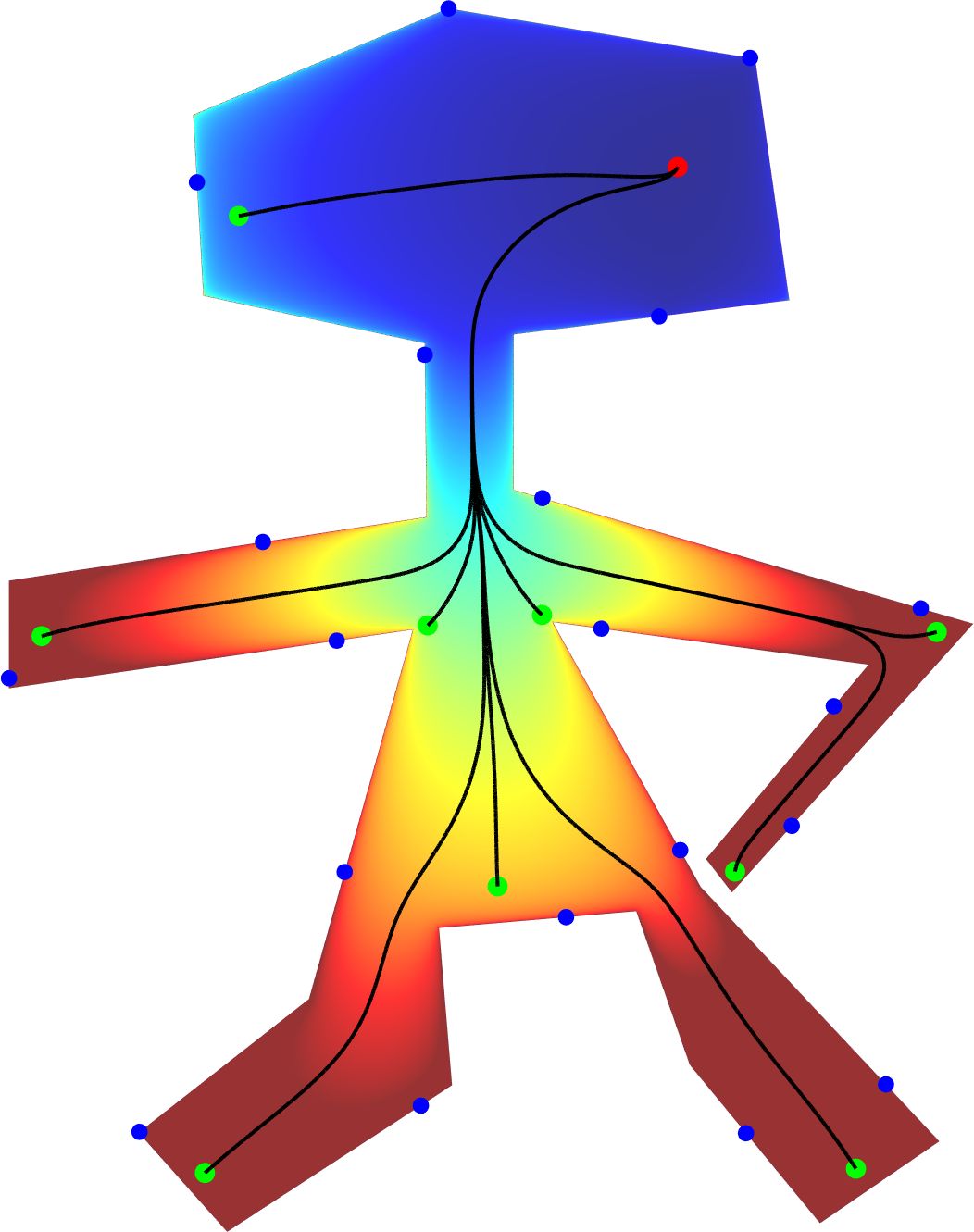}\\ KL, $n=20$}\hfill
  \parbox{.3\textwidth}{\centering\includegraphics[height=2.2in]{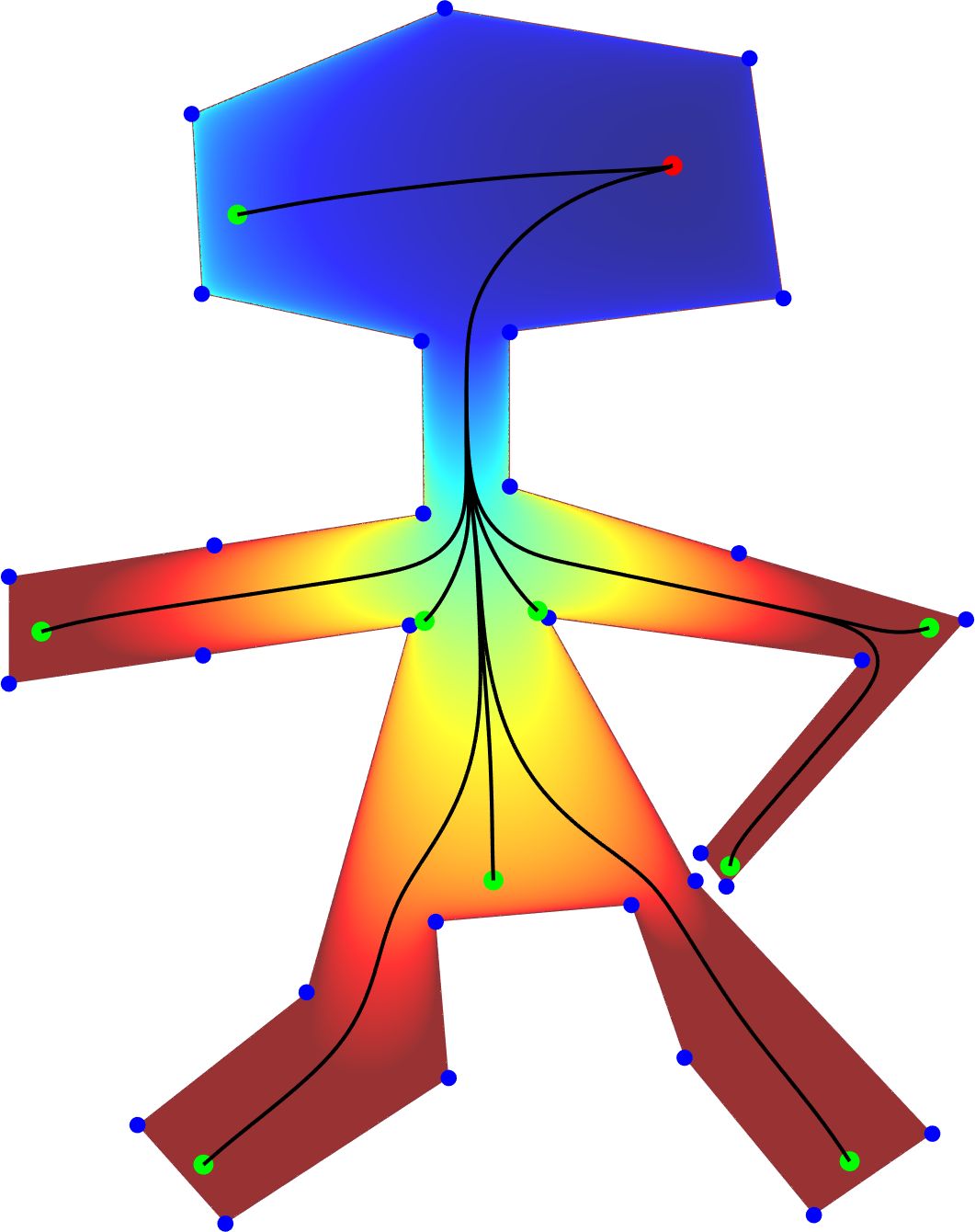}\\ KL, $n=30$}\hfill
  \parbox{.3\textwidth}{\centering\includegraphics[height=2.2in]{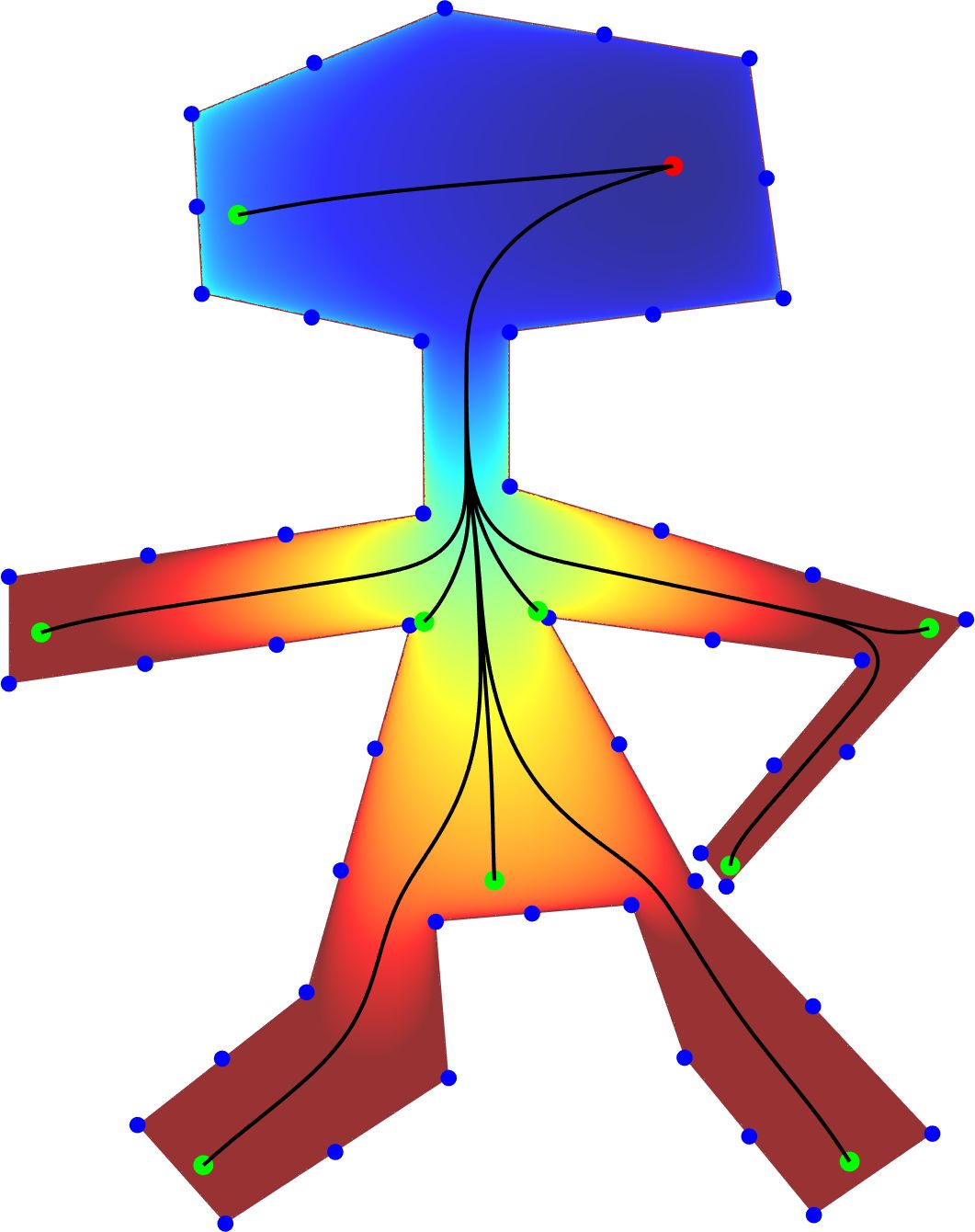}\\ KL, $n=50$}\hfill
  \parbox{.04\textwidth}{\centering\includegraphics[height=2.2in]{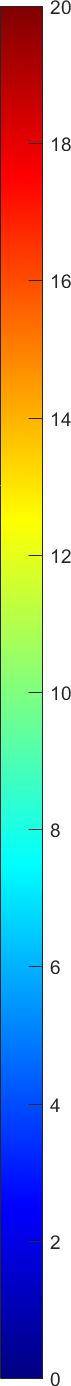}\\ ~}\\[1ex]
  \parbox{.3\textwidth}{\centering\includegraphics[height=2.2in]{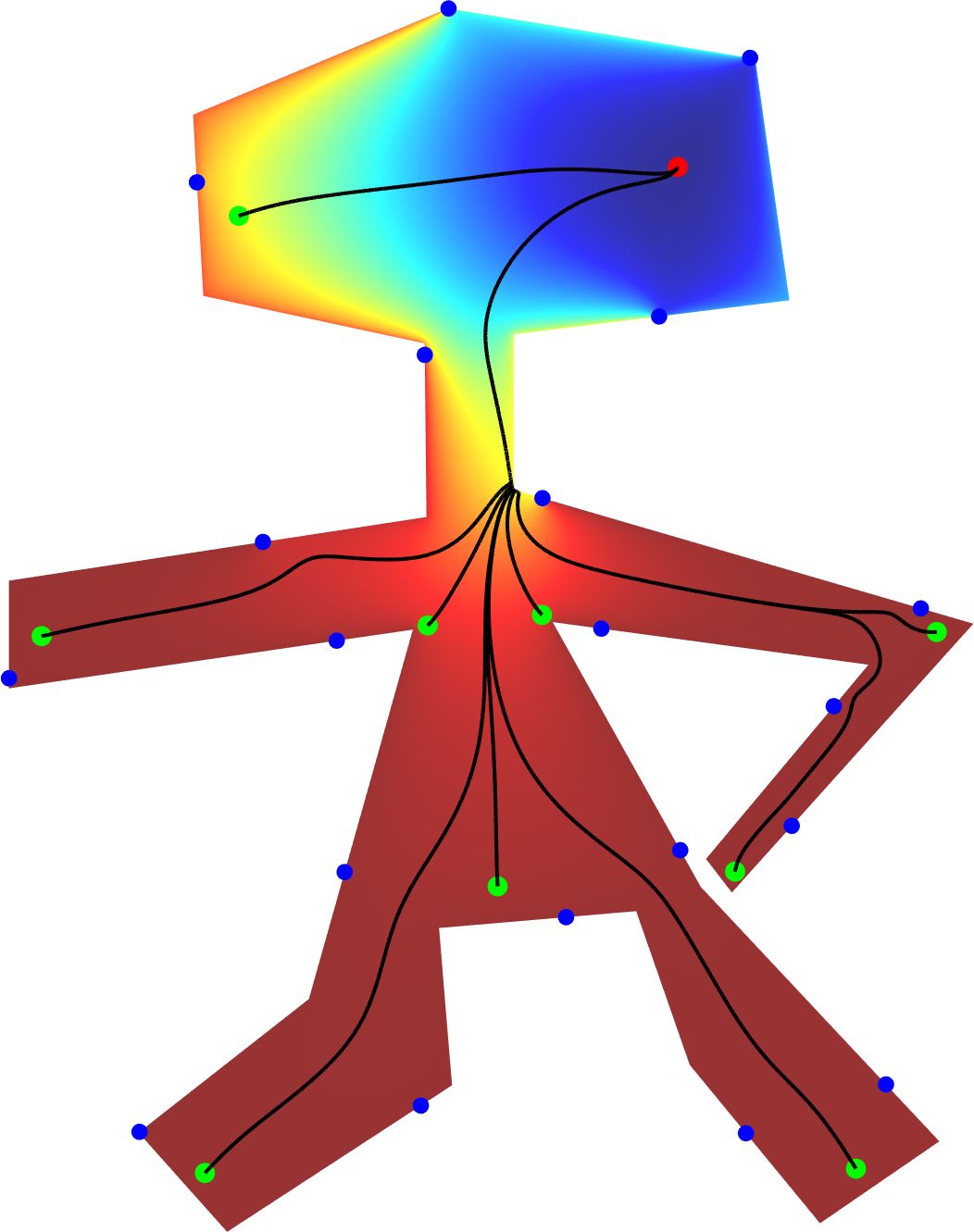}\\ H, $n=20$}\hfill
  \parbox{.3\textwidth}{\centering\includegraphics[height=2.2in]{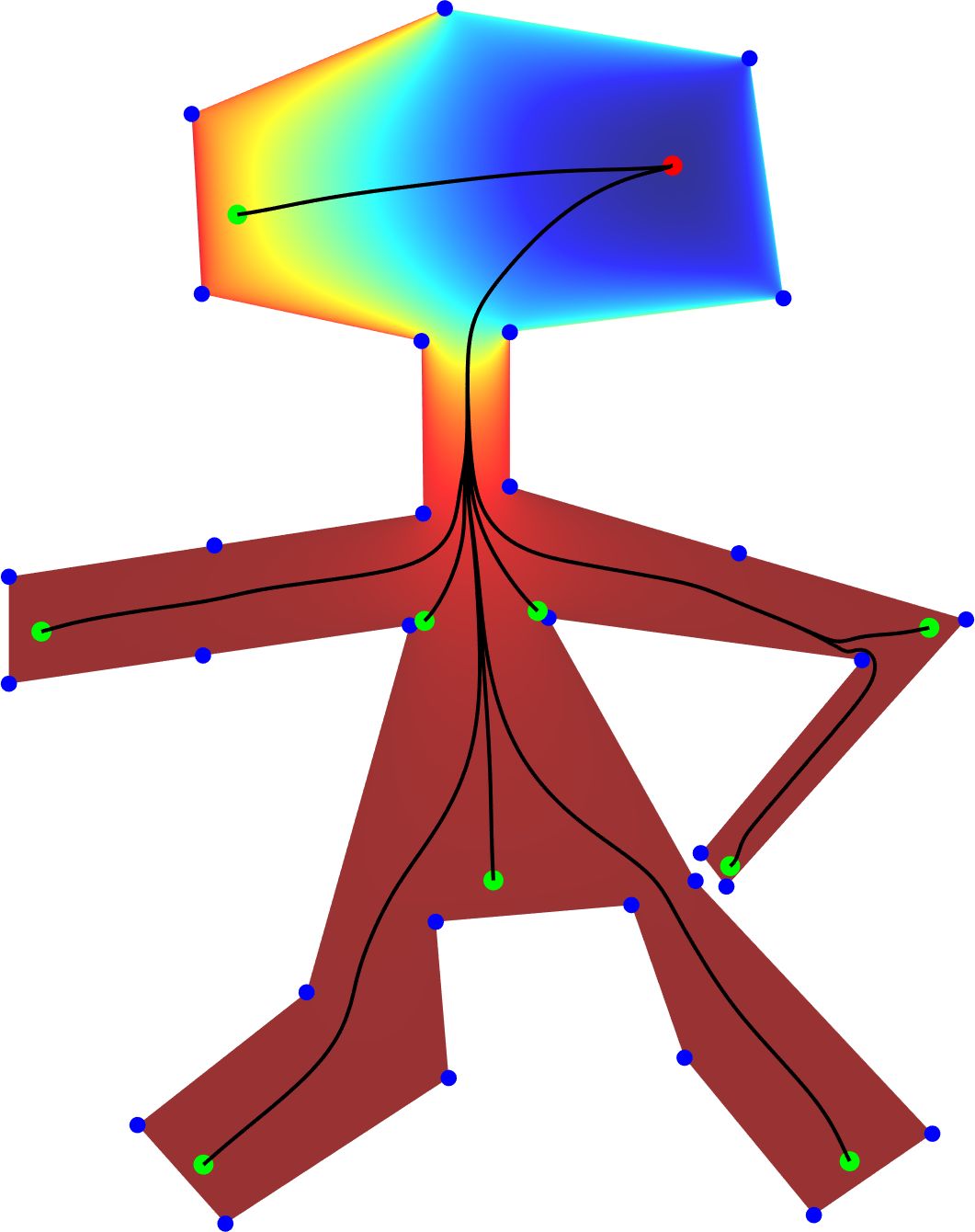}\\ H, $n=30$}\hfill
  \parbox{.3\textwidth}{\centering\includegraphics[height=2.2in]{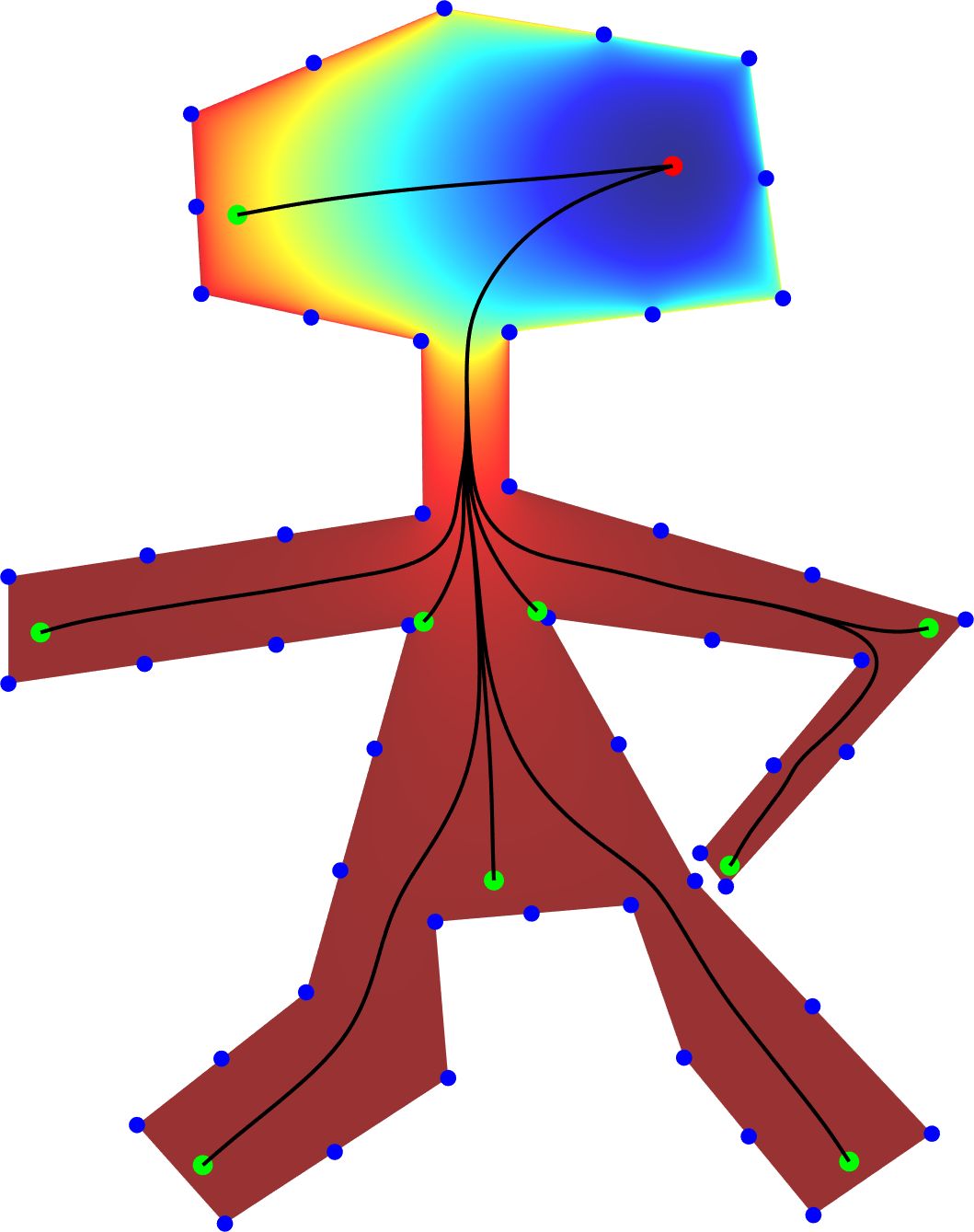}\\ H, $n=50$}\hfill
  \parbox{.04\textwidth}{\centering\includegraphics[height=2.2in]{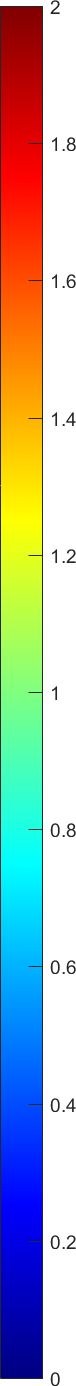}\\ ~}\par
  \caption{Same as Figure~\ref{fig:convex-KL-H} for a non-convex polygonal domain. Convergence to the ``continuous'' paths of Figure~\ref{fig:grad-descent} is more rapid for the KL distance than for the H distance.}
  \label{fig:concave-KL-H}
\end{figure}

\subsection{Examples}

Figures~\ref{fig:convex-KL-H} and~\ref{fig:concave-KL-H} show what happens to the paths shown in Figure~\ref{fig:grad-descent} when the reduced $f$-divergence distance~\eqref{eq:df-discrete} is used instead of the original continuous distance~\eqref{eq:df}. The coordinates are reduced by uniformly partitioning the boundary of the domain. For polygonal boundaries (as in Figure~\ref{fig:concave-KL-H}), it seems natural to partition according to the polygon edges, namely at least one reduced coordinate per edge. Long edges are further partitioned uniformly until the partition length is less than some threshold. Since reduced coordinates lose the invariance (to $f$) property, different paths are obtained for the reduced KL and H distances. These are typically not as natural as the original paths, especially when $n$ is small. Obviously, for very large $n$, the resulting paths approach the invariant continuous case of Figure~\ref{fig:grad-descent}.

%%%%%%%%%%%%%%%%%%%%%%%%%%%%%%%%%%%%%%%%%%%%%%%%%%%%%%%%%%%%%%%%%%%%%%%%%%%%

\section{Computing Reduced Coordinates}

In practice, computing the reduced coordinates requires the solution of linear Laplace equations on a dense triangulation of the domain, resulting in the reduced coordinates for all $k$ points of the triangulation. Using $n$ reduced coordinates implies $n$ linear systems, all having the same $k\times k$ matrix, but with different right hand sides (corresponding to the boundary conditions),
\begin{equation}\label{eq:rhs}
  A\Phi_j = b_j, \qquad j=1,\dots,n.
\end{equation}
The matrix $A$ is the common matrix, $\Phi_j$ is the unknown $k$-vector of the $j$-th reduced coordinates, and $b_j$ is the boundary condition vector corresponding to the $j$-th coordinate, the binary indicator of the segment $[t_j,t_{j+1}]$ of points on the boundary loop,
\begin{equation}\label{eq:binary-indicator}
  b_{jr} = \begin{cases}
             1, & \text{if $r\in[t_j,t_{j+1}]$,}\\
             0, & \text{otherwise.}
           \end{cases}
\end{equation}
The matrix $A$ is the standard Laplacian operator on the triangulation, a sparse symmetric (positive definite) matrix with the so-called cotangent weights~\cite{Pinkall:1993:CDM} corresponding to edges of the triangulation, which are always positive if a constrained Delaunay triangulation of the domain~\cite{deBerg:2008:CGA,Shewchuk:1996:TEA} is used. Since all $n$ linear systems share the common matrix $A$, they may be solved efficiently by pre-factoring $A$ and performing back-substitution for the different $b_j$~\cite{Press:2007:NRI}.

%%%%%%%%%%%%%%%%%%%%%%%%%%%%%%%%%%%%%%%%%%%%%%%%%%%%%%%%%%%%%%%%%%%%%%%%%%%%

\section{Discrete Routing Graph} \label{section:discrete}

We now show how to make path-planning using reduced $f$-divergence distance even more practical. This means allowing the definition of a finite and reasonably sized set of $m$ \emph{sites} $S$ to be used in the domain $\Omega$, where the path planner moves only along (the straight line) edges of a graph connecting these sites. To mimic the gradient-descent path in this discrete world, a network (graph) is defined on $S$, such that~\eqref{eq:neighbours} -- the discrete analog of~\eqref{eq:grad_df=0} -- holds.

To achieve this, we follow the logic of Bose and Morin~\cite{Bose:2004:ORI} (also used by Ben-Chen et al.~\cite{BenChen:2011:DCO}), who show that the \emph{Delaunay triangulation}~\cite{deBerg:2008:CGA} of $S$ supports greedy routing on the convex hull of $S$ using the simple Euclidean ($L_2$) distance between points in the plane. The reason that the Delaunay triangulation has this property is because it is the dual to the Euclidean \emph{Voronoi diagram} of $S$, namely two sites are connected by an edge iff their two corresponding Voronoi cells share a common edge. The proof that greedy routing works relies on the fact that the Euclidean distance is a metric, and its Voronoi cells are convex polygons. Since the Voronoi diagram of $S$ using a reduced $f$-divergence distance is more complicated, the condition must be modified, requiring the concept of a \emph{local} Voronoi cell of a site in a network:

\begin{definition}[Local Voronoi cell]
Let $G$ be a graph on $S$, which in turn is a set of $m$ sites sampled in $\Omega$. The \emph{local Voronoi cell} of $s\in S$ relative to $G$ is
\[
  LV_G(s) = \{z\in\Omega : \forall r\in N_G(s), d_f(s,z) < d_f(r,z) \},
\]
where $N_G(s)$ is the set of neighbors of $s$ in $G.$ In other words, $LV_G(s)$ is the set of all points in $\Omega$ closer to $s$ than to any of the neighbors of $s$ in $G$.
\end{definition}

The greedy routing property now explicitly guarantees that~\eqref{eq:neighbours} is satisfied.

\begin{definition}[Greedy routing property]
The graph $G=(S,E)$ has the \emph{greedy routing property} if for every site $s\in S$, $LV_G(s)$ does not contain any site other than $s$,
\[
  \forall s,t \in S : t\in LV_G(s) \qquad\text{iff}\qquad t=s.
\]
\end{definition}

Constructing a graph $G$ on $S$ having the greedy routing property is not as straightforward as it seems. It is not sufficient to merely take the dual to the Voronoi diagram of $S$. This is because the local Voronoi cells of the reduced $f$-divergence distance may have irregular structure, including not being connected (with so-called \emph{orphan cells}), and is further made more complicated by the fact that the reduced $f$-divergence distance may be asymmetric. This is especially true when the number of coordinates $n$ is small. At the other extreme, obviously the clique graph (where all sites are connected to each other) has the greedy routing property, but this is a gross overkill, as we would like to have as sparse a graph as possible, with edges as short as possible. A planar graph would be the most desirable.

\begin{figure}
  \parbox{.3\textwidth}{\centering\includegraphics[height=2.2in]{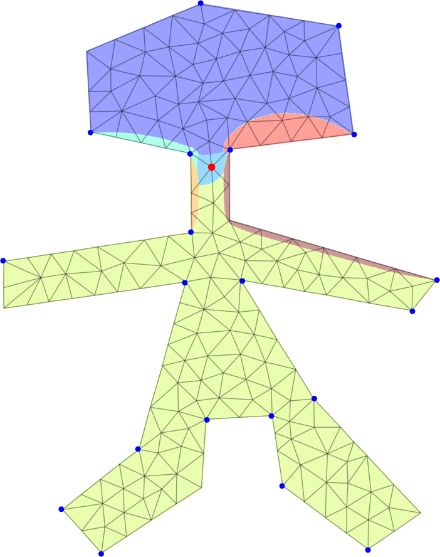}}\hfill
  \parbox{.3\textwidth}{\centering\includegraphics[height=2.2in]{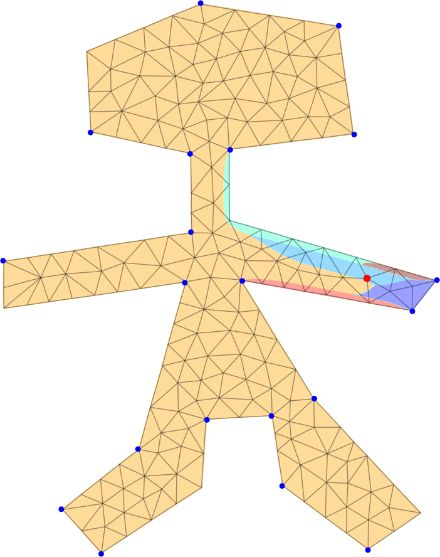}}\hfill
  \parbox{.3\textwidth}{\centering\includegraphics[height=2.2in]{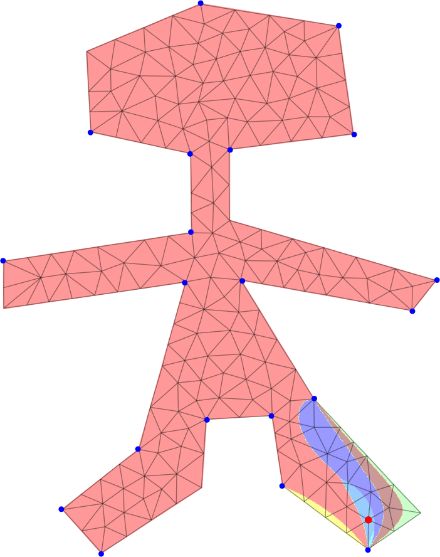}}\\[1ex]
  \parbox{.3\textwidth}{~}\hfill
  \parbox{.3\textwidth}{\centering\includegraphics[height=2.2in]{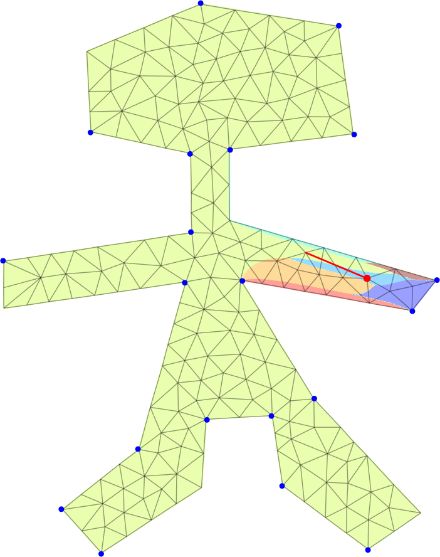}}\hfill
  \parbox{.3\textwidth}{\centering\includegraphics[height=2.2in]{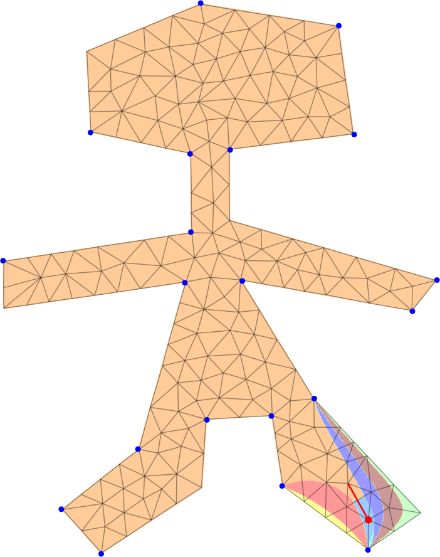}}\par
  \caption{Construction of a greedy routing graph for the KL distance function with $n=20$ reduced coordinates and $m=200$ sites sampled from a simply-connected domain. \emph{Top:} Delaunay triangulation of the sites and the local Voronoi cells of select red sites (colored in light blue) and its neighbors. Note that the local Voronoi cell may be disconnected. Should the local Voronoi cell contain sites other than the red site, the triangulation is augmented with red edges from the red site to other sites, until the local Voronoi cell contains no other sites. \emph{Bottom:} Augmented triangulations of the Delaunay triangulation above, where the shape is missing if no augmented edges were required. The underlying triangulation of the domain on which the coordinates are computed contains $k=2\times10^5$ points.}
  \label{fig:greedyRoutingGraph-simple}
\end{figure}

So the remaining question is how, given $S$ and $d_f$, to construct a greedy routing graph on $S$. This is done by an incremental algorithm. Starting with the (Euclidean) Delaunay triangulation constrained by a polygonal outline of the domain, this graph is \emph{augmented} with additional edges until it becomes greedy, namely, given a site $s\in S$, edges are added in $G$ between $s$ and other sites, until  $LV_G(s)$ contains only $s$. By definition, each addition of an edge shrinks $LV_G(s)$. Obviously this procedure eventually terminates when the worst case of $s$ being connected to all other sites is obtained. A good heuristic is to add edges between $s$ and other sites in order of increasing Euclidean distance to $s$. We call the resulting greedy graph the \emph{augmented Delaunay triangulation}. In practice, no Voronoi diagrams are computed, and the only data structure required to support the graph construction algorithm is the $m\times m$ matrix of pairwise $d_f$ distances between the $m$ sites, sampled from the dense underlying triangulation used to pre-compute the coordinates (requiring solutions to the Laplace equation). Note that this matrix is not symmetric if the reduced $f$-divergence distance is not symmetric.

%%%%%%%%%%%%%%%%%%%%%%%%%%%%%%%%%%%%%%%%%%%%%%%%%%%%%%%%%%%%%%%%%%%%%%%%%%%%

\section{Experimental Results}

\begin{figure}
  \parbox{.3\textwidth}{\centering\includegraphics[height=1.8in]{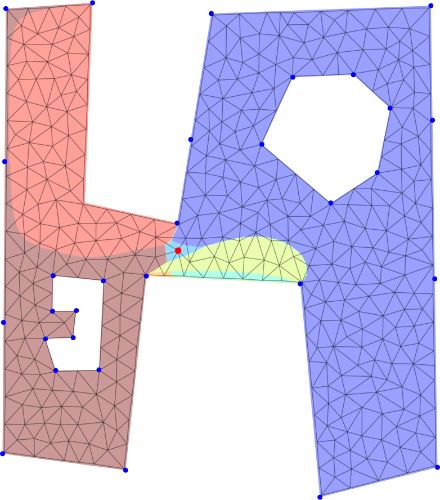}}\hfill
  \parbox{.3\textwidth}{\centering\includegraphics[height=1.8in]{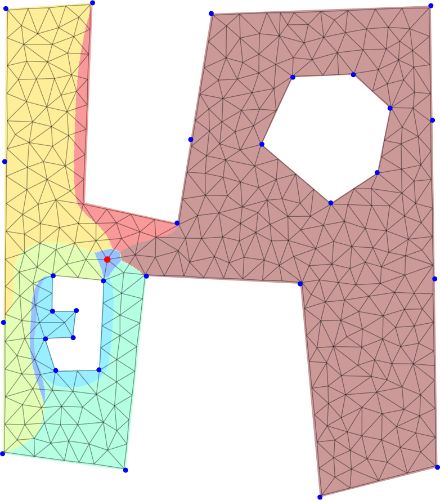}}\hfill
  \parbox{.3\textwidth}{\centering\includegraphics[height=1.8in]{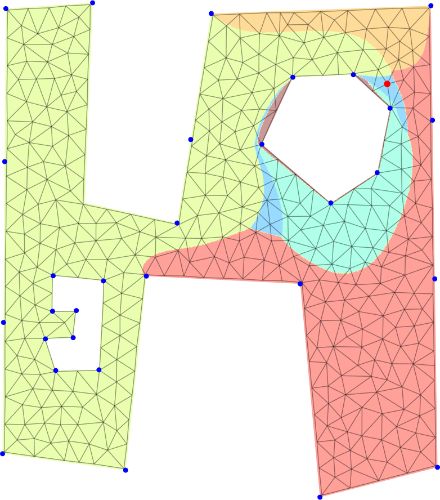}}\\[1ex]
  \parbox{.3\textwidth}{~}\hfill
  \parbox{.3\textwidth}{\centering\includegraphics[height=1.8in]{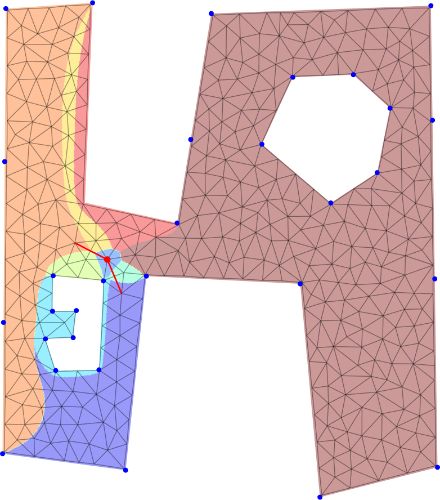}}\hfill
  \parbox{.3\textwidth}{\centering\includegraphics[height=1.8in]{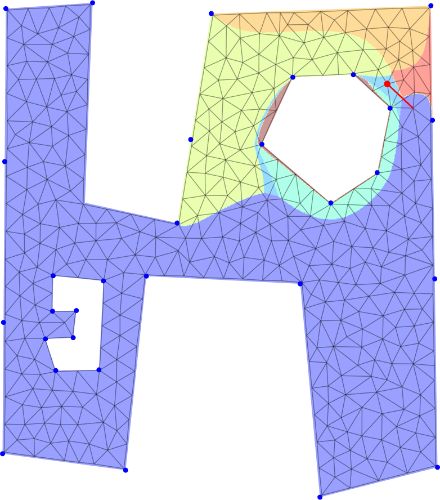}}\par
  \caption{Same as Figure~\ref{fig:greedyRoutingGraph-simple} for a multiply-connected domain with $n=30$ reduced coordinates and $m=400$ sites.}
  \label{fig:greedyRoutingGraph-hole}
\end{figure}

Figure~\ref{fig:greedyRoutingGraph-simple} and~\ref{fig:greedyRoutingGraph-hole} show parts of the greedy routing graphs generated by our algorithm on a few sites sampled in two different domains. The black edges are the initial constrained Delaunay triangulation, and the red edges are those augmented by our algorithm for three select sites to obtain the greedy routing property. Figures~\ref{fig:greedyRoutingGraph-KL-simple} and~\ref{fig:greedyRoutingGraph-KL-hole} show the greedy augmented Delaunay triangulation on three increasing sets of sites with increasing sets of reduced coordinates using a similar coloring of the edges. As before, the boundary polygon was segmented to form coordinates by placing one (blue) sample point at each polygon vertex, and then partitioning each edge to segments of equal lengths until the length of each such segment is below a predefined threshold. The results demonstrate how the Delaunay triangulation is already very close to greedy, requiring only a slight augmentation, for large values of $n$. Figure~\ref{fig:path-KL-H} shows routing trees for three target vertices on the resulting graph ($n=20$, $m=400$) for two different distance functions on the multiply-connected shape.

\begin{figure}
  \parbox[c][1.8in]{.05\textwidth}{\centering \rotatebox[origin=c]{90}{$n=27$}}\hfill
  \parbox{.23\textwidth}{\centering\includegraphics[height=1.7in]{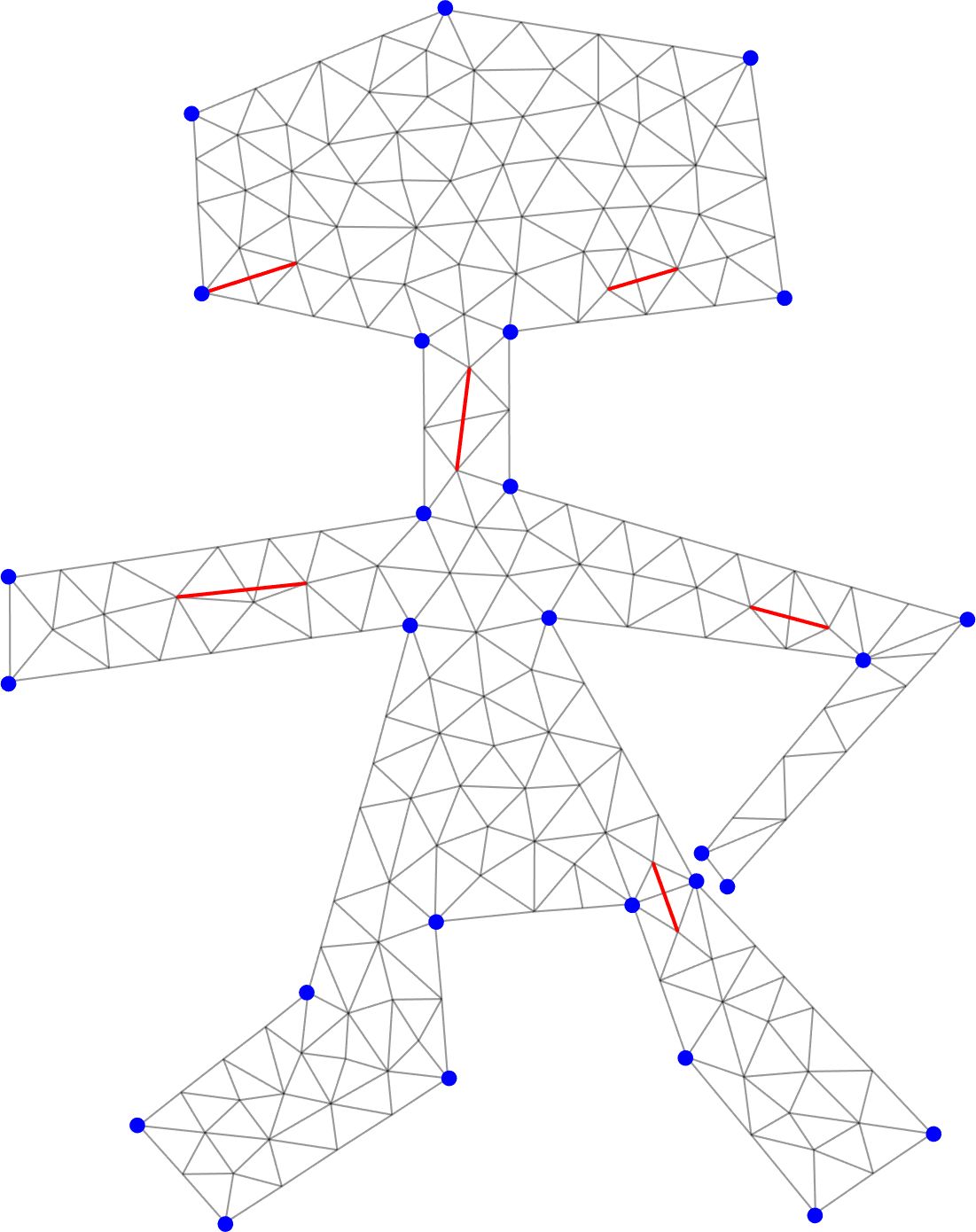}}\hfill
  \parbox{.23\textwidth}{\centering\includegraphics[height=1.7in]{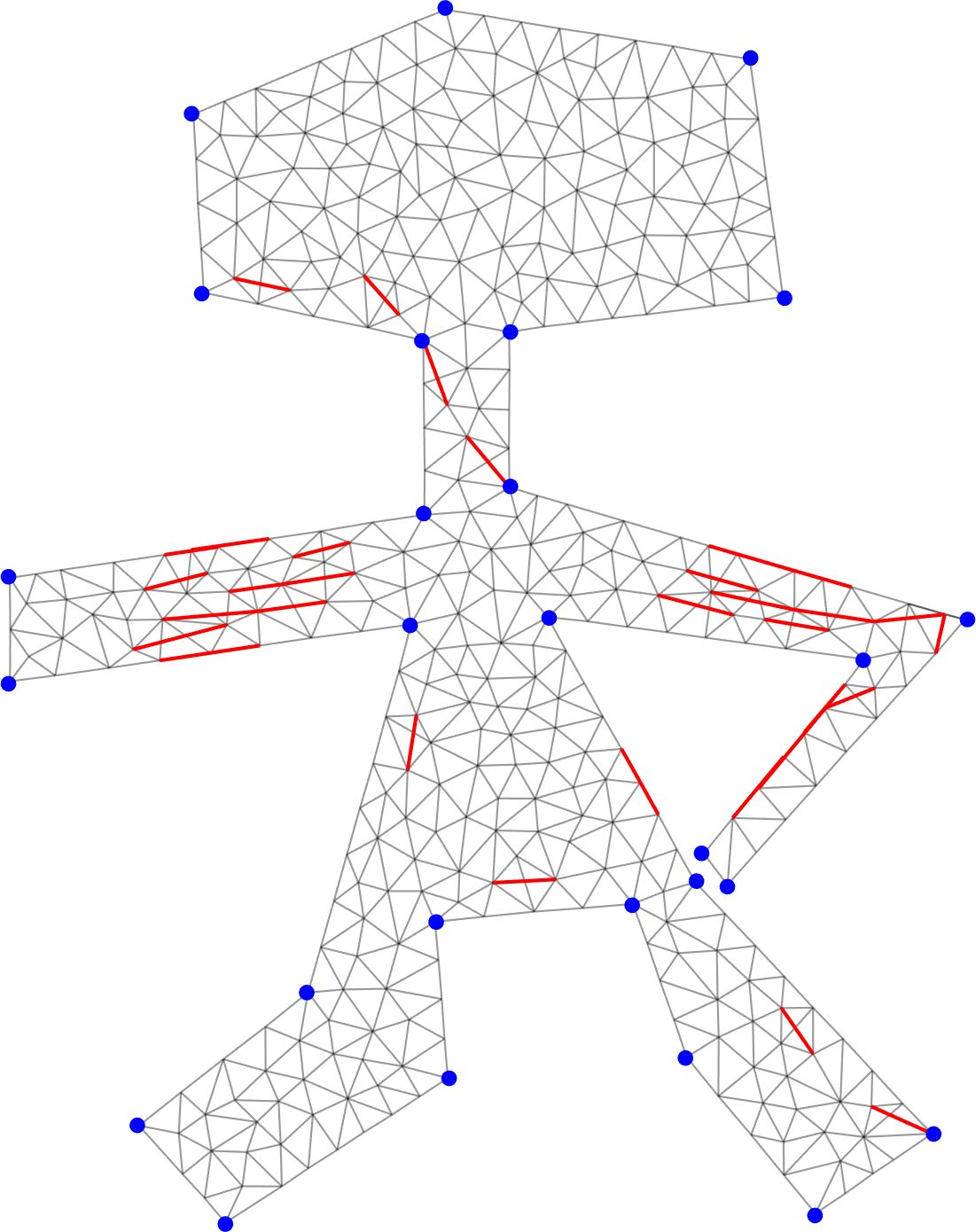}}\hfill
  \parbox{.23\textwidth}{\centering\includegraphics[height=1.7in]{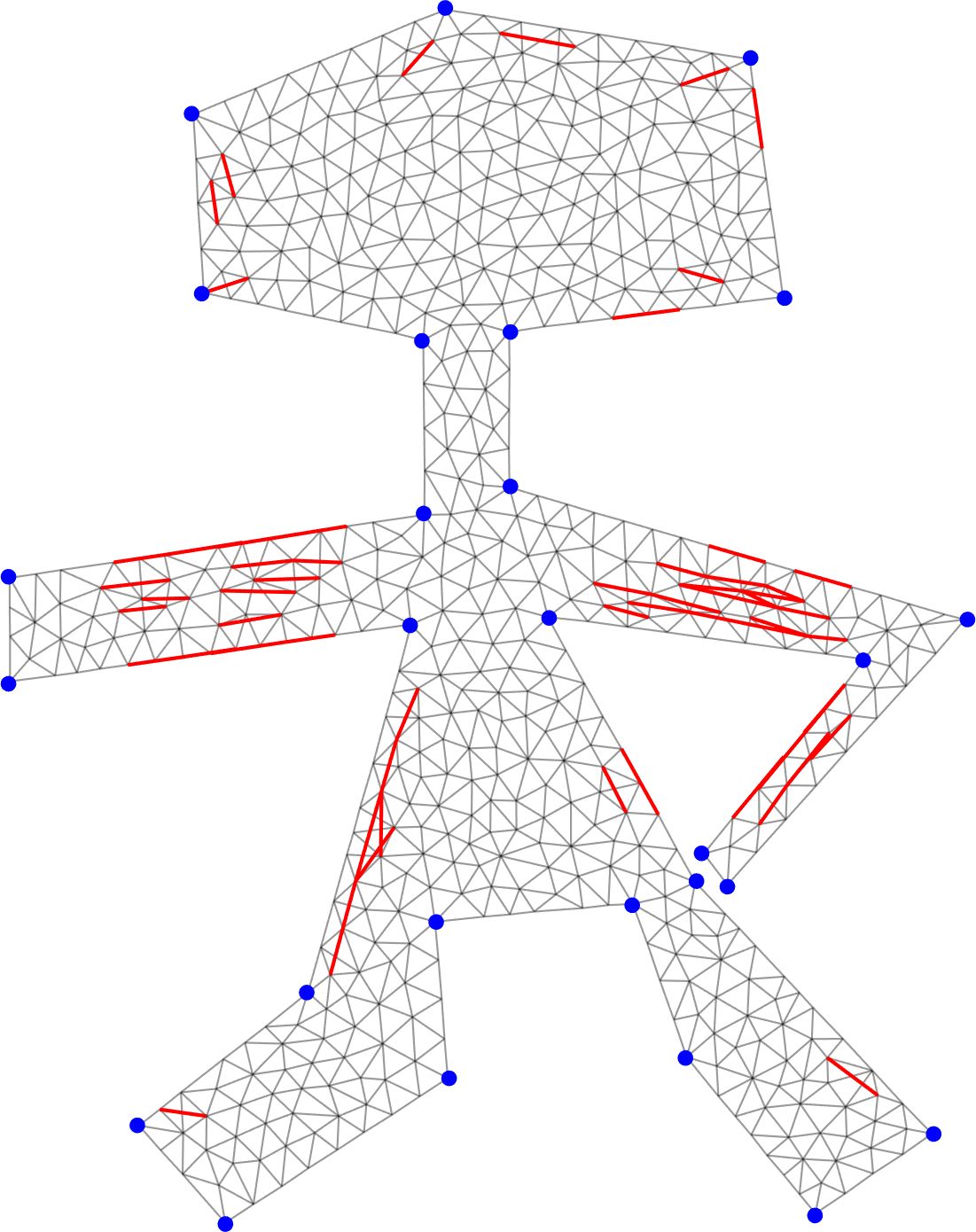}}\hfill
  \parbox{.23\textwidth}{\centering\includegraphics[height=1.7in]{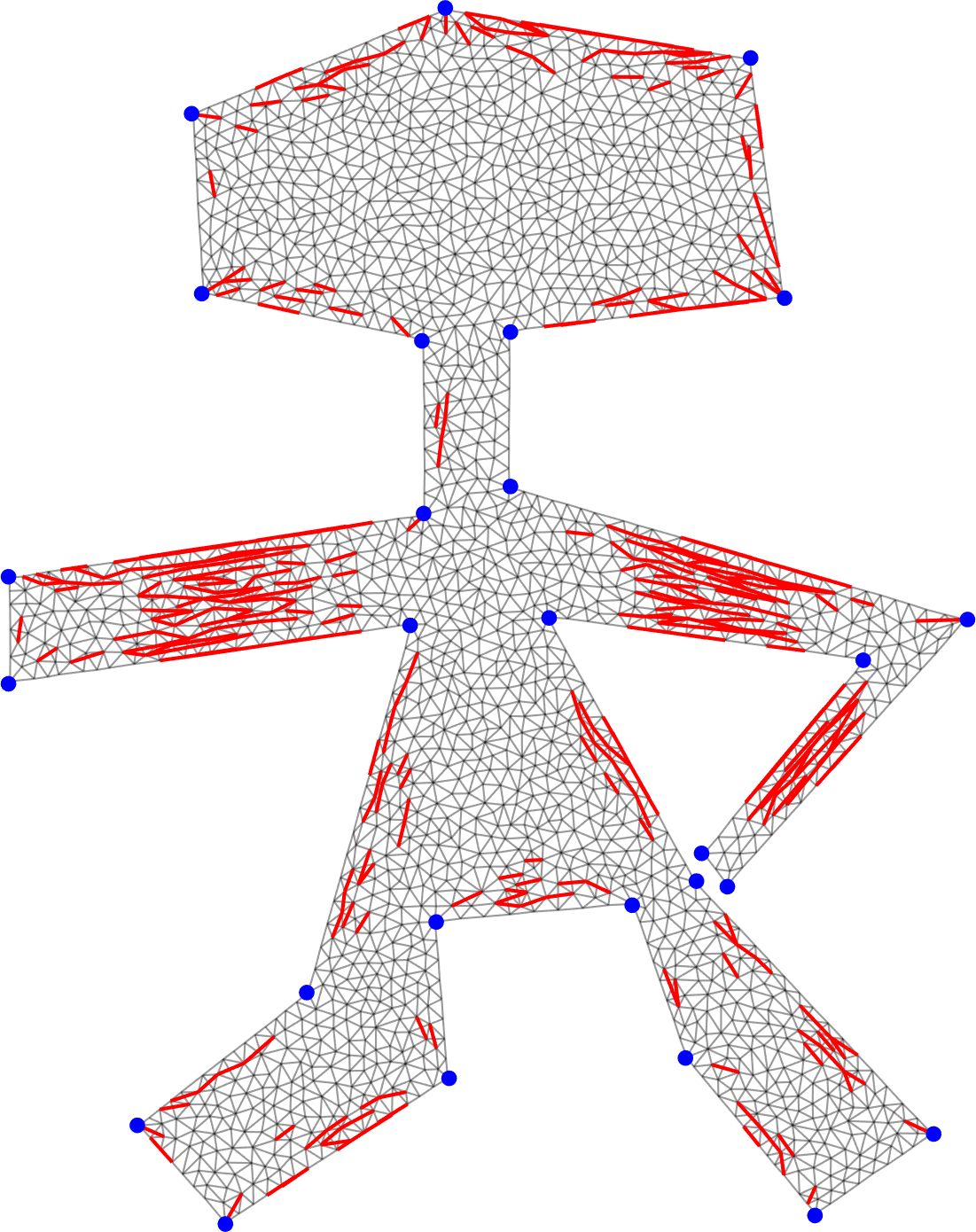}}\\[1ex]
  \parbox[c][1.7in]{.05\textwidth}{\centering \rotatebox[origin=c]{90}{$n=40$}}\hfill
  \parbox{.23\textwidth}{\centering\includegraphics[height=1.7in]{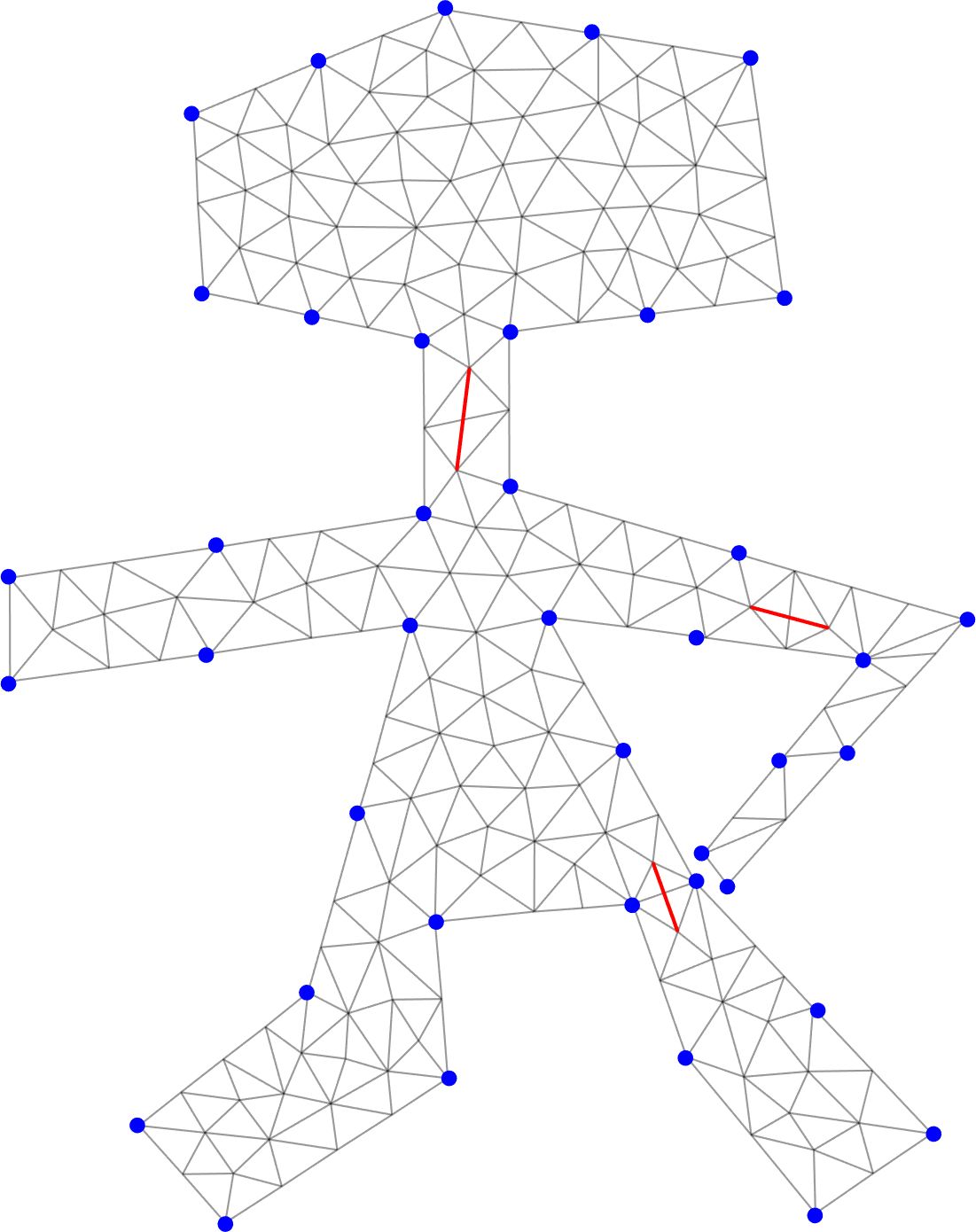}}\hfill
  \parbox{.23\textwidth}{\centering\includegraphics[height=1.7in]{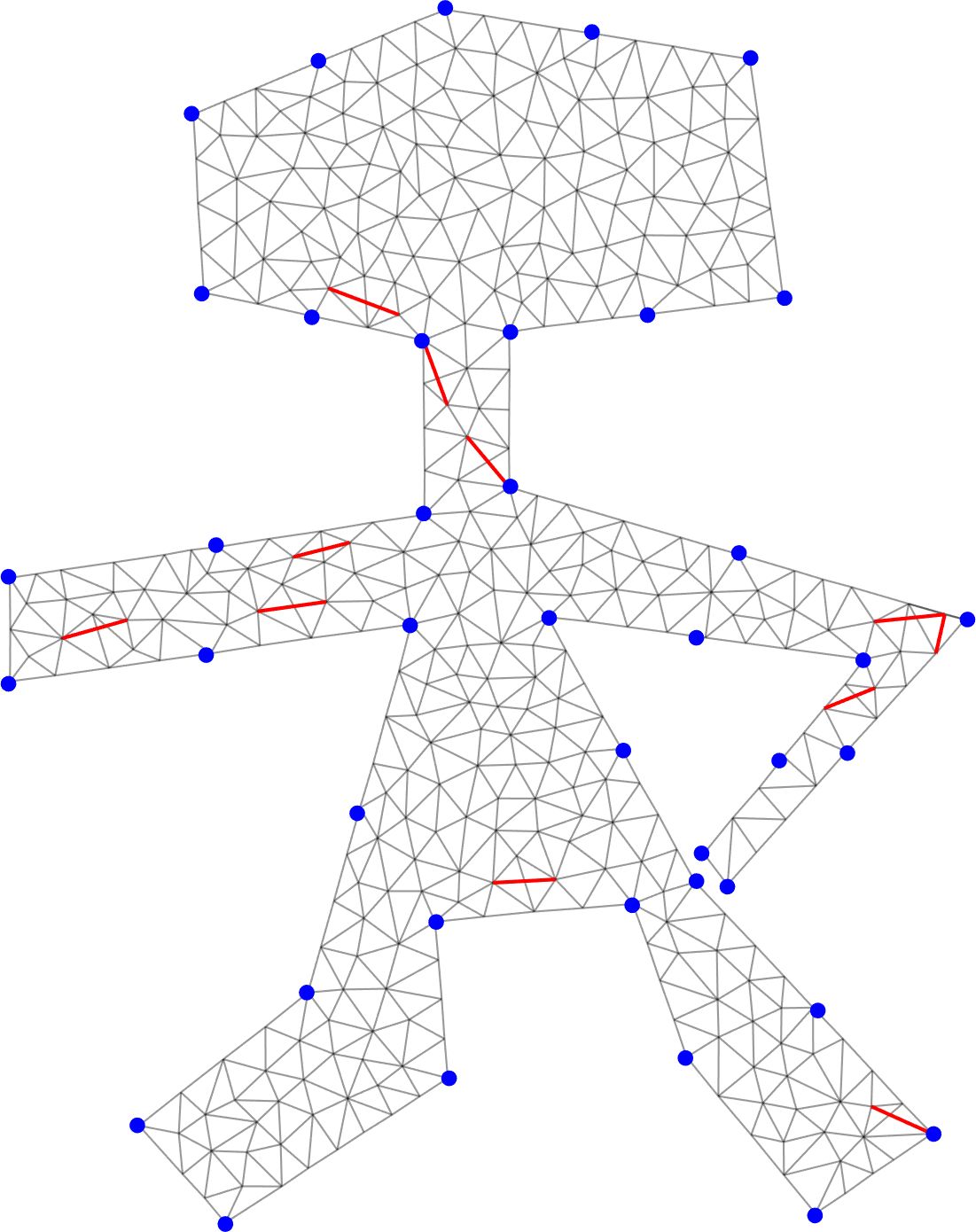}}\hfill
  \parbox{.23\textwidth}{\centering\includegraphics[height=1.7in]{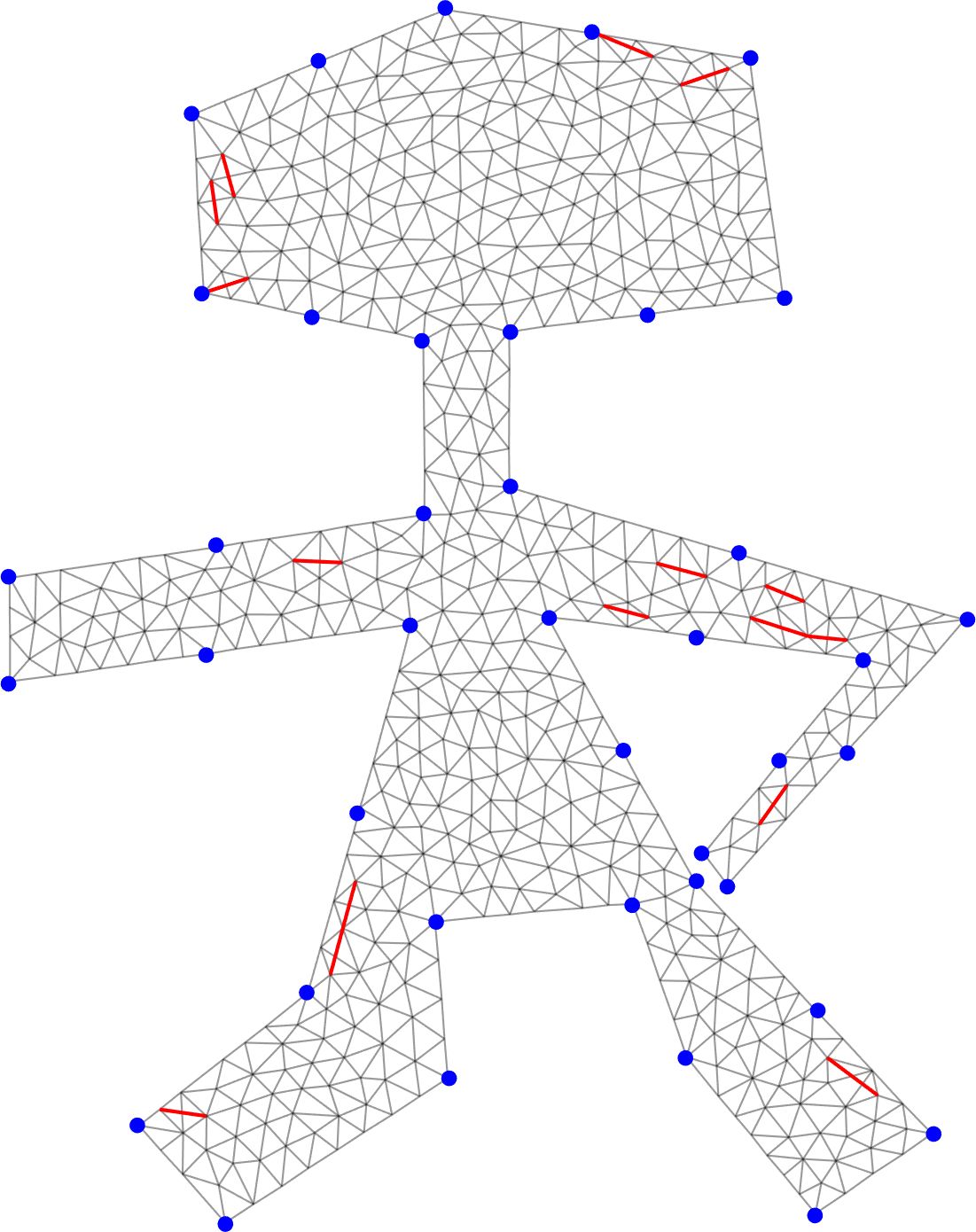}}\hfill
  \parbox{.23\textwidth}{\centering\includegraphics[height=1.7in]{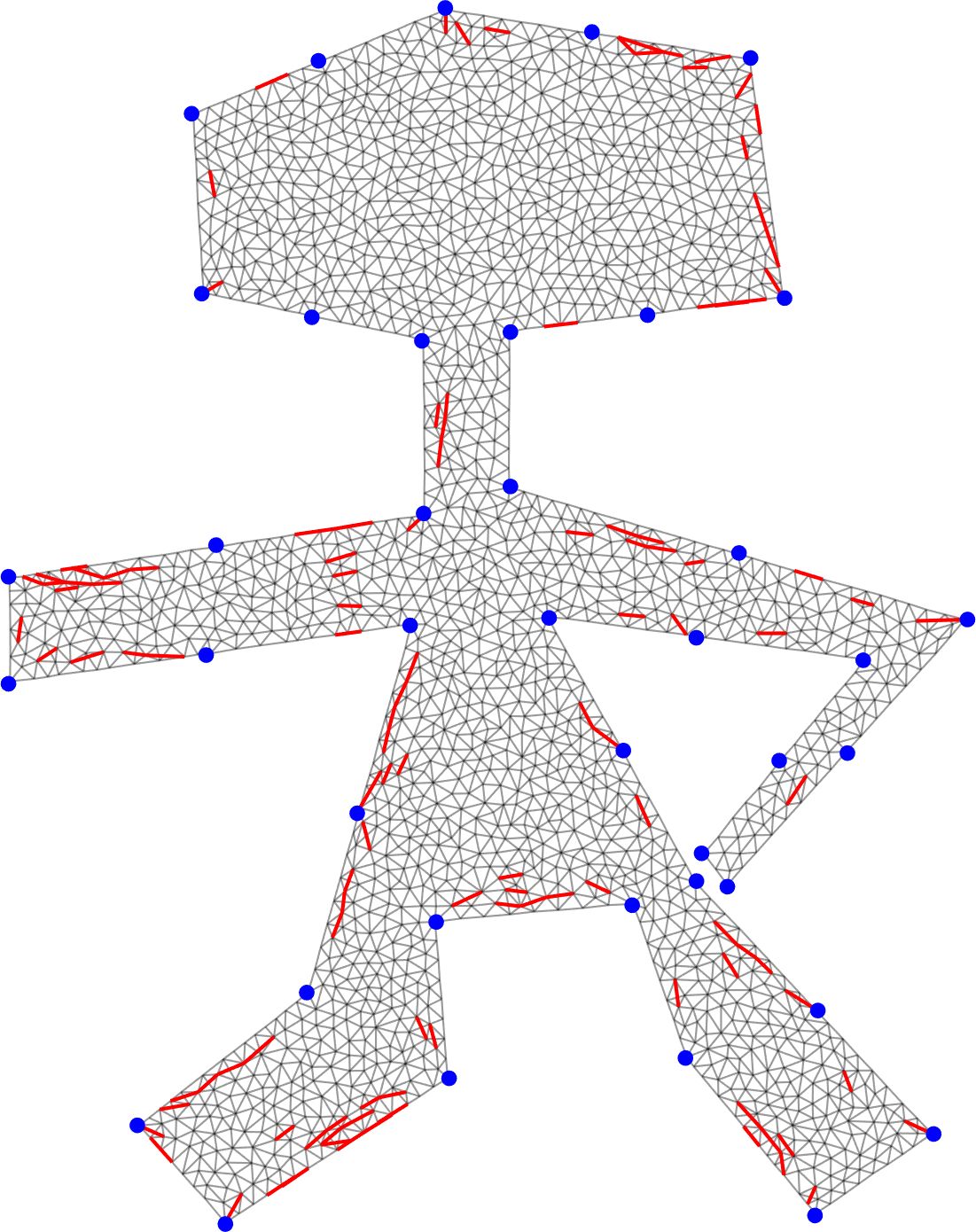}}\\[1ex]
  \parbox[c][1.7in]{.05\textwidth}{\centering \rotatebox[origin=c]{90}{$n=60$}}\hfill
  \parbox{.23\textwidth}{\centering\includegraphics[height=1.7in]{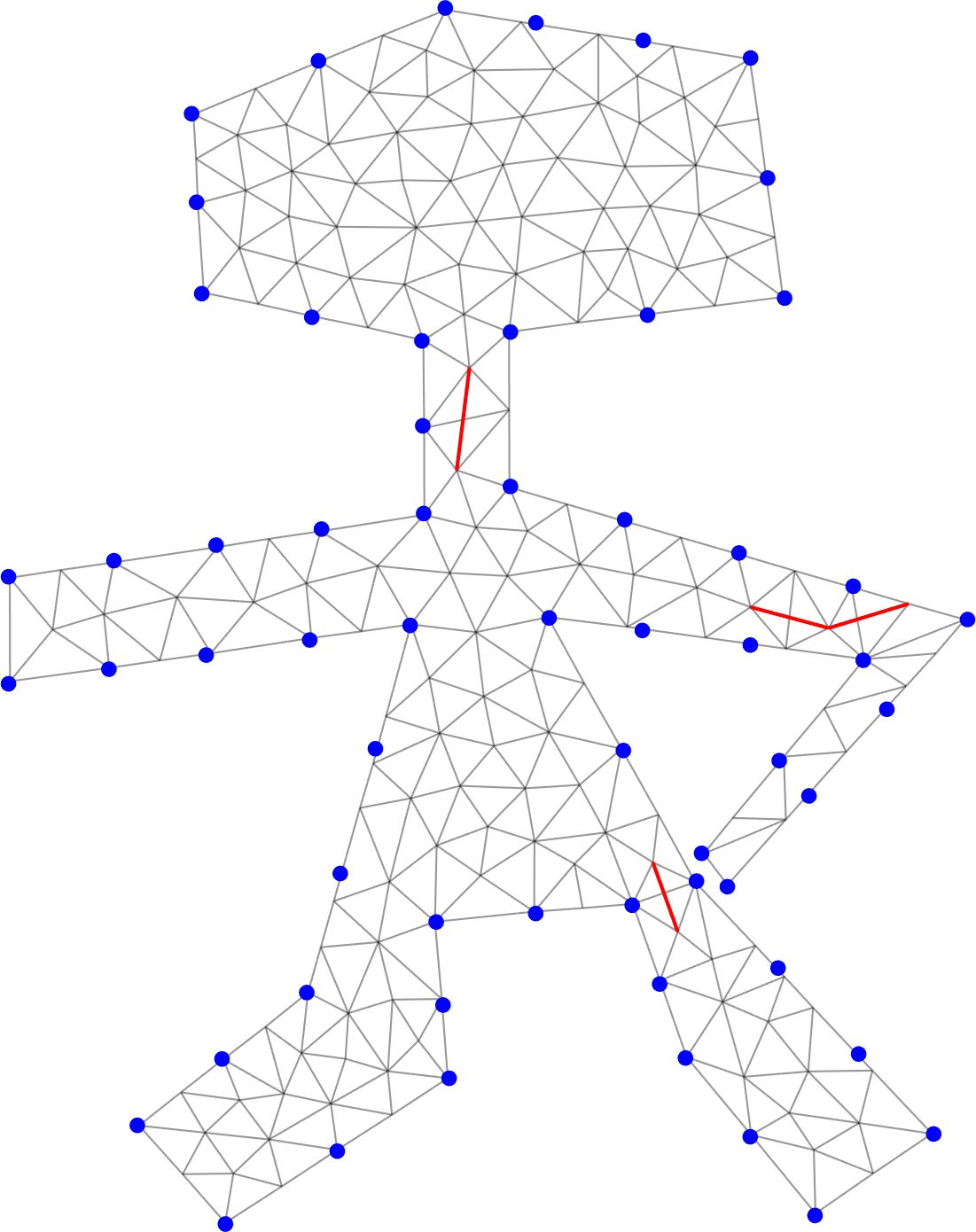}}\hfill
  \parbox{.23\textwidth}{\centering\includegraphics[height=1.7in]{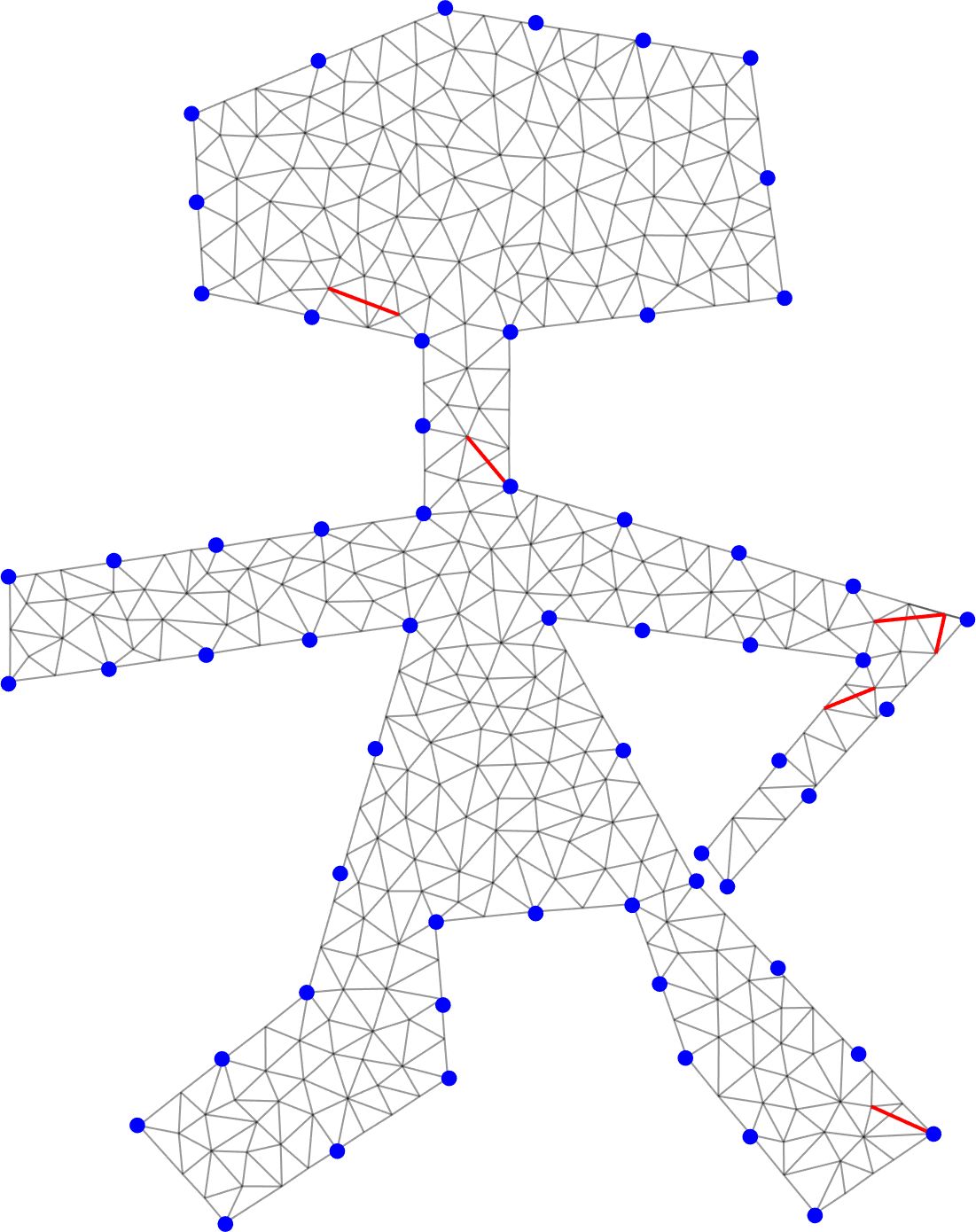}}\hfill
  \parbox{.23\textwidth}{\centering\includegraphics[height=1.7in]{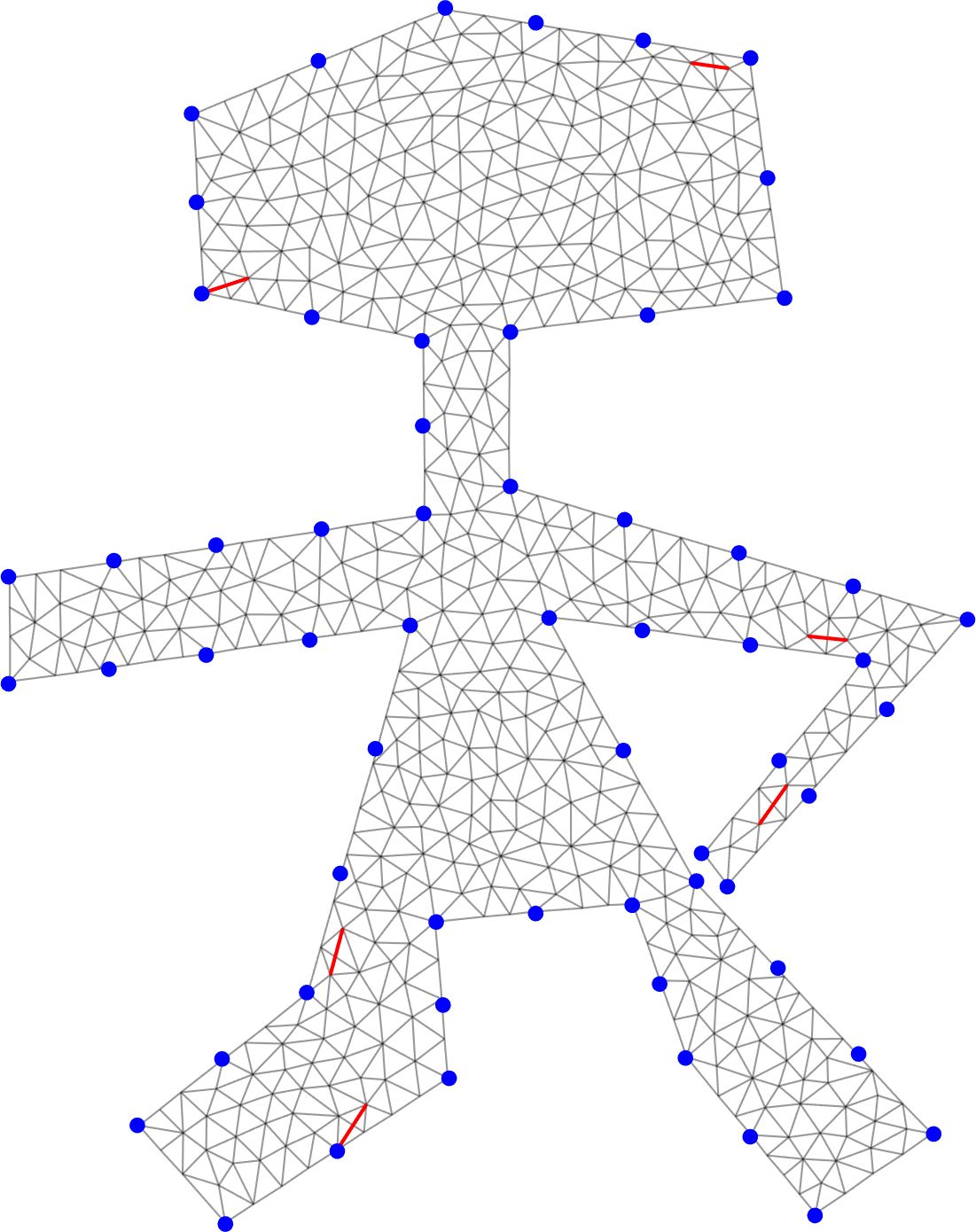}}\hfill
  \parbox{.23\textwidth}{\centering\includegraphics[height=1.7in]{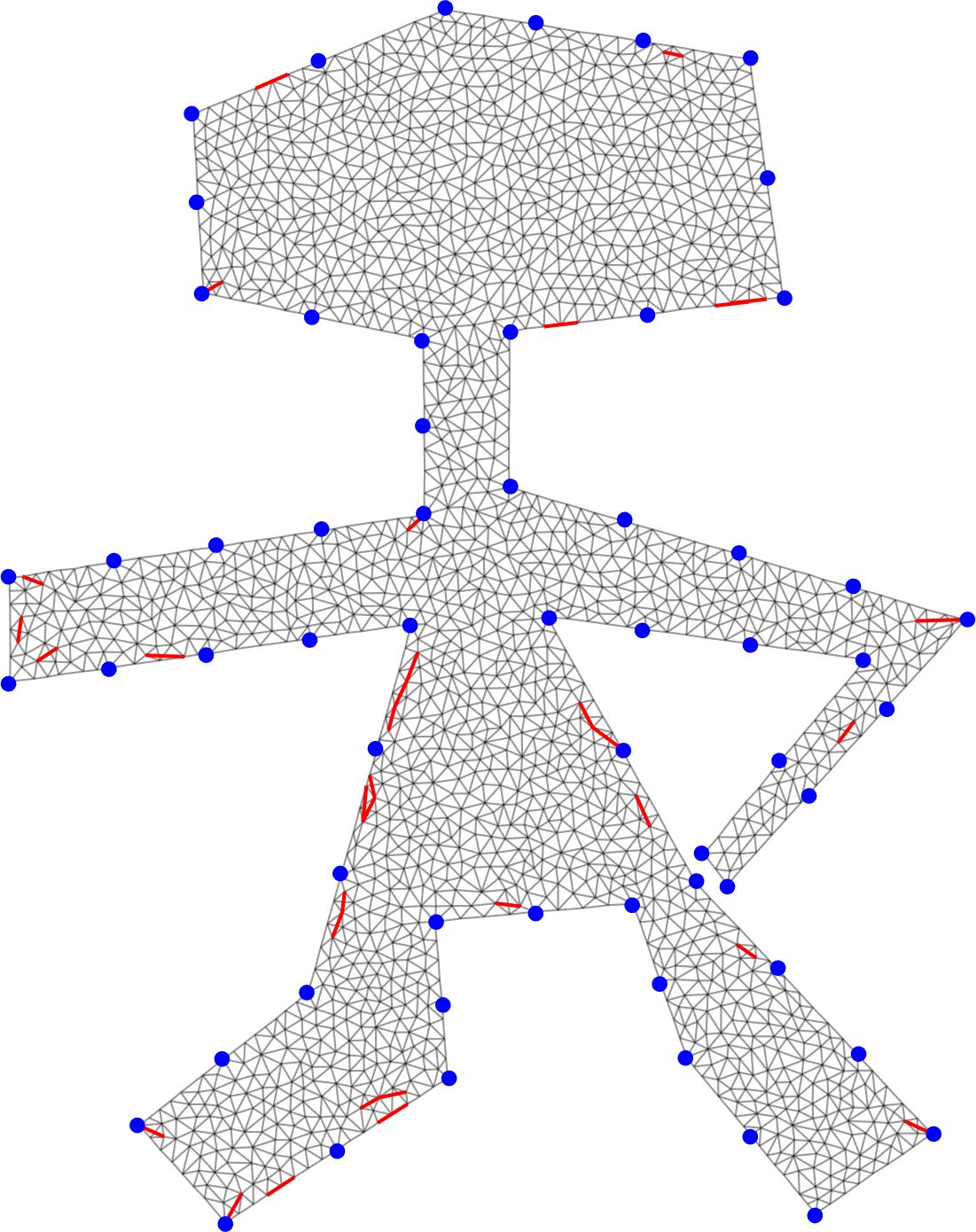}}\\[1ex]
  \parbox{.05\textwidth}{~}\hfill
  \parbox{.23\textwidth}{\centering $m=200$}\hfill
  \parbox{.23\textwidth}{\centering $m=400$}\hfill
  \parbox{.23\textwidth}{\centering $m=600$}\hfill
  \parbox{.23\textwidth}{\centering $m=2000$}\par
  \caption{Discrete greedy routing graph using the KL distance function with $n$ reduced coordinates and $m$ sites on a simply connected domain for different values of $n$ and $m$. Black edges are constrained Delaunay triangulation and red edges are those augmented to obtain the greedy routing property. Note how the number of augmented edges decreases with increasing $n$.}
  \label{fig:greedyRoutingGraph-KL-simple}
\end{figure}

\begin{figure}
  \parbox[c][1.7in]{.05\textwidth}{\centering \rotatebox[origin=c]{90}{$n=40$}}\hfill
  \parbox{.23\textwidth}{\centering\includegraphics[height=1.5in]{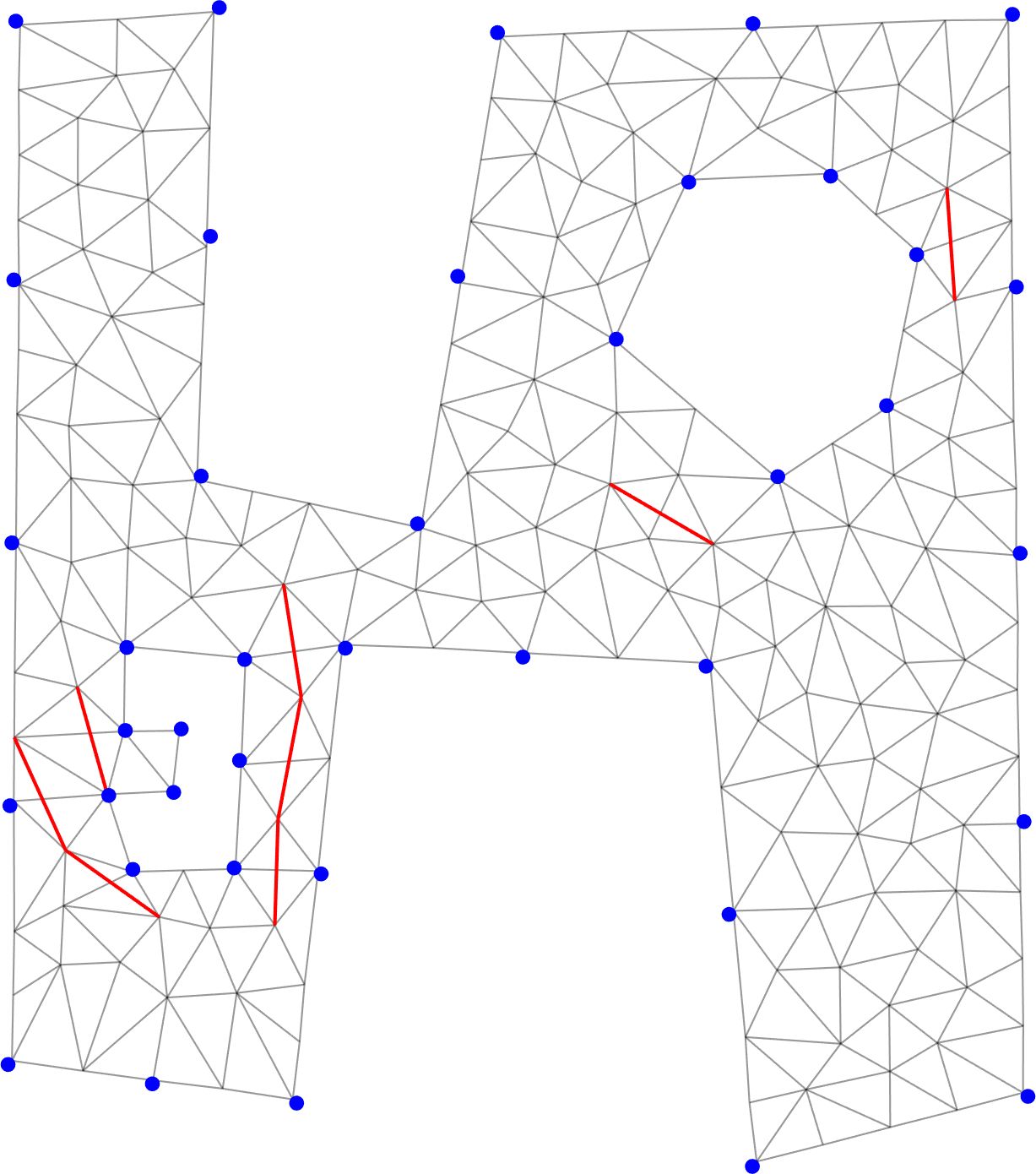}}\hfill
  \parbox{.23\textwidth}{\centering\includegraphics[height=1.5in]{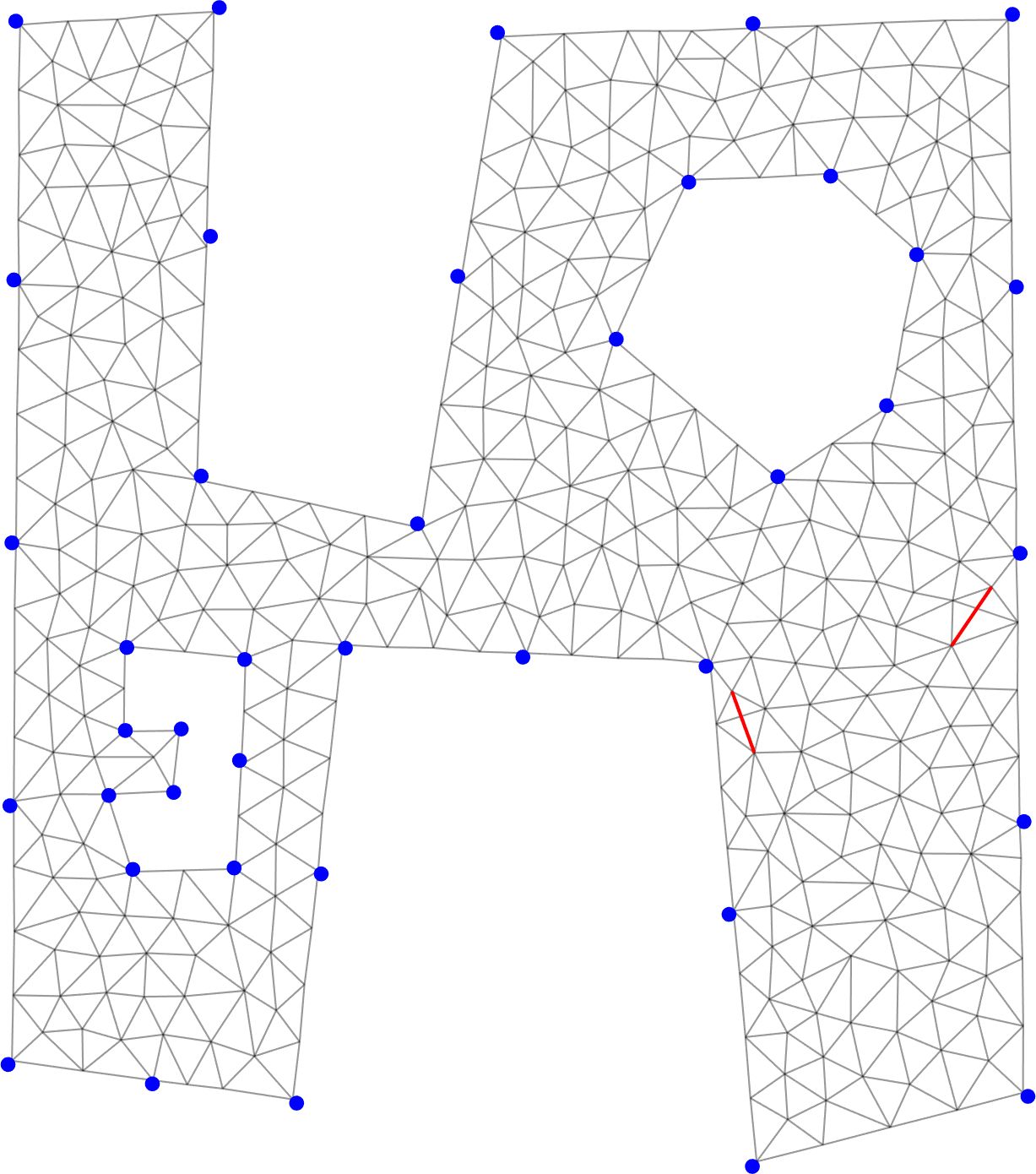}}\hfill
  \parbox{.23\textwidth}{\centering\includegraphics[height=1.5in]{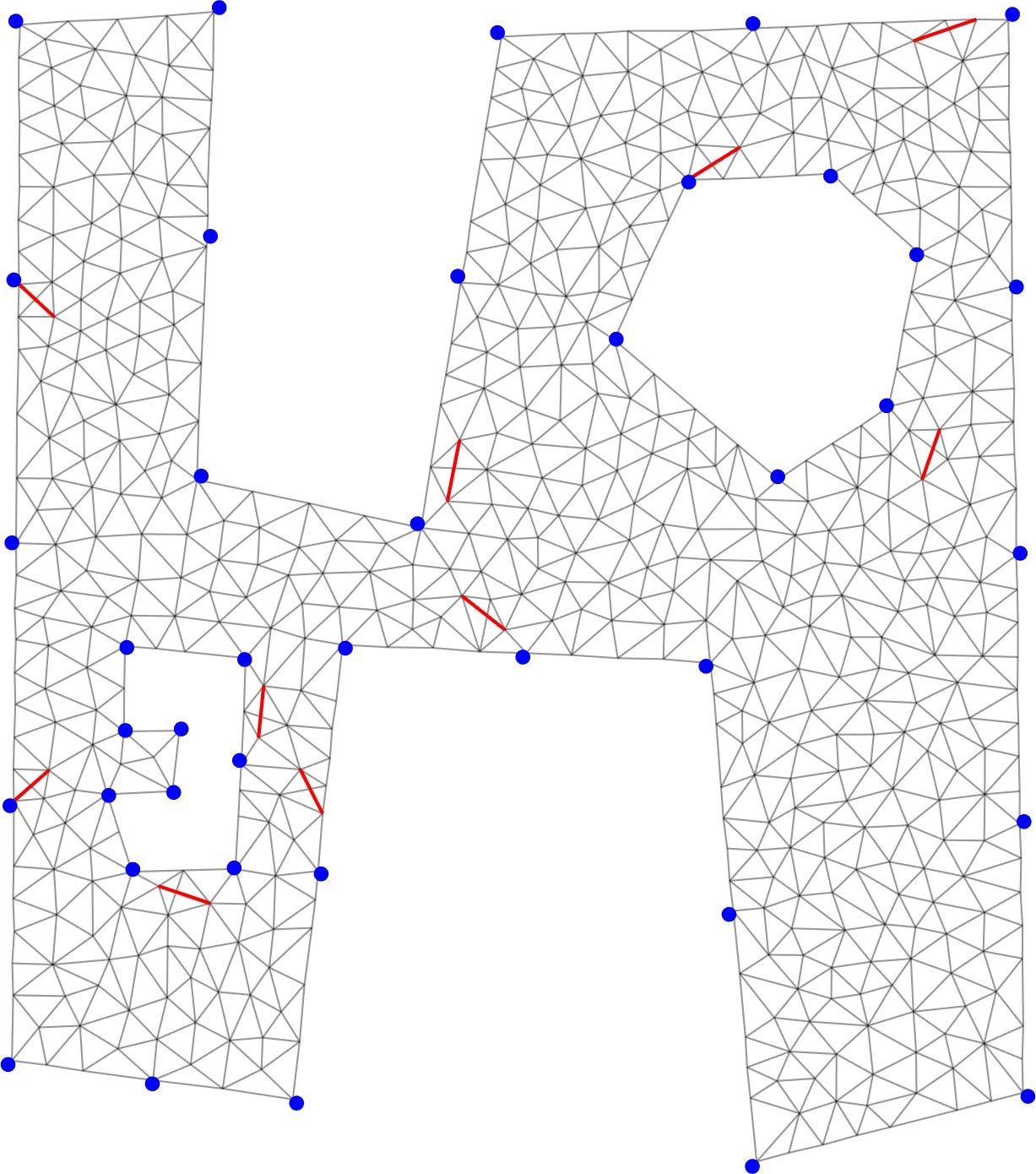}}\hfill
  \parbox{.23\textwidth}{\centering\includegraphics[height=1.5in]{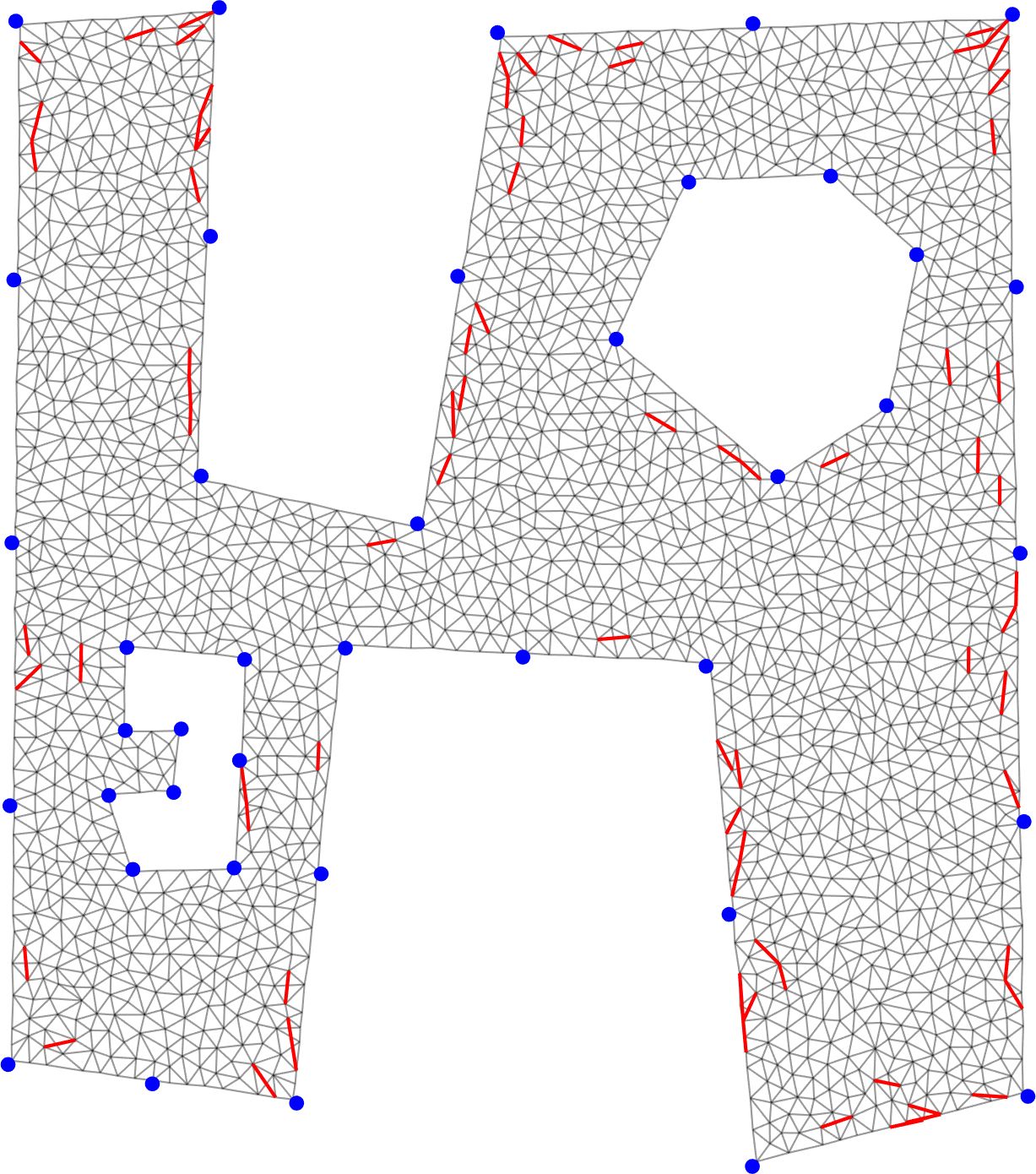}}\par
  \parbox[c][1.7in]{.05\textwidth}{\centering \rotatebox[origin=c]{90}{$n=60$}}\hfill
  \parbox{.23\textwidth}{\centering\includegraphics[height=1.5in]{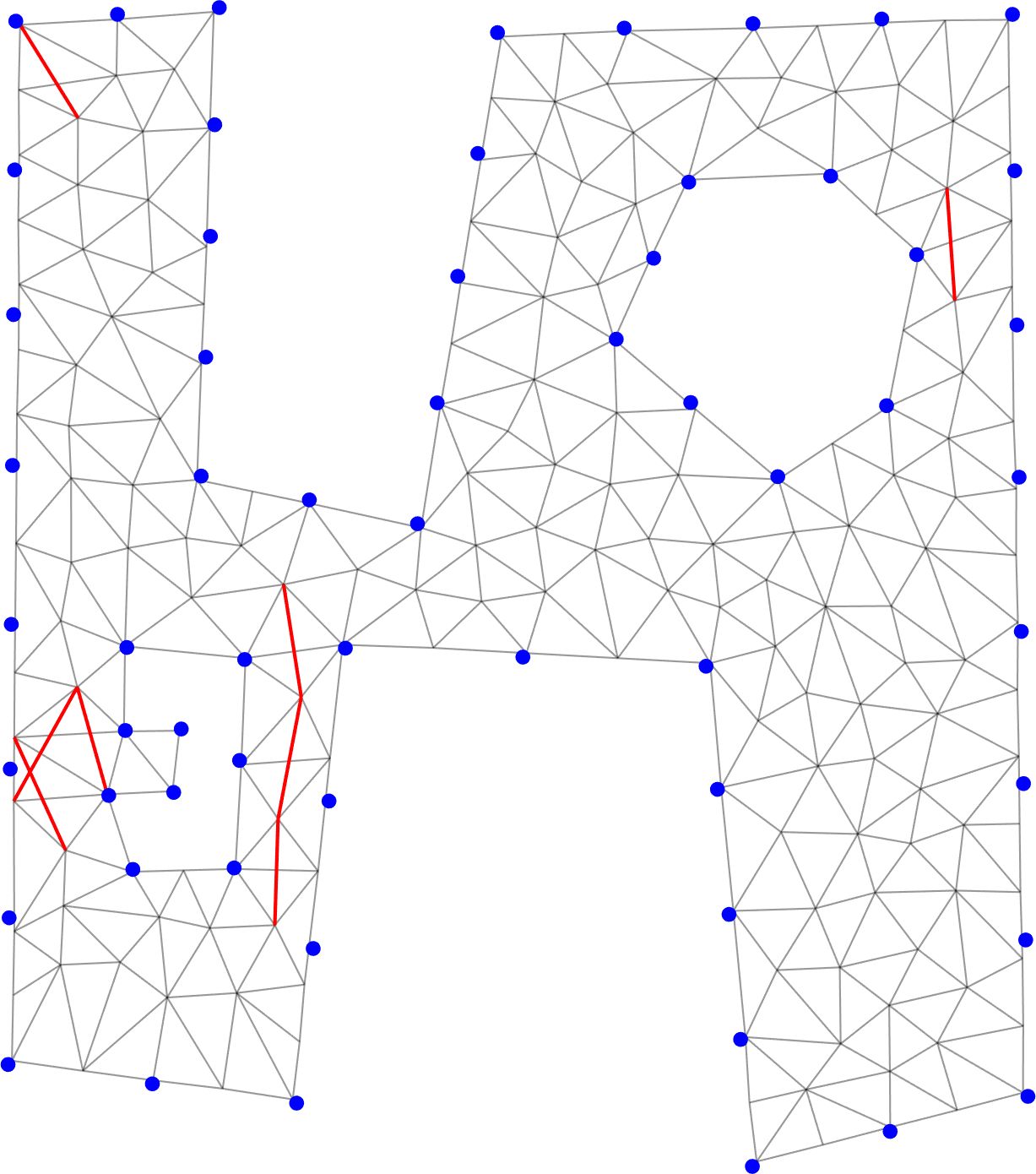}}\hfill
  \parbox{.23\textwidth}{\centering\includegraphics[height=1.5in]{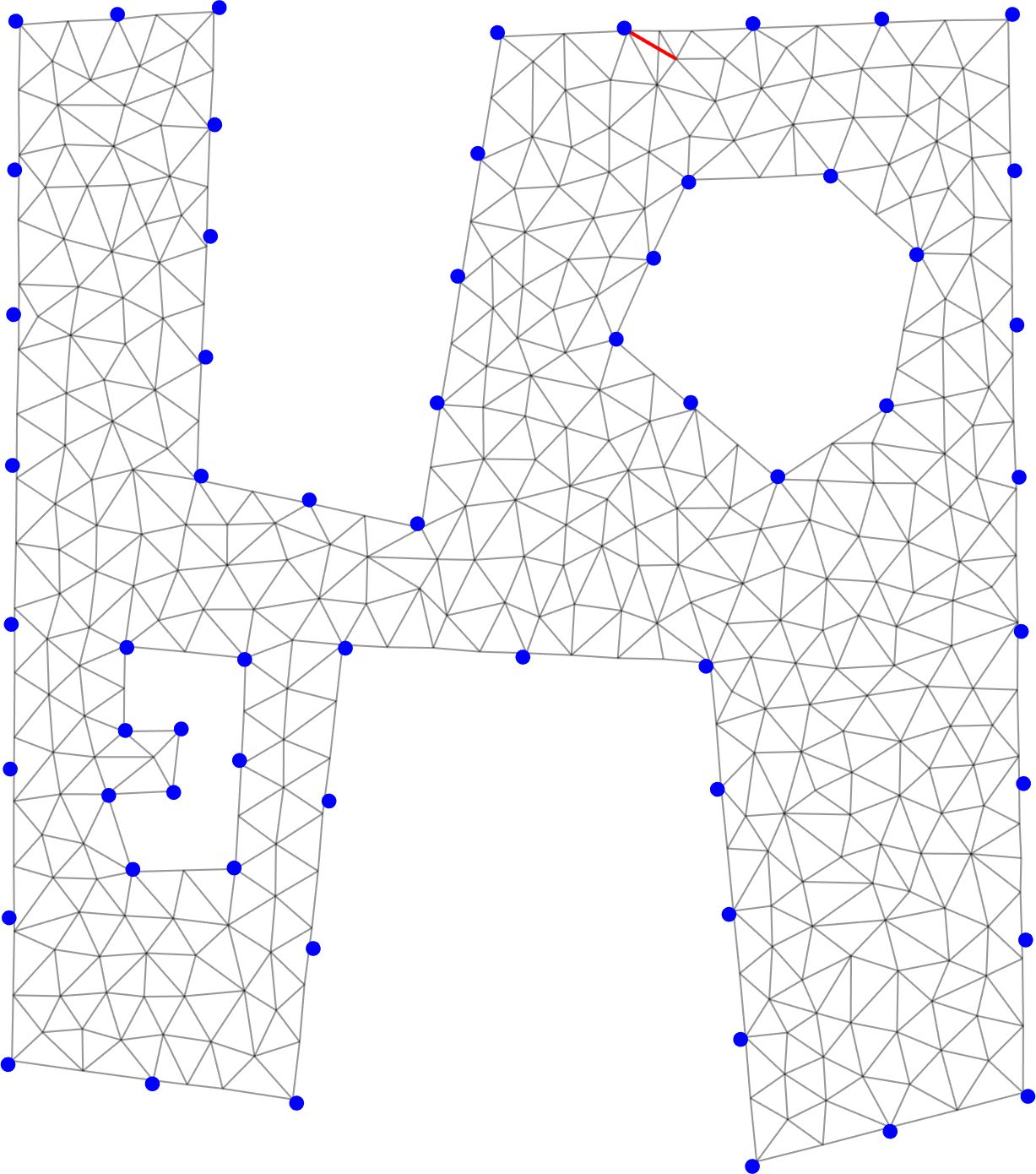}}\hfill
  \parbox{.23\textwidth}{\centering\includegraphics[height=1.5in]{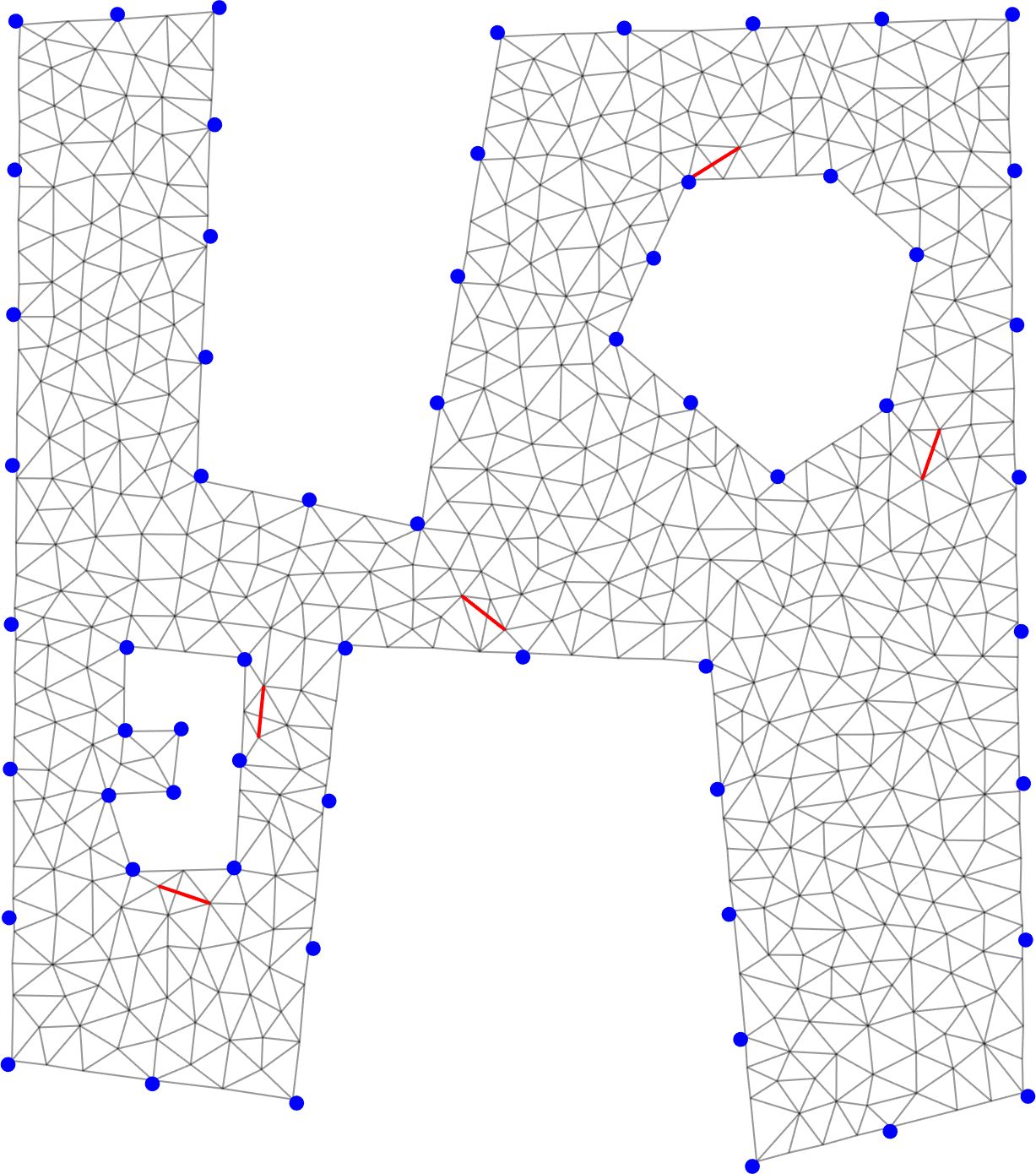}}\hfill
  \parbox{.23\textwidth}{\centering\includegraphics[height=1.5in]{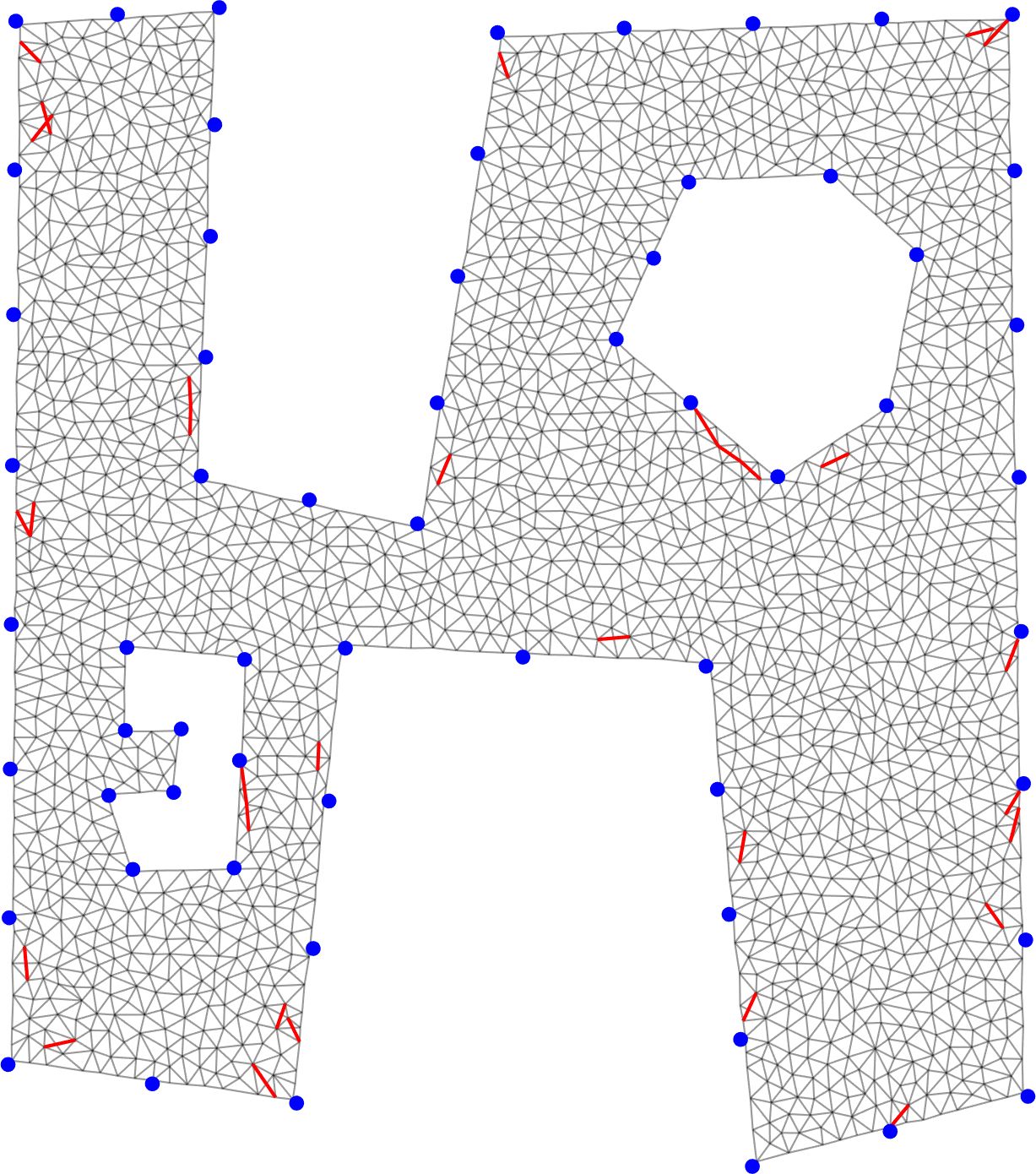}}\par
  \parbox{.05\textwidth}{~}\hfill
  \parbox{.23\textwidth}{\centering $m=200$}\hfill
  \parbox{.23\textwidth}{\centering $m=400$}\hfill
  \parbox{.23\textwidth}{\centering $m=600$}\hfill
  \parbox{.23\textwidth}{\centering $m=2000$}\par
  \caption{Same as Figure~\ref{fig:greedyRoutingGraph-KL-simple} for a multiply-connected domain.}
  \label{fig:greedyRoutingGraph-KL-hole}
\end{figure}

\begin{figure}
  \parbox{.3\textwidth}{\centering\includegraphics[height=2in]{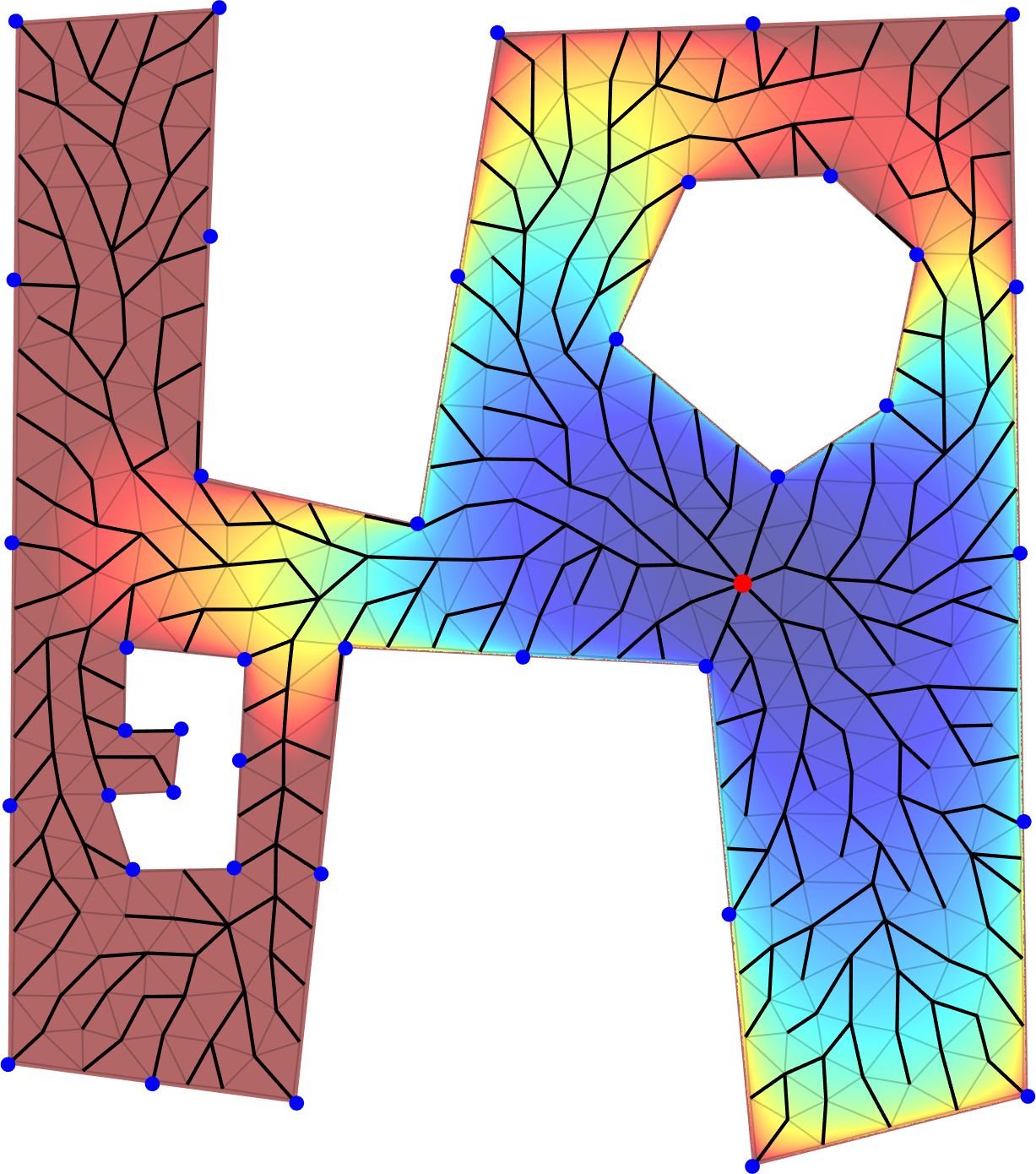}}\hfill
  \parbox{.3\textwidth}{\centering\includegraphics[height=2in]{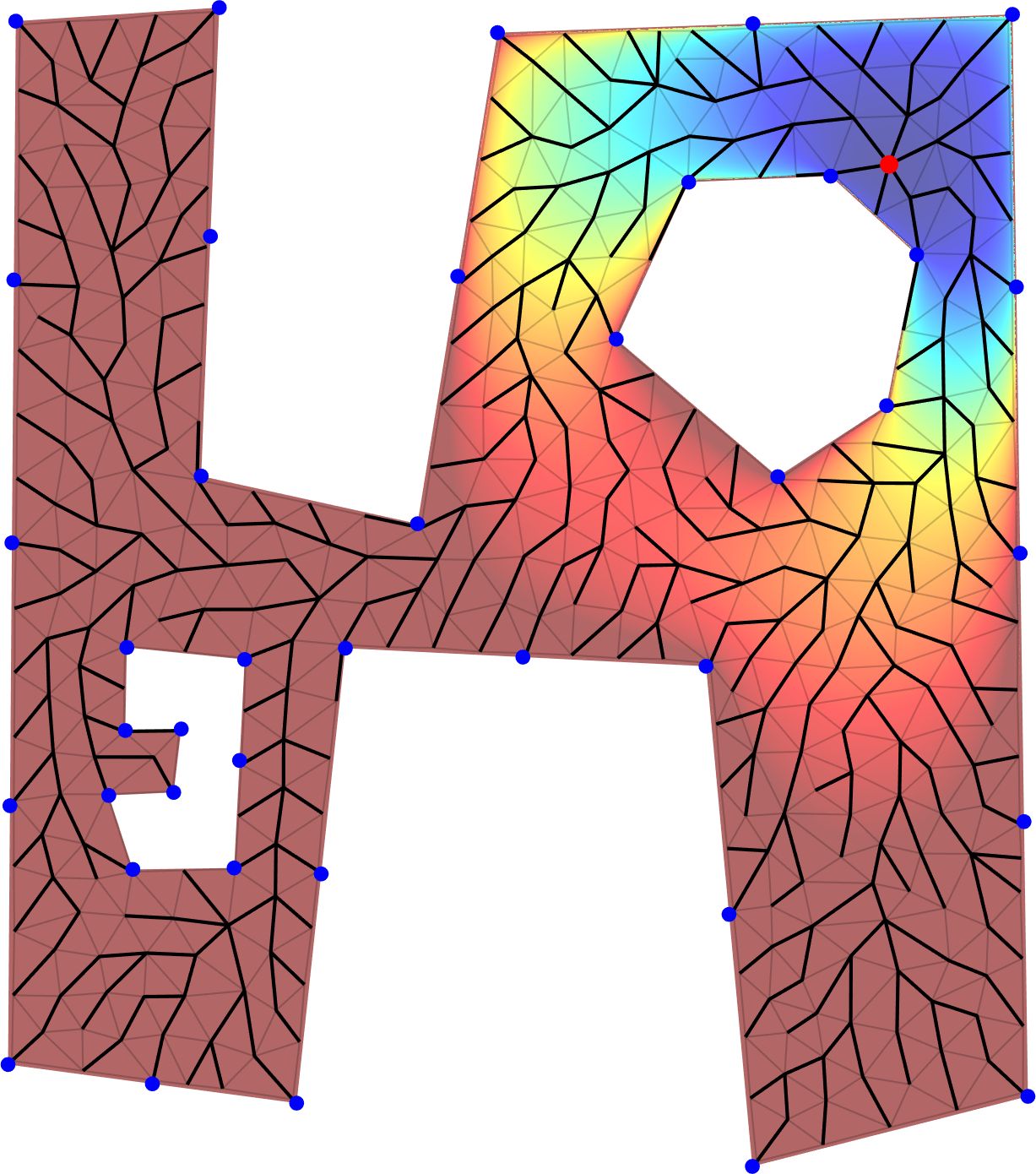}}\hfill
  \parbox{.3\textwidth}{\centering\includegraphics[height=2in]{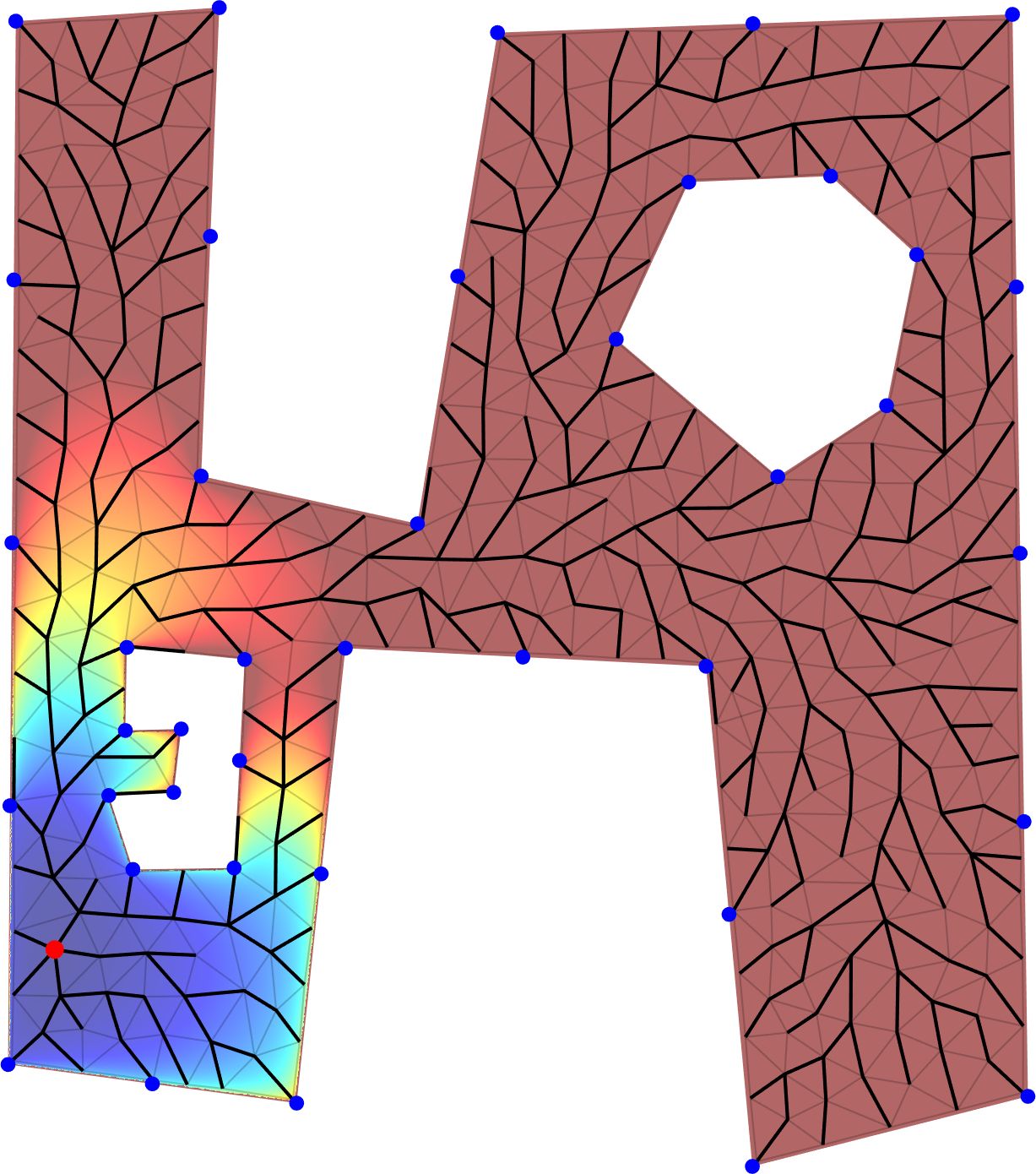}}\hfill
  \parbox{.05\textwidth}{\centering\includegraphics[height=2in]{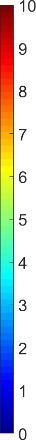}}\\[1ex]
  \parbox{.3\textwidth}{\centering\includegraphics[height=2in]{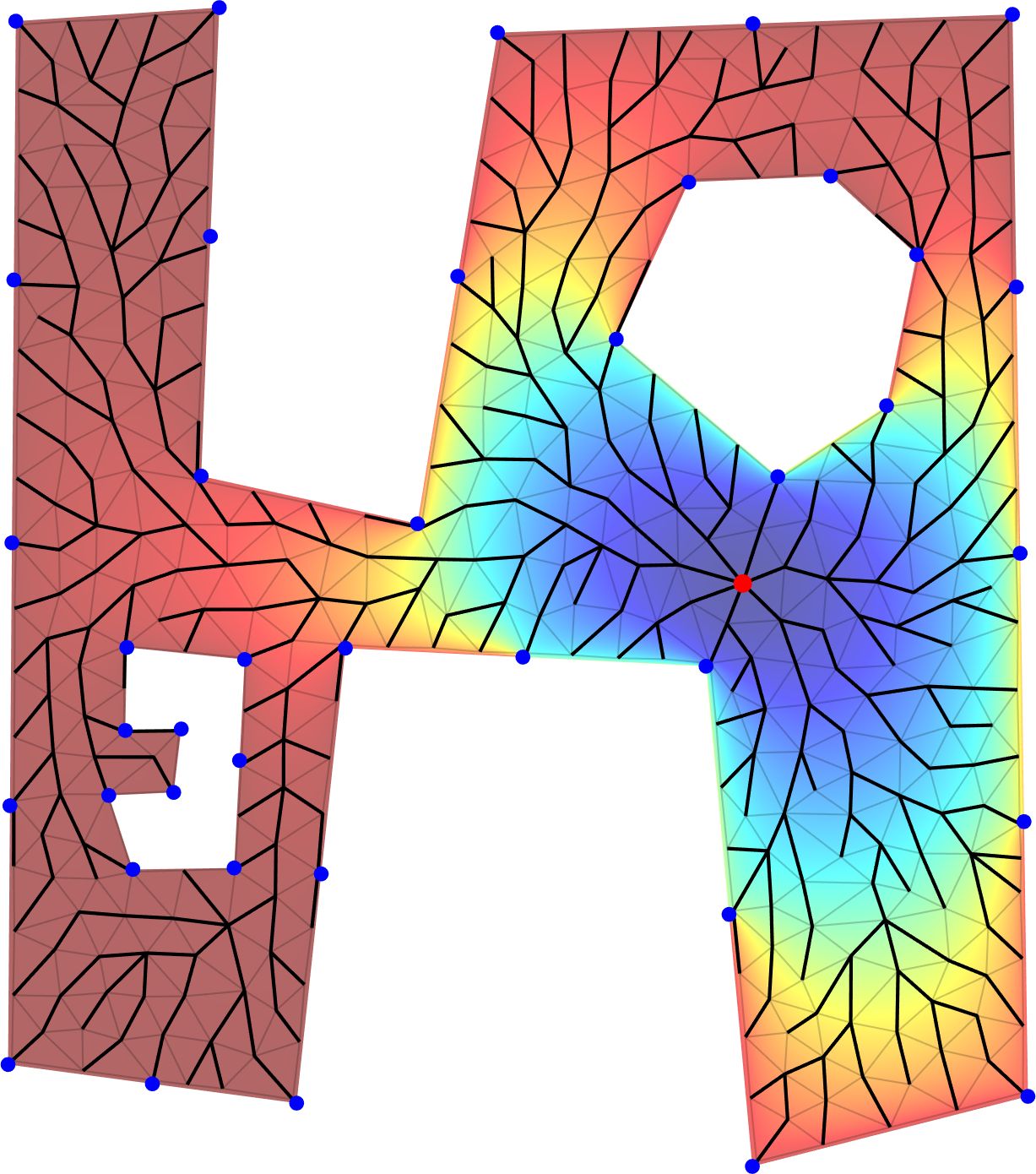}}\hfill
  \parbox{.3\textwidth}{\centering\includegraphics[height=2in]{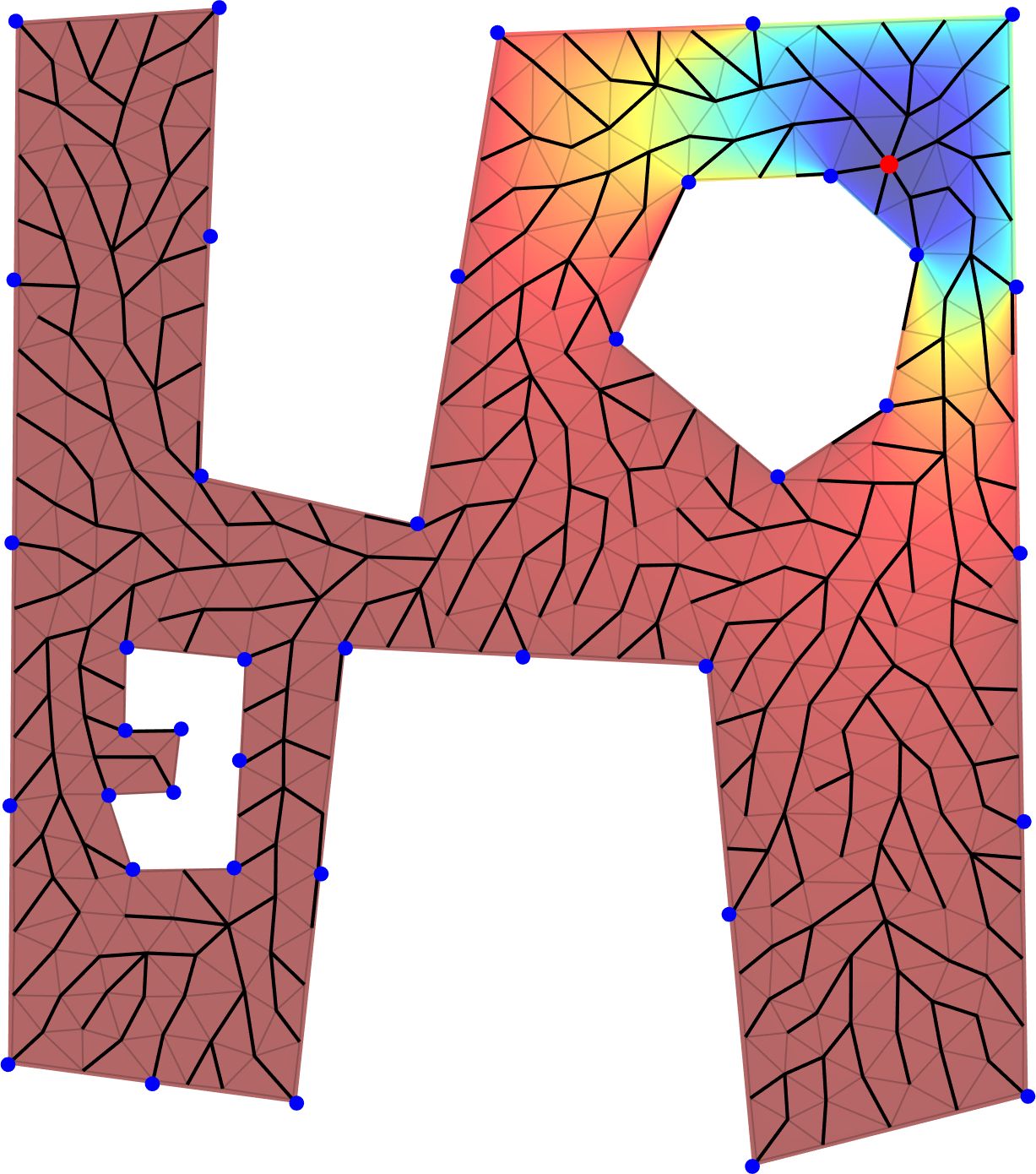}}\hfill
  \parbox{.3\textwidth}{\centering\includegraphics[height=2in]{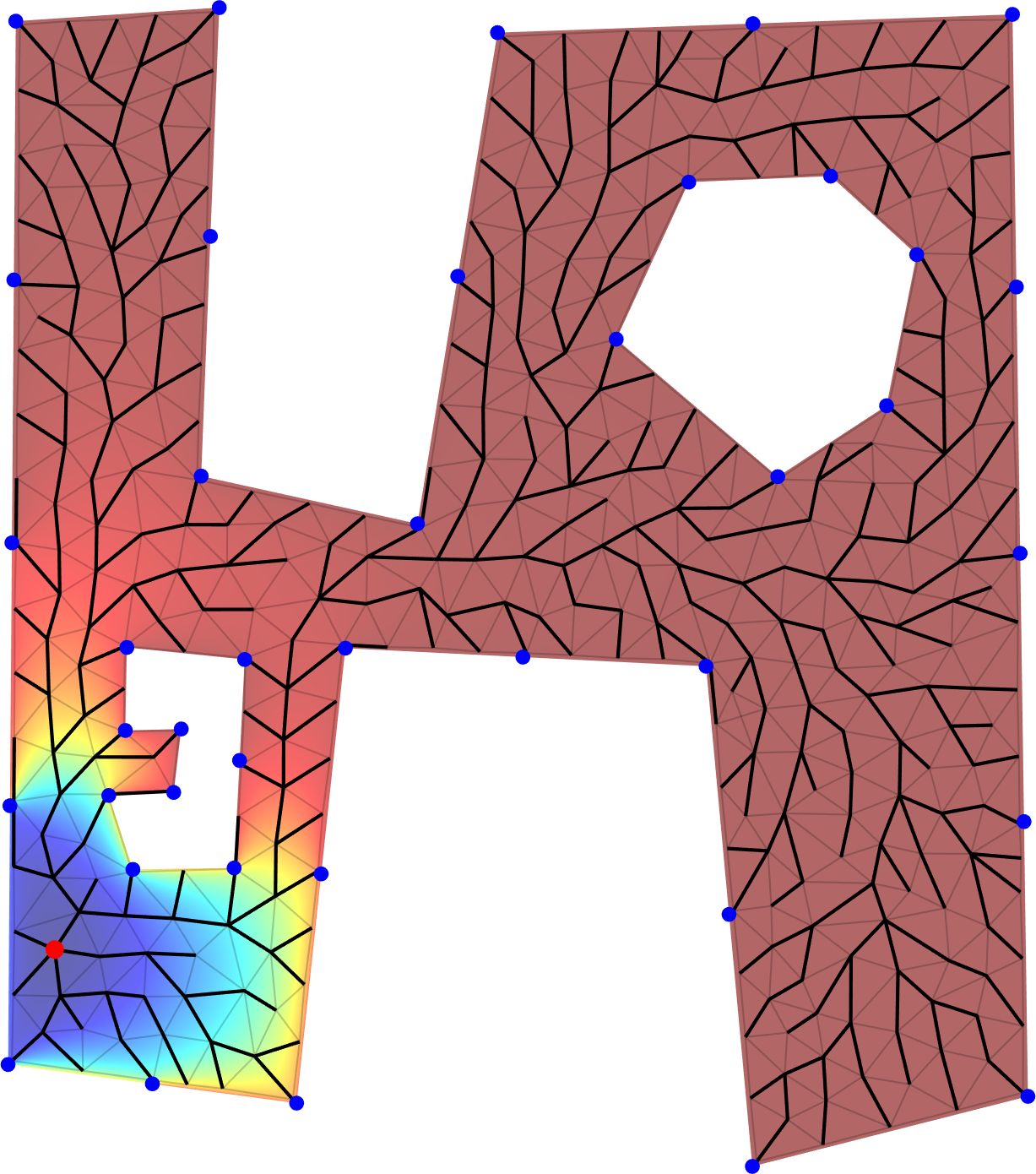}}\hfill
  \parbox{.05\textwidth}{\centering\includegraphics[height=2in]{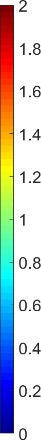}}\par
  \caption{Path tree (thick black edges) to select (red) target sites for the multiply-connected domain and graph of Figure~\ref{fig:greedyRoutingGraph-KL-hole} for $n=40$ reduced coordinates and $m=400$ sites. The domain is color-coded according to the $f$-divergence distance of reduced coordinates from the target site. \emph{Top:} Using the KL distance function. \emph{Bottom:} Using the H distance function. Note the subtle differences between the two rows.}
  \label{fig:path-KL-H}
\end{figure}

%%%%%%%%%%%%%%%%%%%%%%%%%%%%%%%%%%%%%%%%%%%%%%%%%%%%%%%%%%%%%%%%%%%%%%%%%%%%

\section{Conclusion and Discussion}

We have described a practical method for path planning on planar domains based on gradient-descent using a distance function. The key is to use an $f$-divergence function on discrete coordinate vectors. Our discrete coordinate vectors are just the harmonic measures of a partition of the boundary into $n$ segments, which are the inner products of the Poisson kernel with (``box'') indicator basis functions for each segment. We speculate that a Divergence Gradient Theorem holds for other sets of coordinates derived by an inner product of more sophisticated basis functions, such as the piecewise linear ``tent'' function over two adjacent boundary segments. This is the method used to construct harmonic \emph{barycentric coordinates} on a polygonal domain~\cite{Joshi:2007:HCF}. Also possible are Gaussian basis functions over the boundary, as long as they are not too narrow or too wide. These variants of reduced coordinates may be computed similarly to the basic reduced coordinates used in this paper merely by changing the right hand sides $b_j$ of the linear equations in~\eqref{eq:rhs} to something more sophisticated than the binary vector in~\eqref{eq:binary-indicator}. Figure~\ref{fig:path-box-tent-Gaussian} shows the gradient-descent trees generated by these coordinates using the KL distance function.

\begin{figure}[t!]
  \parbox{.33\textwidth}{\centering\includegraphics[height=2in]{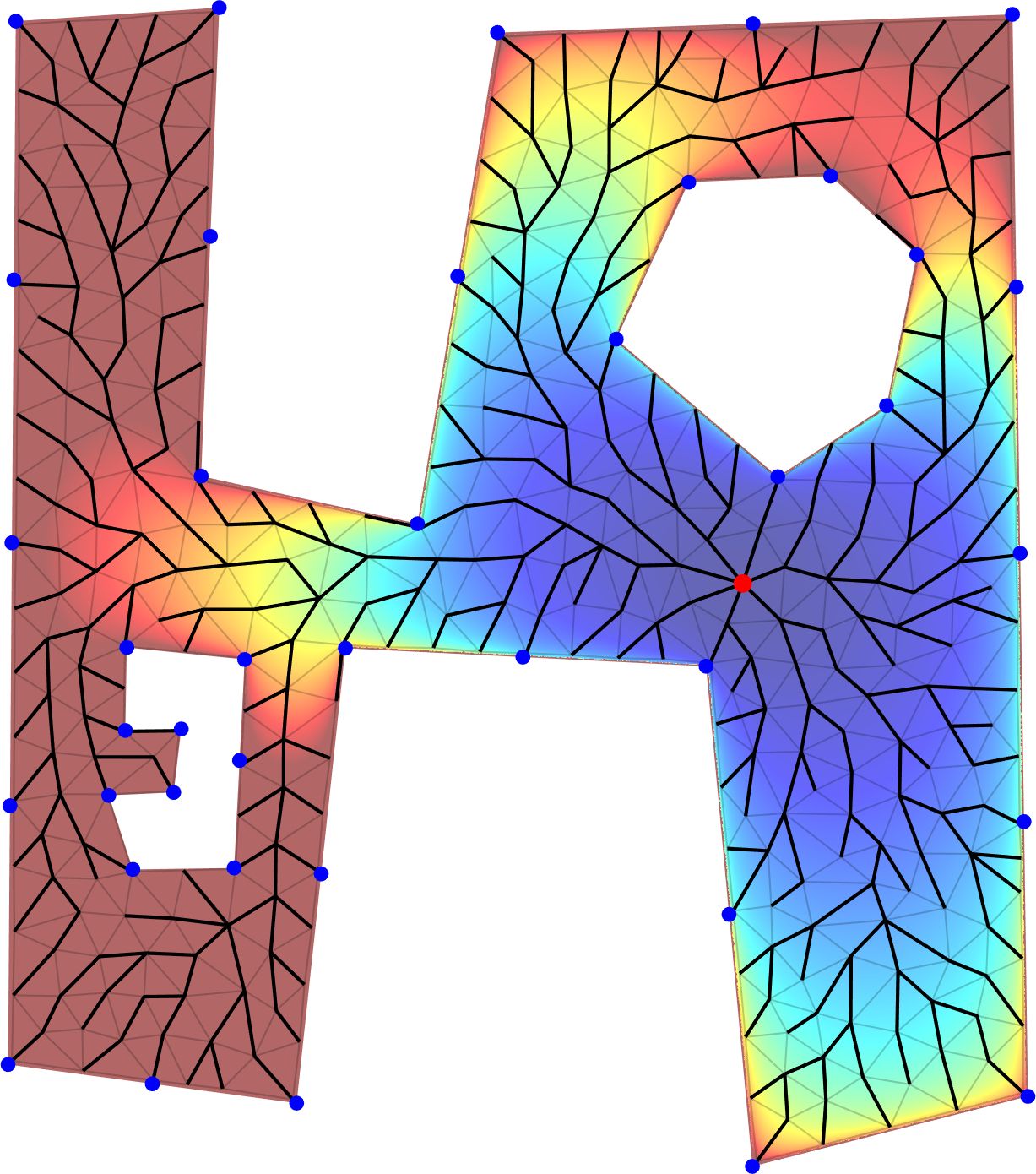}}\hfill
  \parbox{.33\textwidth}{\centering\includegraphics[height=2in]{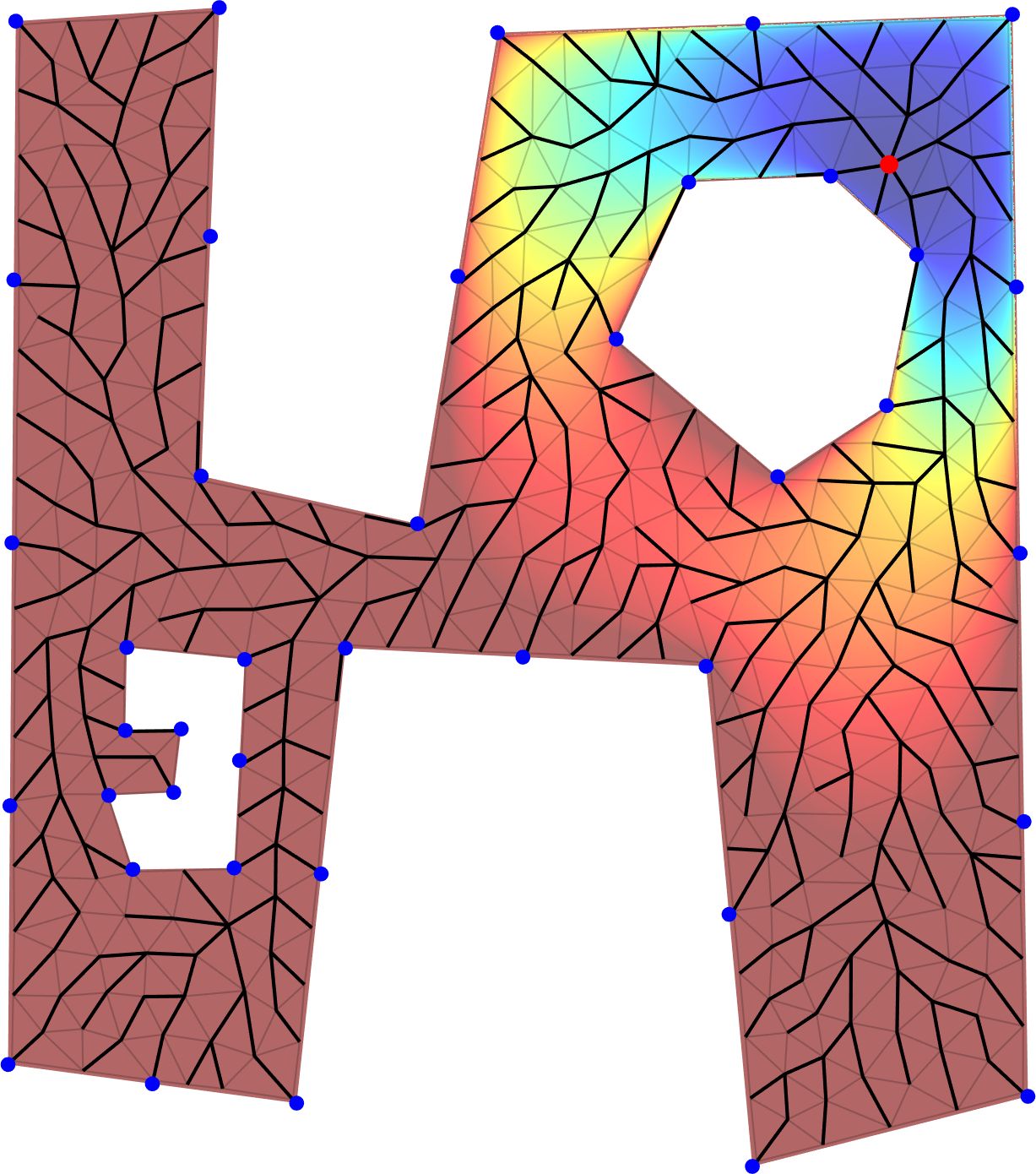}}\hfill
  \parbox{.33\textwidth}{\centering\includegraphics[height=2in]{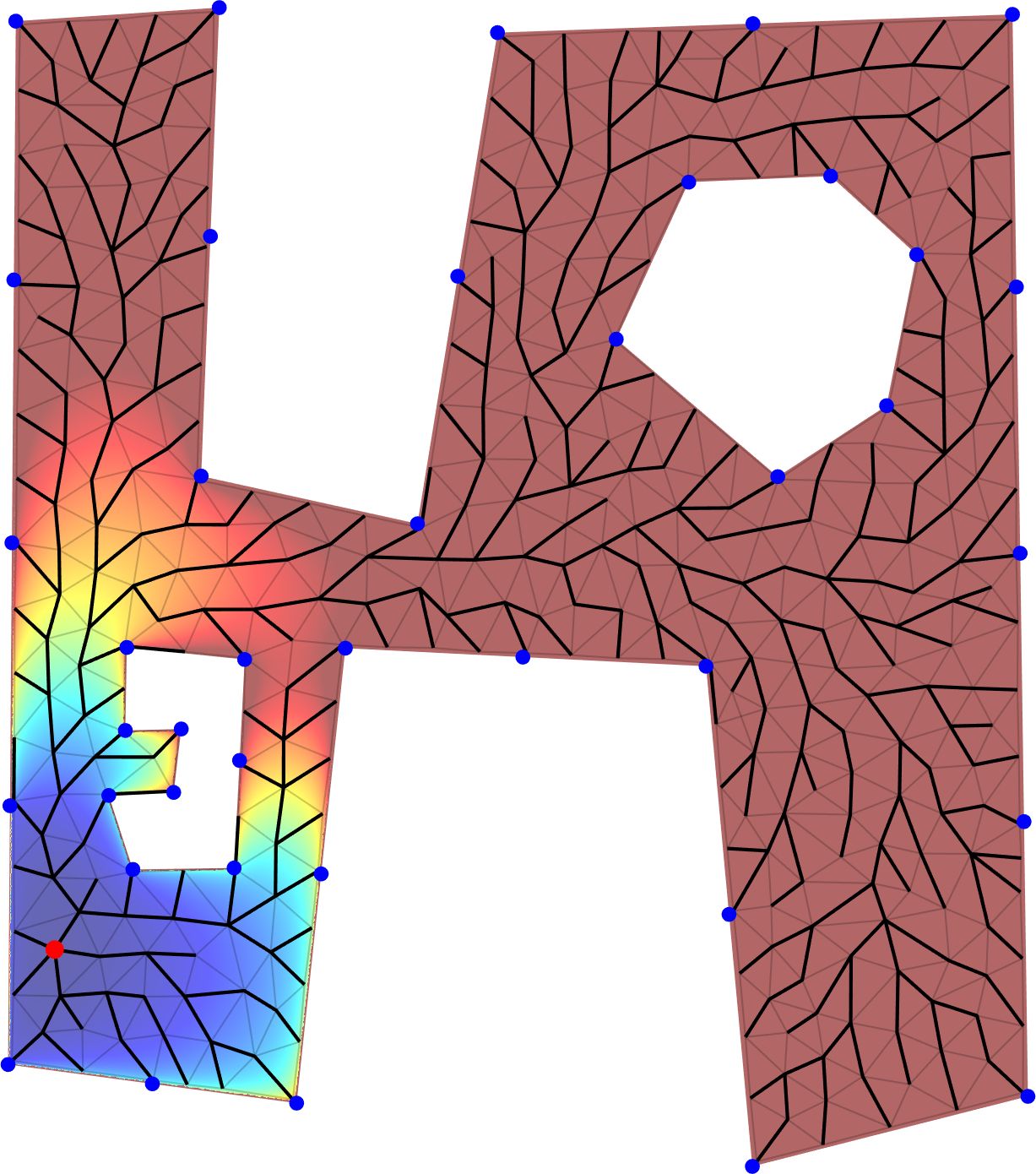}}\\[1ex]
  \parbox{.33\textwidth}{\centering\includegraphics[height=2in]{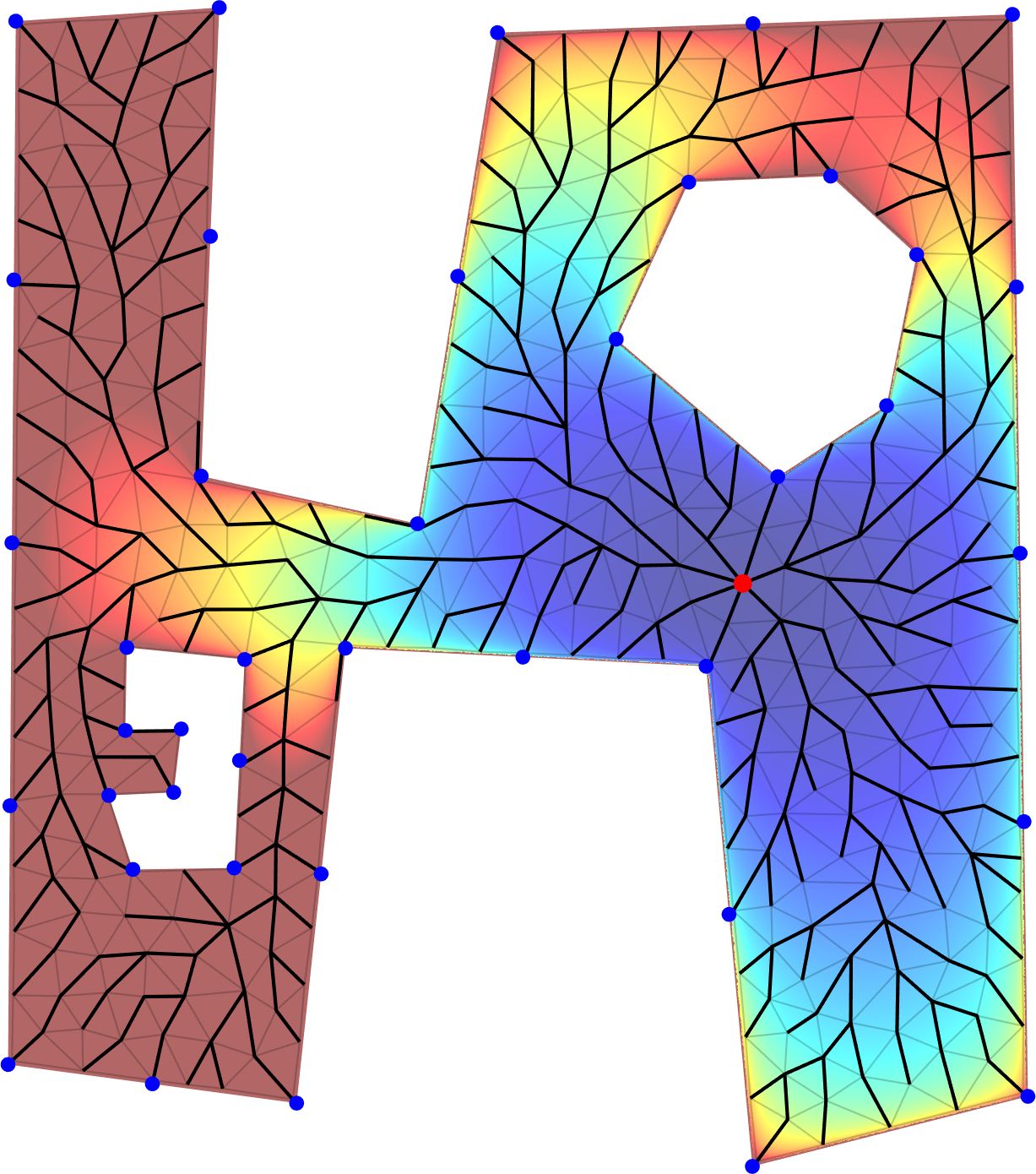}}\hfill
  \parbox{.33\textwidth}{\centering\includegraphics[height=2in]{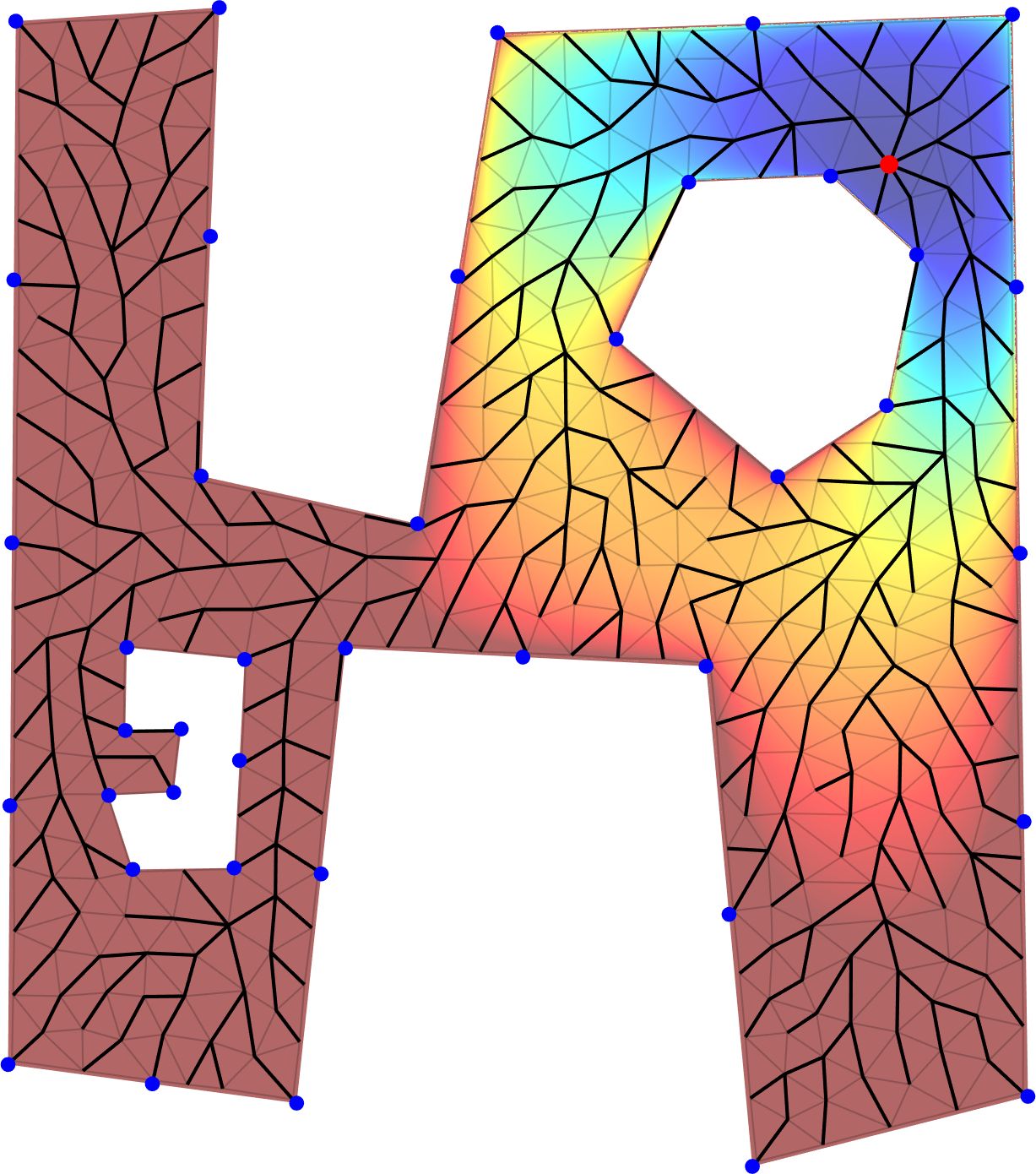}}\hfill
  \parbox{.33\textwidth}{\centering\includegraphics[height=2in]{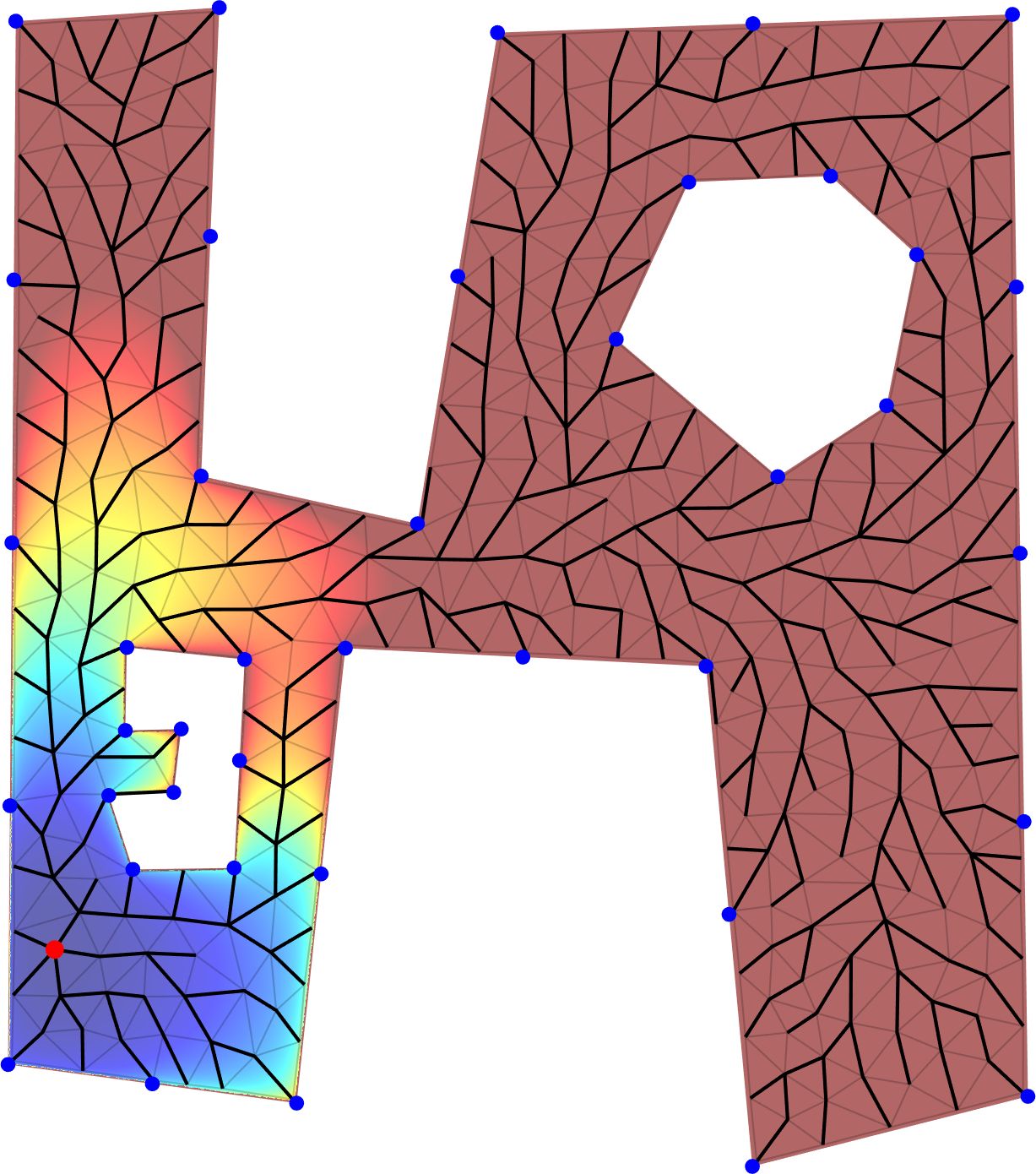}}\\[1ex]
  \parbox{.33\textwidth}{\centering\includegraphics[height=2in]{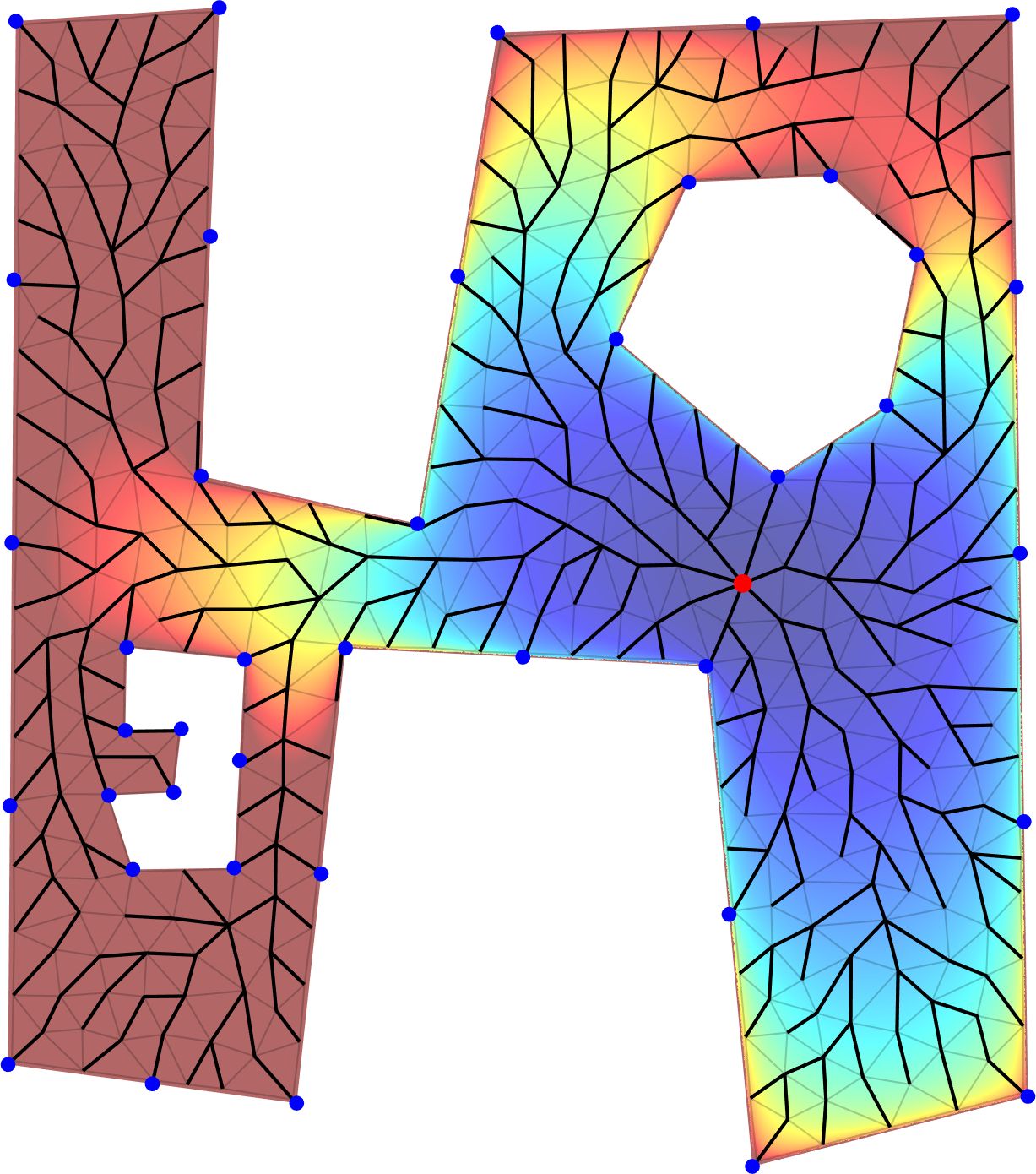}}\hfill
  \parbox{.33\textwidth}{\centering\includegraphics[height=2in]{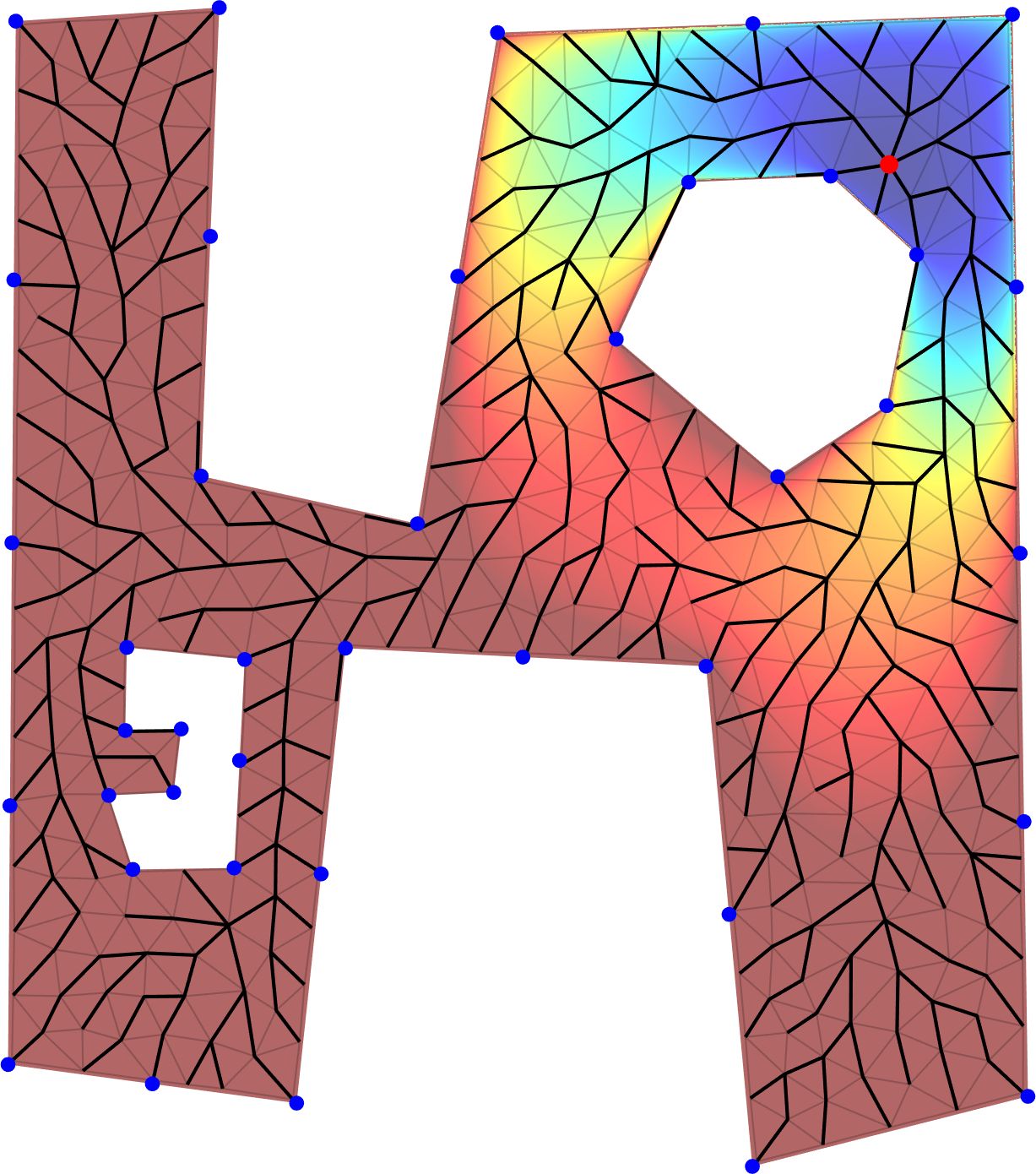}}\hfill
  \parbox{.33\textwidth}{\centering\includegraphics[height=2in]{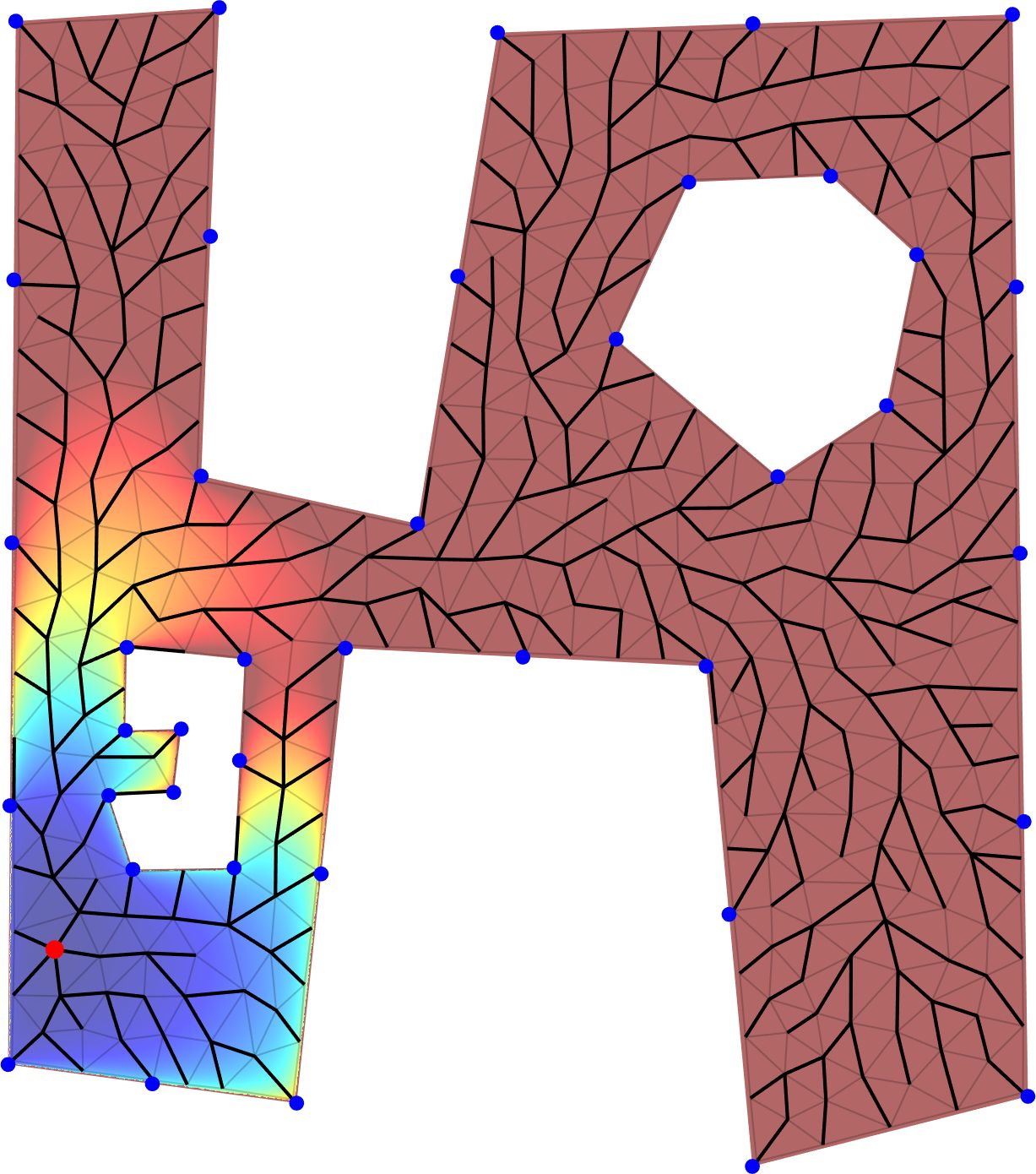}}\par
  \caption{Path tree (thick black edges) to select (red) target sites for the multiply-connected domain and graph of Figure~\ref{fig:greedyRoutingGraph-KL-hole} for $n=20$ reduced coordinates and $m=400$ sites using the KL distance function. The domain is color-coded according to the distance from the target site. \emph{Top:} Using a box basis function. \emph{Middle:} Using a tent basis function. \emph{Bottom:} Using a Gaussian basis function.}
  \label{fig:path-box-tent-Gaussian}
\end{figure}

Our Divergence Gradient Theorem is proven only for the case of a simply-connected domain. This is done by conformal reduction to the canonical case of a unit disk domain with the target at the origin. As with Chen et al.~\cite{Chen:2016:PPW}, we do not have a proof of the lack of local minima for a multiply-connected domain (where the holes in the domain would correspond to obstacles in a real-world scenario), but speculate that it indeed is the case. This is supported by all our experimental results. Note that in this case the distance function may contain critical points which are saddles, so the continuous gradient may vanish there, but these are not local minima, so not fatal.

Our experimental results also indicate that as the number of reduced coordinates increases, the augmentation of the Delaunay triangulation of the sites decreases. We wonder if there exists a condition, possibly on the number of coordinates ($n$) and the number of sites ($m$) which guarantees that the Delaunay triangulation is greedy in its own right. Alternatively, under which conditions does there exist a planar graph (possibly a non-Delaunay triangulation) on the sites which is greedy? This may well be the dual to the Voronoi diagram (for $d_f$).

%\nocite{*}
\bibliographystyle{siamplain}
\bibliography{journals,references}

\begin{thebibliography}{10}

\bibitem{Axler:2001:HFT}
{\sc S.~Axler, P.~Bourdon, and W.~Ramey}, {\em Harmonic Function Theory},
  vol.~137 of Graduate Texts in Mathematics, Springer, New York, 2nd~ed., 2001.

\bibitem{BenChen:2011:DCO}
{\sc M.~Ben-Chen, S.~J. Gortler, C.~Gotsman, and C.~Wormser}, {\em Distributed
  computation of virtual coordinates for greedy routing in sensor networks},
  Discrete Applied Mathematics, 159 (2011), pp.~544--560.

\bibitem{Bose:2004:ORI}
{\sc P.~Bose and P.~Morin}, {\em Online routing in triangulations}, SIAM
  Journal on Computing, 33 (2004), pp.~937--951.

\bibitem{Chen:2016:PPW}
{\sc R.~Chen, C.~Gotsman, and K.~Hormann}, {\em {Path Planning with
  Divergence-Based Distance Functions}}, ArXiv e-prints,  (2017),
  \url{https://arxiv.org/abs/1708.02845}.

\bibitem{Conolly:1993:TAO}
{\sc C.~I. Connolly and R.~A. Grupen}, {\em The applications of harmonic
  functions to robotics}, Journal of Field Robotics, 10 (1993), pp.~931--946.

\bibitem{Csiszar:1967:ITM}
{\sc I.~Csisz{\'a}r}, {\em Information-type measures of difference of
  probability distributions and indirect observations}, Studia Scientiarum
  Mathematicarum Hungarica, 2 (1967), pp.~299--318.

\bibitem{deBerg:2008:CGA}
{\sc M.~de~Berg, O.~Cheong, M.~van Kreveld, and M.~Overmars}, {\em
  Computational Geometry: Algorithms and Applications}, Springer, Berlin,
  3rd~ed., 2008.

\bibitem{Garnett:2005:HM}
{\sc J.~B. Garnett and D.~E. Marshall}, {\em Harmonic Measure}, vol.~2 of New
  Mathematical Monographs, Cambridge University Press, New York, 2005.

\bibitem{Joshi:2007:HCF}
{\sc P.~Joshi, M.~Meyer, T.~DeRose, B.~Green, and T.~Sanocki}, {\em Harmonic
  coordinates for character articulation}, ACM Transactions on Graphics, 26
  (2007), pp.~Article 71, 9 pages.

\bibitem{Khatib:1986:RTO}
{\sc O.~Khatib}, {\em Real-time obstacle avoidance for manipulators and mobile
  robots}, International Journal of Robotics Research, 5 (1986), pp.~90--98.

\bibitem{Kim:1992:RTO}
{\sc J.-O. Kim and P.~K. Khosla}, {\em Real-time obstacle avoidance using
  harmonic potential functions}, IEEE Transactions on Robotics and Automation,
  8 (1992), pp.~338--349.

\bibitem{Koren:1991:PFM}
{\sc Y.~Koren and J.~Borenstein}, {\em Potential field methods and their
  inherent limitations for mobile robot navigation}, in Proceedings of the 1991
  IEEE International Conference on Robotics and Automation, Sacramento, Apr.
  1991, pp.~1398--1404.

\bibitem{Kullback:1951:OIA}
{\sc S.~Kullback and R.~A. Leibler}, {\em On information and sufficiency}, The
  Annals of Mathematical Statistics, 22 (1951), pp.~79--86.

\bibitem{Liese:2006:ODA}
{\sc F.~Liese and I.~Vajda}, {\em On divergences and informations in statistics
  and information theory}, IEEE Transactions on Information Theory, 52 (2006),
  pp.~4394--4412.

\bibitem{Pinkall:1993:CDM}
{\sc U.~Pinkall and K.~Polthier}, {\em Computing discrete minimal surfaces and
  their conjugates}, Experimental Mathematics, 2 (1993), pp.~15--36.

\bibitem{Press:2007:NRI}
{\sc W.~H. Press, S.~A. Teukolsky, W.~T. Vetterling, and B.~P. Flannery}, {\em
  Numerical Recipes in C: The Art of Scientific Computing}, Cambridge
  University Press, New York, 3rd~ed., 2007.

\bibitem{Rimon:1992:ERN}
{\sc E.~Rimon and D.~E. Koditschek}, {\em Exact robot navigation using
  artificial potential functions}, IEEE Transactions on Robotics and
  Automation, 8 (1992), pp.~501--518.

\bibitem{Shewchuk:1996:TEA}
{\sc J.~R. Shewchuk}, {\em Triangle: Engineering a {2D} quality mesh generator
  and {D}elaunay triangulator}, in Applied Computational Geometry. Towards
  Geometric Engineering, M.~C. Lin and D.~Manocha, eds., vol.~1148 of Lecture
  Notes in Computer Science, Springer, Berlin, 1996, pp.~203--222.

\end{thebibliography}

\end{document}